\documentclass{article}

\usepackage[dvipsnames]{xcolor}

\usepackage{float}
\usepackage[english]{babel}
\usepackage[utf8]{inputenc}
\usepackage[T1]{fontenc}
\usepackage[formats]{listings}
\usepackage{geometry}
\usepackage{enumitem}
\usepackage{authblk}
\usepackage{csquotes} 
\RequirePackage{algorithm}
\RequirePackage{algorithmic}

\usepackage{hyperref}

\usepackage{graphicx}
\usepackage[colorinlistoftodos]{todonotes}
\usepackage{setspace}
\usepackage{placeins}
\usepackage{enumitem}
\usepackage{titlesec}
\usepackage{comment}
\usepackage{caption}
\usepackage{subcaption}
\usepackage{enumitem}

\usepackage{adjustbox}

\usepackage{amsmath}
\usepackage{amssymb}
\usepackage{mathtools}
\usepackage{amsthm}
\usepackage{bm}
\usepackage{relsize}
\usepackage{physics}

\usepackage[capitalize,noabbrev]{cleveref}

\theoremstyle{plain}
\newtheorem{theorem}{Theorem}[section]
\newtheorem{proposition}[theorem]{Proposition}
\newtheorem{lemma}[theorem]{Lemma}

\theoremstyle{definition}
\newtheorem{definition}[theorem]{Definition}
\newtheorem{assumption}[theorem]{Assumption}
\newtheorem{remark}[theorem]{Remark}
\theoremstyle{definition}
\newtheorem{example}[theorem]{Example}

\newcommand{\stmis}{\tilde{\bm{\beta}} ^\star}
\newcommand{\E}{\mathbb{E}}

\newcommand{\R}{\mathbb{R}}
\newcommand{\xbf}{\mathbf{x}}
\newcommand{\zbf}{\mathbf{z}}
\newcommand{\ybf}{\mathbf{y}}
\newcommand{\ubf}{\mathbf{u}}
\newcommand{\Xbf}{\mathbf{X}}
\newcommand{\yprebf}{\mathbf{y}_{\mathsf{pre}}}
\newcommand{\Xprebf}{\mathbf{X}_{\mathsf{pre}}}

\newcommand{\that}{\hat{t}}
\newcommand{\what}{\hat{\mathbf{w}}}
\newcommand{\yhat}{\hat{y}}
\newcommand{\xnewbf}{\mathbf{x}_{\mathrm{new}}}
\newcommand{\ynew}{{y}_{\mathrm{new}}}

\newcommand{\betabf}{\bm{\beta}}
\newcommand{\alphabf}{\bm{\alpha}}
\newcommand{\alphahatbf}{\hat{\bm{\alpha}}}
\newcommand{\betahatbf}{\hat{\bm{\beta}}}
\newcommand{\epbf}{{\bm{\varepsilon}}}
\newcommand{\Sigmabf}{\mathbf{\Sigma}}
\newcommand{\betastar}{\mathbf{\beta} ^\star}
\newcommand{\Bstar}{\mathbf{B} ^\star}
\newcommand{\Bstarmis}{\tilde{\mathbf{B}} ^\star}
\newcommand{\BBstar}{\mathbf{B} ^\star {\mathbf{B} ^\star}^\top}

\newcommand{\Bhat}{\widehat{\mathbf{B}}}

\newcommand{\Dhat}{\widehat{\mathbf{D}}}
\newcommand{\Ohat}{\widehat{\mathbf{O}}}
\newcommand{\Qhat}{\widehat{\mathbf{Q}}}

\newcommand{\qbf}{\hat{\mathbf{q}}}
\newcommand{\Ibf}{\mathbf{I}}

\newcommand{\Loss}{\mathcal{L}}
\newcommand{\alphastar}{\bm{\alpha} ^\star}
\newcommand{\qtrue}{q}

\newcommand{\st}{{\bm{\beta} ^\star}}
\newcommand{\lambdaa}{\lambda_\mathbf{\alpha}}
\newcommand{\lambdab}{\lambda_\mathbf{\beta}}
\newcommand{\lambdaB}{\lambda}
\newcommand{\lambdazero}{\lambda_0}
\newcommand{\Gammahat}{\hat{\bm{\Gamma}}}
\newcommand{\Lambdahat}{\hat{\bm{\Lambda}}}
\newcommand{\BSC}{B_\mathsf{SC}}
\newcommand{\VSC}{V_\mathsf{SC}}

\newcommand{\Rcal}{\mathfrak{R}}
\newcommand{\cfrak}{\mathfrak{c}}
\newcommand{\argmin}{\operatornamewithlimits{argmin}}

\newcommand{\dhat}{\hat{d}}
\newcommand{\rhat}{\hat{r}}
\newcommand{\eigtop}{\sigma_{\mathrm{max}}}
\newcommand{\eigbot}{\sigma_{\mathrm{min}}}
\newcommand{\vs}{\varphi}
\newcommand{\xH}{\mathbf{x}_H}
\newcommand{\ho}{h_0}
\newcommand{\hind}{\tilde{h}}
\newcommand{\What}{\hat{\mathbf{W}}}
\newcommand{\Zj}{\mathbf{Z} }
\newcommand{\Zhj}{\mathbf{Z} _h}
\newcommand{\Th}{\hat{\mathbf{T}}_h}
\newcommand{\Ph}{\hat{\mathbf{P}}_h}
\newcommand{\Pbf}{\hat{\mathbf{P}}}
\newcommand{\upbf}{\bm{\upsilon}}
\newcommand{\Phpj}{\hat{\mathbf{P}}_{h,\perp} }
\newcommand{\Phpinf}{\hat{\mathbf{P}}_{h,\perp}^{\infty}}
\newcommand{\upbfj}{\bm{\upsilon} }
\newcommand{\upbfjtilde}{\tilde{\bm{\upsilon}} }
\newcommand{\upbfjbreve}{\breve{\bm{\upsilon}} }
\newcommand{\stj}{{\bm{\beta} ^\star}}

\newcommand{\upbfjtildeT}{\tilde{\bm{\upsilon}}^{\st\top}}
\newcommand{\upbfjbreveT}{\breve{\bm{\upsilon}}^{\st\top}}

\newcommand{\Pw}{\hat{\mathbf{P}}_{h, \perp}^{\varpi,\st}}
\newcommand{\Pwinf}{\hat{\mathbf{P}}_{h, \perp}^{\varpi,\infty}}
\newcommand{\Pz}{\hat{\mathbf{P}}_{h, \perp}^{z,\st}}
\newcommand{\Pzinf}{\hat{\mathbf{P}}_{h, \perp}^{z,\infty}}
\newcommand{\xibf}{\bm{\xi}}
\newcommand{\Xbfj}{{\Xbf} }
\newcommand{\Xbfhj}{{\Xbf} _h}
\newcommand{\XbfhjT}{{{\Xbf} _h}^\top}
\newcommand{\rzero}{\rho_i^{(0)}}
\newcommand{\rone}{\rho_i^{(1)}}
\newcommand{\bigsum}{\mathlarger{\mathlarger{\sum}}}
\newcommand{\Lbf}{\mathbf{L}}

\newcommand{\normop}[1]{\norm{#1}_\mathrm{op}}

\newif\ifshowcomments
\showcommentstrue
\ifshowcomments

\newcommand{\ba}[1]{\textcolor{OliveGreen}{[BA: #1]}}
\newcommand{\yi}[1]{\textcolor{Blue}{[YL: #1]}}
\newcommand{\subha}[1]{\textcolor{Plum}{[SS: #1]}}
\else

\newcommand{\ba}[1]{}
\newcommand{\yi}[1]{}
\newcommand{\subha}[1]{}

\fi

\ifshowcomments
\title{ Understanding Optimal Feature Transfer\\via a Fine-Grained Bias-Variance Analysis}
\else
\title{Understanding Optimal Feature Transfer\\via a Fine-Grained Bias-Variance Analysis}
\fi

\author[1]{Yufan Li\thanks{yufan\_li@g.harvard.edu}}
\author[1]{Subhabrata Sen\thanks{subhabratasen@fas.harvard.edu}}
\author[2]{Ben Adlam \thanks{adlam@google.com}}
\affil[1]{Department of Statistics, Harvard University}
\affil[2]{Google DeepMind}

\begin{document}

\maketitle

\begin{abstract}

In the transfer learning paradigm models learn useful representations (or features) during a data-rich pretraining stage, and then use the pretrained representation to improve model performance on data-scarce downstream tasks. In this work, we explore transfer learning with the goal of optimizing downstream performance. We introduce a simple linear model that takes as input an arbitrary pretrained feature transform. We derive exact asymptotics of the downstream risk and its \textit{fine-grained} bias-variance decomposition. We then identify the pretrained representation that optimizes the asymptotic downstream bias and variance averaged over an ensemble of downstream tasks. Our theoretical and empirical analysis uncovers the surprising phenomenon that the optimal featurization is naturally sparse, even in the absence of explicit sparsity-inducing priors or penalties. Additionally, we identify a phase transition where the optimal pretrained representation shifts from hard selection to soft selection of relevant features.
\end{abstract}

\section{Introduction}

Data scarcity poses significant challenges across various domains, such as computer vision, audio processing, natural language processing, graph learning, and multi-modal learning to name a few. While large models show immense potential when trained on extensive datasets \cite{kaplan2020scaling,hoffmann2022training,henighan2020scaling}, the specific, high-quality data available for many important applications is limited \cite{hu2019strategies,finn2017model,ren2018meta,sun2019meta,lee2022rethinking,wang2020generalizing,li2021universal}. The transfer learning paradigm has been proposed as a potential solution to overcome these limitations in diverse research areas \cite{achiam2023gpt,zhu2023not,raina2007self,blitzer2006domain,peters2019tune,parisi2022unsurprising,wang2015transfer,jang2019learning,chen2020graph,dai2007co,shie2015transfer,buffelli2020meta,hernandez2021scaling}. In this setting, one learns an effective representation in an upstream pretraining stage from data-rich tasks; this representation is then used to improve model performance on downstream tasks where data are scarce.

Despite the widespread adoption of this paradigm, the mechanisms that produce effective representations remain poorly understood. In particular, one may ask the following question: for a given ensemble of downstream tasks, what constitutes optimal pretrained representation? In this paper, we study a family of downstream linear regression tasks of the form  $\ybf \;=\; \Xbf\,{\st} + \epbf,$ where the true parameter \({\st}\) factors through a shared representation \(\Bstar \) via \({\st} = \Bstar {\alphastar}\). We assume that each task-specific coefficient \({\alphastar}\) is drawn from a known prior, and the data matrix \(\Xbf\) has Gaussian covariates with covariance \(\Sigmabf\). Typically, transfer learning proceeds in two steps: (i) obtain an estimator $\tilde{\mathbf{B}}^\star$ of the representation \(\Bstar\) in an upstream, data-rich stage, and (ii) solve the downstream problem \(\ybf = \Xbf \tilde{\mathbf{B}}^\star \alphastar + \epbf\) for each new task by estimating \(\alphastar\). A key observation of our work is that directly plugging in \(\tilde{\mathbf{B}}^\star\) can be suboptimal even if $\tilde{\mathbf{B}}^\star$ accurately estimates $\Bstar$, because \(\tilde{\mathbf{B}}^\star\) alone does not take into account the covariance structure of \(\Sigmabf\) or the distribution of \(\alphastar\). Instead, one can optimize the feature transform \(\Bhat\) to minimize the average downstream risk—the expected risk over the prior of $\alphastar$. Since the downstream data are not available during pretraining, we derive a closed-form expression for the asymptotic downstream risk in the high-dimensional limit, which then becomes a fully differentiable objective for choosing \(\Bhat\). We then examine the structure of the optimal \(\Bhat\) and its dependence on the shared representation \(\Bstar\) and the data covariance \(\Sigmabf\).

\subsection{Contributions}
% Our analysis shows that this optimal transform \(\Bhat\) depends nontrivially on \(\Bstar \), \(\Sigmabf\), and the prior of $\alphastar$. To elucidate these dependencies, we develop a \emph{fine-grained} bias-variance decomposition of the population risk, showing how different aspects of \(\Bhat\) can reduce bias while inflating variance, and vice versa from a series of ablation studies. In a simpler setting where \(\Bhat\) shares eigenvectors with \(\Sigmabf\), we can explicitly characterize how the balance between aligning to \(\Bstar\) and leveraging \(\Sigmabf\) shapes the form of \(\Bhat\). Our theoretical and empirical analysis uncovers the surprising phenomenon that the optimal featurization is naturally sparse, even in the absence of explicit sparsity-inducing priors or penalties. Additionally, we uncover a \emph{phase transition}: when the effective rank of \(\Bstar\) is below a certain threshold, the optimal transform “hard-selects” principal components, analogous to the classical principal component regression (PCR). Above that threshold, it “soft-selects” features, smoothly weighting relevant directions for improved performance. We also conduct numerical experiments when the assumption of shared eigenvectors is lifted. In this more general settings, we empirically observe how the singular vectors of the optimized \(\Bhat\) align partially with those of both \(\Bstar\) and \(\Sigmabf\). 

Our contributions include: 
\begin{enumerate}
    \item Deriving the exact asymptotics for the downstream risk and its fine-grained bias-variance decomposition given an arbitrary linear representation from upstream. To elucidate these dependencies, we develop a fine-grained bias-variance decomposition of the population risk, showing how different aspects of \(\Bhat\) can reduce bias while inflating variance, and vice versa from a series of ablation studies. In a simpler setting where \(\Bhat\) shares eigenvectors with \(\Sigmabf\), we can explicitly characterize how the balance between aligning to \(\Bstar\) and leveraging \(\Sigmabf\) shapes the form of \(\Bhat\).

\item Proposing optimization methodologies to minimize the asymptotic downstream risk as a function of the pretrained representation; conducting ablation studies comparing the total risk, bias, and variance of the optimally pretrained predictor with predictors with no featurization or ground-truth featurization. We also adapt the methodology for a minimax objective that controls the worst-case performance among downstream tasks.

\item Finding that learning both task-relevant features and structures in data covariates are vital, by interpreting the structure of optimal pretraining. Our theoretical and empirical analysis uncovers the surprising phenomenon that the optimal featurization is naturally sparse, even in the absence of explicit sparsity-inducing priors or penalties. Additionally, we uncover a phase transition: when the effective rank of \(\Bstar\) is below a certain threshold, the optimal transform “hard-selects” principal components (analogous to the classical principal component regression), whereas above that threshold, it “soft-selects” features, smoothly weighting relevant directions for improved performance.

\item Conducting numerical experiments when the assumption of shared eigenvectors is lifted. In these more general settings, we empirically observe how the singular vectors of the optimized \(\Bhat\) align partially with those of both \(\Bstar\) and \(\Sigmabf\).

\end{enumerate}

    % \item Demonstrating how estimation errors from pretraining translate into the downstream risk.Our framework has natural connections with meta-learning [CITE], multi-task learning [CITE] etc. 

\noindent
\textbf{Organization:} We introduce our model and the downstream estimation strategy in Section~\ref{Preliminary}. Section~\ref{sec_risk_results} derives the sharp asymptotics for the downstream risk, bias and variance for any given pretrained representation. In Section~\ref{section4}, we optimize the average risk on the downstream task ensemble, and characterize the optimal pre-trained representation. Finally, we investigate the structure of the optimal representation in Section \ref{sec:structure}.

\section{Preliminaries}\label{Preliminary}

% \textbf{Notations.} Expectation over a random variable $X$ is denoted by $\E_{X}$. We use $(\cdot)^+$ to denote the Moore-Penrose pseudo-inverse, $\norm{\cdot}_2$ to denote the $\ell_2$ vector norm, $\norm{\cdot}_\mathrm{op}$ to denote the matrix operator norm, $\bm{1}_{p\times 1}$ to denote a length-$p$ column vector of $1$s, $\eigtop(\mathbf{A}), \eigbot(\mathbf{A})$ to denote the maximum and minimum eigenvalues of symmetric real-valued matrix $\mathbf{A}$, and $\stackrel{L}{=}$ to denote equality in distribution. 

\subsection{Setting}

The transfer learning paradigm comprises two main stages: (i) \textbf{Upstream learning:} The model learns a useful data representation (or feature transform) from a large, often diverse dataset; we also refer to this stage as pretraining and the learned representation as the pretrained representation; (ii) \textbf{Downstream model-fitting:} The model is applied to specific downstream tasks using a smaller, task-specific dataset.

We model downstream tasks as linear regression problems with Gaussian covariates:
\begin{equation}\label{formulation}
    y \;=\; \mathbf{x}^{\top}\betastar + \varepsilon,\quad \quad \betastar \;=\; \Bstar \alphastar,
\end{equation}
where $\mathbf{x}\sim N(\bm{0}, \Sigmabf)\in \R^p$, $\varepsilon\sim N(0,\sigma^2)$, and $\Bstar\in \R^{p\times q}$ is a common, low-rank feature matrix for all downstream tasks. Each task differs by having a specific weight vector $\alphastar\in \R^q$ that forms the parameter $\betastar = \Bstar \alphastar.$ This formulation follows past \cite{tripuraneni2020theory,tripuraneni2021provable,Singh_2023,kong2020robust,sun2021towards,chua2021fine,wang2023improved,du2020few} and concurrent work \cite{bilaj2024meta,hu2024revisiting,watkins2024optimistic,zhong2024bridging,jin2024meta} that assumes that downstream regression tasks share a common linear representation that can be learned upstream. 

When specialized downstream data of size $n$ are limited (i.e., $n<p$), directly estimating $\betastar$ can suffer high variance. A standard approach leverages the low-rank representation $\Bstar$: one can regress on $\alphastar$ via $\Xbf \Bstar$, which lives in $\R^{n\times q}$, reducing dimensionality. However, we will see that simply using $\Bstar$ to featurize data—even if it were known—does not always minimize downstream risk. The covariance $\Sigmabf$ of new data and the prior distribution of $\alphastar$ should also matter.

To capture the variability across different downstream tasks, we impose a prior on \(\alphastar\). Specifically, we assume that \(\alphastar\) is drawn independently from a distribution \(P_{\alphastar}\) with zero mean and covariance \(q^{-1}\Sigmabf_{\alphastar}\). Our objective is to determine a pretrained feature matrix \(\Bhat\) that minimizes the expected downstream risk averaged over this ensemble:
\[
R^{\mathsf{avg}} = \E_{\alphastar}\Biggl[\E_{(\Xbf,\ybf, y_{\mathrm{new}},{\xbf}_{\mathrm{new}})}\Bigl[(y_{\mathrm{new}} - \hat{y}_{(\Bhat, \ybf, \Xbf)}(\xbf_{\mathrm{new}}))^2\Bigr]\Biggr].
\]
Here, the predictor \(\hat{y}_{(\Bhat, \ybf, \Xbf)}\) (introduced in \Cref{solvable}) uses the pretrained feature matrix \(\Bhat\) and is fitted on data \((\ybf, \Xbf)\) for each downstream task. Averaging the risk over \(P_{\alphastar}\) allows us to seek a feature matrix \(\Bhat\) that performs well across the entire ensemble of downstream tasks.

To characterize the optimal feature matrix, we will derive the asymptotic limit of the averaged risk,
$\mathfrak{R}^{\text{avg}}(\Bhat, \Bstar, \Sigmabf, \Sigmabf_{\alphastar})$
and provide its bias-variance decomposition (see \Cref{sec_risk_results}). We further analyze the optimization problem
\[
\min_{\Bhat} \mathfrak{R}^{\text{avg}}(\Bhat, \Bstar, \Sigmabf, \Sigmabf_{\alphastar}),
\]
using both analytical and empirical methods (discussed in \Cref{sec:structure}).

Our primary focus in this paper is on deriving and analyzing the optimal \(\Bhat\) as a function of \(\Bstar\), \(\Sigmabf\). While in practice these quantities may be learned during the upstream pretraining stage, we assume throughout that the ground-truth features $\Bstar$ and data covariance $\Sigmabf$ are known. The motivation for this assumption is to detach the estimation effects of $\Bstar$ and $\Sigmabf$  from their role in determining the optimal choice of $\Bhat$. Nevertheless, in \Cref{app:pretrainsamplecomp} we consider a simple setting where \(\Bstar\) is unknown but can be estimated from upstream regression tasks; we then derive the error incurred by using an estimate \(\tilde{\mathbf{B}}^\star\) in place of \(\Bstar\) in \(\mathfrak{R}^{\text{avg}}\).

\subsection{Ridgeless Regression with Pretrained Representation} \label{solvable}

In this section, we introduce a linear predictor, \(\hat{y} = \hat{y}_{(\Bhat, \ybf, \Xbf)}\), which is fitted on data \((\ybf, \Xbf)\) from a specific downstream task. For our purposes here, the feature matrix \(\Bhat \in \mathbb{R}^{p \times k}\) is treated as a fixed, deterministic input. We will delve into the risk decomposition of \(\hat{y}\) in \Cref{sec_risk_results} and, in \Cref{sec:structure}, identify the form of \(\Bhat\) that optimizes the downstream risk of \(\hat{y}\).

\begin{definition}[Predictor for Downstream Tasks]\label{Downstream}
For a downstream task with data $(\Xbf, \ybf)$, define the empirical loss 
\begin{equation}\label{loss}
\begin{aligned}
    \Loss&(\betabf, \alphabf) := \norm{\ybf-\Xbf \betabf}_2^2+\lambdazero\bigg(\lambdaB\norm{\betabf-\Bhat\alphabf}_2^2+\lambdaa \norm{\alphabf}_2^2+\lambdab\norm{\betabf}_2^2 \bigg).
\end{aligned}
\end{equation}
where $\lambda_0,\lambdaa, \lambdab, \lambdaB>0$. Then the downstream predictor is defined as 
$$ \hat{y}=\xnewbf^\top \hat{\betabf} \quad \text{where} \quad \qty(\hat{\betabf}, \hat{\alphabf})=\lim_{\lambda_0 \to 0} \argmin_{\betabf, \alphabf} \Loss (\betabf, \alphabf).$$
\end{definition}

% The loss \eqref{loss} contains an extra penalty $\lambdaB\|\betabf - \Bhat\alphabf\|_2^2$ that softly enforces “\emph{featurized}” solutions $\betabf \approx \Bhat \alphabf$. By adjusting $\bm{\lambda} = (\lambdaa,\lambdab,\lambdaB)$, we can move continuously from no featurization-taking $\lambdaB\to0$ so that $\betabf$ is unconstrained, recovering standard ridgeless regression $\Xbf^\top (\Xbf \Xbf^\top)^+\ybf$ to Strong featurization: taking $\lambdab\to0$ and $\lambdaB\to\infty$, forcing $\betabf = \Bhat\alphabf$. This becomes a regression on the featurized data $\Bhat^\top\xbf$.

% The loss \eqref{loss} contains an extra penalty $\lambdaB\|\betabf - \Bhat\alphabf\|_2^2$ that softly enforces “\emph{featurized}” solutions $\betabf \approx \Bhat\alphabf$. 

% By adjusting $\bm{\lambda} = (\lambdaa,\lambdab,\lambdaB)$, we can move continuously from no featurization-taking $\lambdaB\to0$ recovers standard ridgeless regression $\Xbf^\top (\Xbf \Xbf^\top)^+\ybf$-to strong featurization: taking $\lambdab\to0$ and $\lambdaB\to\infty$ forces $\betabf = \Bhat\alphabf$

Loss function Eq.~\eqref{loss} contains a penalty term $\lambdaB\|\betabf-\Bhat\alphabf\|_2^2$ that softly enforces the featurization $\betahatbf =\Bhat \alphahatbf $. The strength of featurization is controlled by regularization parameters $\lambda_0,\lambdaa, \lambdab, \lambdaB$. One can adjust these parameters to interpolate from strong featurization to no featurization. For instance, we can set $\lambda =0$ which recovers standard ridgeless estimator. More interestingly, we can consider a strong-featurization limit by letting $\lambdab\to 0$ and $\lambdaB \to +\infty$ for any fixed $\lambda_{\alpha}>0$. It follows from \Cref{explicitsol} that
\begin{equation}\label{feat}
    \betahatbf \to \Bhat \alphahatbf_0 , \qquad \alphahatbf  \to \alphahatbf_0 := \qty(\Xbf  \Bhat)^+\ybf. 
\end{equation}
Therefore, in this limit, $\yhat $ makes predictions by regressing with $\alphahatbf_0 $ composed with the featurization generated by $\Bhat$, which is reminiscent of final-layer head-tuning in neural networks \cite{ba2022high,yang2020feature,nichani2024provable}. Note that $\alphahatbf_0 $ is the minimum $\ell_2$-norm solution, and by a well-known result (see e.g. \cite{hastie2022surprises}), it is also the limit of gradient flow on the objective $\|\ybf -\Xbf \Bhat \alphabf\|^2_2$ when initialized at zero. These interpretations are the main motivation to consider the ridgeless limit, $\lambda_0 \to 0$. 

% Meanwhile, the standard ridgeless predictor rization can be recovered from \Cref{Downstream} by effectively removing the penalty term $\lambdaB\|\betabf-\Bhat\alphabf\|_2^2$ linking $\alphabf$ and $\betabf$ in $\mathcal{L} $. This is achieved by setting $\lambda\to0$, which implies $\Gammahat \to \lambdab \Ibf_p$, and the predictor is given by Eq.~\eqref{sol}. Alternatively, assuming $\hat{d}_i>0$ for all $i=1,\ldots, p$ and taking $\lambdaa\to 0$, we again see $\Gammahat \to \lambdab \Ibf_p$ since $\hat{\alphabf}$ is free to exactly solve $\betabf=\Bhat\alphabf$ for any $\betabf$ without constraints on its norm.

Next we present explicit expressions for $\qty(\betahatbf , \alphahatbf )$. We first write the singular value decomposition of $\Bhat$ as
\begin{equation}\label{BhatSVD}
    \Bhat=\Qhat^\top \Dhat \Ohat, \qquad   \dhat_i:=\qty(\Dhat \bm{1}_{k\times 1})_{i}, i = 1, \cdots, p, 
\end{equation}
where $\Qhat \in \R^{p\times p}$ and $\Ohat \in \R^{k \times k}$ are orthogonal matrices and $\Dhat \in \R^{p\times k}$ is diagonal. 

% Furthermore, we assume that, conditioned on $\Bstar$, $\Bhat$ is independent of the fine-tuning data. 

See \Cref{app:explicit} for a proof of the following result. 

% Our motivation for considering the specific form of loss function \eqref{loss} is two-fold: (i) it has meaningful interpretations that provide natural connections among major themes in transfer learning; (ii) the predictor $\yhat $'s generalization risk at the large system limit can be studied analytically. 

\begin{proposition}[Explicit Expression of the Optimizers]\label{explicitsol}
    We have 
    %\begin{equation}
        \begin{align}
            \betahatbf =\Gammahat^{-1}{\Xbf }^{\top}\qty({\Xbf } \Gammahat^{-1} {\Xbf }^\top)^{+}\ybf , \qquad \quad\alphahatbf =\qty(\Bhat^\top \Bhat + \frac{2 \lambdaa}{\lambdaB})^{-1}\Bhat^\top \betahatbf \label{sol}
        \end{align}
    %\end{equation}
where $(\cdot)^+$ denotes Moore-Penrose pseudo-inverse and $\Gammahat:=\Qhat^\top \Lambdahat \Qhat \in \R^{p\times p}$. Here, $\Lambdahat \in \R^{p\times p}$ is diagonal such that for $i=1,\ldots,p$ and $\bm{\lambda}=(\lambdaa, \lambdab, \lambdaB)$,
\begin{equation}\label{ri}
    \Lambdahat_{ii}=r(\dhat_i^2, \bm{\lambda}):=\lambdab+\lambdaa \cdot \frac{\dhat_i^2 + \frac{4\lambdaa}{\lambdaB}}{\qty(\dhat_i^2 +\frac{2\lambdaa}{\lambdaB})^2}.
\end{equation} 
We will use the notation $\rhat_i := \Lambdahat_{ii}$ for easier exposition.   
\end{proposition}

\begin{remark}[Monotonicity of $r(\cdot)$]\label{rire}
    The function $d^2 \mapsto r(d^2, \bm{\lambda})$ in Eq.~\eqref{ri} is defined on $[0,+\infty]$ and is strictly decreasing and continuous for any $\lambdaa,\lambdab,\lambdaB>0$. It attains a maximum of $\lambdab+\lambdaB$ at $d^2=0$ and a minimum of $\lambdab$ at $d^2=+\infty$. 
\end{remark}

\subsection{Related Work}\label{RW}
% Aside from differences in the setting, our work differs from \cite{sun2021towards} in two important ways: (i) \cite{sun2021towards} focuses more on methodology and does not investigate underlying mechanisms of optimal pretraining; our work, on the other hand, seeks to understand characteristics of optimally-learned feature matrix and examine, utilizing a fine-grained bias-variance decomposition, how optimal feature learning leverages data distribution and ground-truth representation to improve downstream risk ablating various problem and model parameters (ii) \cite{sun2021towards} is restricted to the simplistic setting where feature matrix, ground-truth linear representation, and data covariance are all diagonal matrices; our work develops theory and algorithms without this restriction.

% Our result highlights the important role of learning dependencies in the data, in addition to the ground-truth representation, of regularizing variance of the downstream risk. 

% Nevertheless, there are important differences between our work and \cite{sun2021towards}. \cite{sun2021towards} are developed in a 

% {\color{red} Decompress further. }

The papers \cite{Dicker,dobriban2018high} first studied the asymptotic risk of ridge regression as the number of data points and number of features grow proportionally. In \cite{hastie2022surprises}, this was extended to ridgeless regression with non-asymptotic bounds, which was further extended by \cite{cheng2022dimension} to a dimension-free setting. Our work builds on techniques developed in these papers, while adding a novel fine-grained bias-variance decomposition. Beyond simple ridge regression, a considerable body of work \cite{adlam2019random,adlam2020neural,adlam2020understanding,mei2018mean,mei2022generalization,ghorbani2021linearized,belkin2020two,hu2022universality} is devoted to featurized models, like kernel ridge regression and two-layer neural networks. These works typically adopt the random feature assumption, where only the final layer is trained. Our studies may be seen as an attempt to move away from this assumption, allowing for a learnable linear kernel. 

Our model formulation follows a long line of past \cite{tripuraneni2020theory,tripuraneni2021provable,Singh_2023,kong2020robust,sun2021towards,chua2021fine,wang2023improved,du2020few} and concurrent work \cite{bilaj2024meta,hu2024revisiting,watkins2024optimistic,zhong2024bridging,jin2024meta} that assumes that downstream regression tasks share a common linear representation that can be learned upstream. However, most of the existing work focuses on proposing a methodology to learn the ground-truth representation and establish an associated sample complexity bound downstream. Additionally, there is often some prior assumptions on low-rank structures or sparsity. Our approach differs from these works in two key aspects (i) we characterize the downstream risk exactly and (ii) we do not impose artificial constraints on the existence of low-dimensional structure or artificially impose sparsity-inducing priors or penalties. Rather, our theoretical and empirical results show that sparsity and feature selection naturally emerges as a consequence of optimizing downstream risk. 

The study of asymptotic risk is inherently tied to optimizing model parameters to minimize risk. In ridge regression, \cite{nakkiran2020optimal} demonstrates that careful tuning of \(\ell_2\) regularization can mitigate the double-descent phenomenon. More relevant to our setting, prior and concurrent works \cite{wu2020optimal,jin2024meta} examine generalized ridge regression. \cite{wu2020optimal} analyzes generalized ridge regression in a single-task setting, characterizing risk and classical bias-variance decomposition while studying penalty matrices that optimize bias and variance. However, their asymptotic analysis is limited to cases where the penalty matrix commutes with the data covariance, and their characterization lacks finite-sample error bounds. Moreover, they do not investigate fine-grained bias-variance decomposition, focus on the transfer learning setting, or observe the feature selection and phase transition phenomena we identify. Meanwhile, concurrent work \cite{jin2024meta} studies generalized ridge regression from a meta-learning perspective, assuming a shared structure among tasks, similar to our work. They characterize asymptotic model risk and analyze the feature matrix that optimizes the asymptotic risk. Their results improve upon \cite{wu2020optimal} by removing the commutativity assumption. However, their setting and analysis differ from ours in two key aspects: (i) we focus on the ridgeless regime, whereas their results require a penalty strength bounded away from zero where there are nontrivial differences in conclusions and analysis. Notably, in their setting, the optimal featurization does not exhibit sparsity (i.e. feature selection); (ii) They do not consider bias-variance decomposition, which is central to our characterization of the feature selection and phase transition phenomena. 

\section{Analytic Results for Downstream Risk}
\label{sec_risk_results}

%\subsection{Exact Asymptotics}\label{sec_risk_results}

Recall that we introduced a linear predictor $\hat{y} = \hat{y}_{(\Bhat, \ybf, \Xbf)},$ which leverages a pretrained featurization \(\Bhat\) and is fitted using data \((\ybf, \Xbf)\) from a downstream task. We now derive asymptotic expressions for bias-variance decomposition of the risk of \(\hat{y}\), defined as
\[
R := \E_{\ynew, \xnewbf} \left(\ynew - \hat{y}(\xnewbf)\right)^2.
\]
We emphasize the importance of fine-grained bias-variance decomposition because explicitly separating these sub-components of risk provides a clearer understanding of how different design choices in feature selection influence overall risk. In particular, we may investigate which \(\Bhat\) optimizes bias and variance component respectively. In \Cref{sec:structure}, we attribute the observed sparsity in optimal featurization (i.e., feature selection) to the fine-grained bias and precisely characterize the associated phase transition phenomenon. 

In this section, we treat \(\Bhat \in \mathbb{R}^{p \times k}\) as a fixed, deterministic input as we derive asymptotic expressions for risk and its bias-variance decompositions.

The classical bias-variance decomposition of the risk $\E_{\epbf}R$ is defined as $\BSC +\VSC $ with
\begin{equation*}
    \BSC =\E_{ \xnewbf }\left(\ynew -\E_{\epbf } \yhat \right)^2, \qquad \VSC =\E_{ \xnewbf }\mathbb{V}_{\epbf } \yhat 
\end{equation*}
for training data $\ybf=\Xbf\st +\epbf$ and new response $\ynew ={\xnewbf }^\top {\st}$. The bias-variance decomposition above is widely recognized in the statistics literature \cite{hastie2009elements,james1997generalizations,hastie2022surprises,wu2020optimal}. However, as pointed out in \cite{adlam2020understanding}, the above is conditional on $\Xbf $ and only decomposes randomness in label noise $\epbf $. Following \cite{adlam2020understanding}, we consider the \textit{fine-grained} bias-variance decomposition
\begin{equation}
    \begin{aligned}
        &\E_{\Xbf , \epbf} R =B +V ,
    \end{aligned}
\end{equation}
where the bias component is $B :=\E_{\xnewbf } \left(\ynew -\mathbb{E}_{\Xbf \!\!, \epbf } \hat{y} \right)^2$ and variance component is $V := V_{\Xbf} +V_{\Xbf,\epbf} +V_{\epbf} $ with a further decomposition
\begin{equation*}
\begin{aligned}
        &V_{\Xbf} :=\mathbb{E}_{\xnewbf } \mathbb{V}_{\Xbf }\mathbb{E}_{\epbf } \hat{y} , \qquad V_{\epbf} :=\E_{\xnewbf } \mathbb{V}_{\epbf }\mathbb{E}_{\Xbf } \hat{y}, \qquad  V_{\Xbf,\epbf} :=\mathbb{E}_{\xnewbf } \mathbb{V}_{\Xbf , \epbf }(\hat{y} )-V_{\Xbf} -V_{\epbf} .
\end{aligned}
\end{equation*}
Here, the bias-variance decomposition is with respect to randomness in \textit{both} the data and label noise; $V_{\Xbf} $, $V_{\epbf} $, and $V_{\Xbf,\epbf} $ are a two-way ANOVA decomposition of $V $: $V_{\Xbf} $ and $V_{\epbf} $ are the variances explained by data and noise individually, and $V_{\Xbf,\epbf} $ the additional variance explained by data and noise jointly. Unlike the classical decomposition, the fine-grained decomposition provides a clear interpretation of each additive component within the asymptotic risk formula (see \eqref{A1A2} and discussion below). Additionally, as observed in \cite{adlam2020understanding} and \cite{yang2020rethinking} interpreting the classical bias can be challenging, especially as it tends to diverge at the interpolation boundary. The fine grained decomposition resolves this by attributing the divergence to the variance component \(V_{\Xbf}\), while the fine grained bias \(B\) exhibits straightforward monotonic trends in our ablation studies (see \Cref{twoboundaries} and \ref{ablation}).

% {\color{red} Say why we do we care about the bias variance instead of just minimizing the whole risk. }

% Note the relationships 
% \begin{equation}
%     \E_{\Xbf } \BSC =B +V_{\Xbf} , \qquad \E_{\Xbf } \VSC =V_{\Xbf,\epbf} +V_{\epbf} \label{DCb}
% \end{equation}
% between the classical and fine-grained decompositions.

% (iii) theoretical and empirical analysis of the FG-decomposed components uncovers a new phase transition within the $n>p$ regime (see \Cref{sec:structure}).

\begin{figure}
    \centering
    \includegraphics[width=0.8\linewidth]{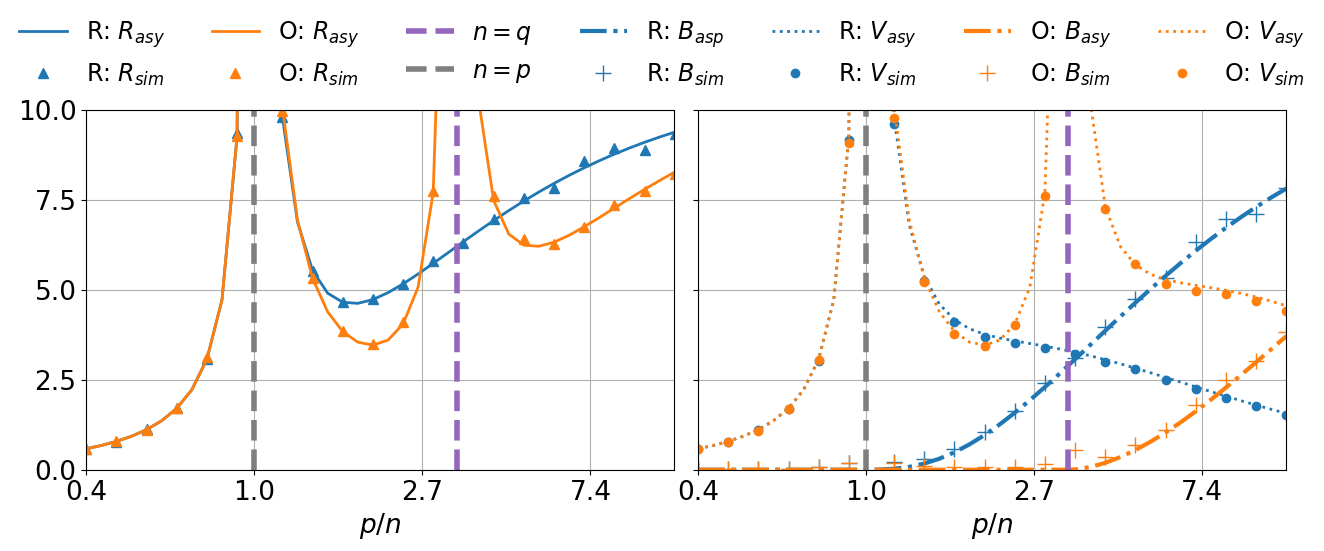}
    \caption{\textbf{(left)}: Compare asymptotic risk $\E_{\Xbf , \epbf }\Rcal $ (denoted $R_{\mathrm{asy}}$) with empirical mean of simulated risk $R $ across 50 sample draws of $\Xbf , \epbf $ ($\st$ fixed) of standard ridgeless predictor (denoted $\mathsf{R}$) and predictor with oracle featurization: $\Bhat\gets \Bstar, \lambdaa=\lambdaB=1,\lambdab=0$ (denoted $\mathsf{O}$). We fix $p=3000$ and vary $n$ from 8200 to 250 (x-axis on log-scale). \textbf{(right)}: Compare asymptotic bias and variance $B , V $ (denoted $B_{\mathrm{asy}}, V_{\mathrm{asy}}$) of the two predictors with their simulated counter-parts (denoted $B_{\mathrm{sim}}, V_{\mathrm{sim}}$). All plots are generated with columns of $\Bstar \in \R^{p \times q}, q=900$ drawn independently from $N(\bm{0}, \Sigmabf^{\Bstar}), \Sigmabf^{\Bstar}_{ij}=0.5^{|i-j|}$, $\Sigmabf\sim \frac{1}{p}\mathbf{W} \mathbf{W}^\top+0.005\cdot \Ibf_p, \mathbf{W}\sim N(\bm{0}, \Ibf_p \otimes \Ibf_p)$ and $\alphastar \sim N(\bm{0},\mathfrak{c}\cdot \Ibf)$. We maintain $\sigma^2=1$ and set $\mathfrak{c}$ such that $\mathsf{SNR}:=\norm{\st}_2/\sigma=10$. }
    \label{twoboundaries}
\end{figure}

We now describe the asymptotics of the downstream objective risk for the fine-tuned predictor from \Cref{Downstream}.

\begin{assumption}\label{Assum}
    Fix $M>0$. Let $\zbf_i\sim N(\bm{0},\Ibf_p)$ i.i.d. and define $\xbf_i= \Sigmabf^{1/2} \zbf_i$. Define $$h:=\mathrm{rank}(\Sigmabf)\leq p$$ and let $\eta_{\min}^+$ be the smallest non-zero eigenvalue of $\Sigmabf$. We assume 
    \begin{equation}\label{asseq}
        {1}/{\eta_{\min}^+}, \;\; \normop{\Sigmabf}, \;\; \normop{\Gammahat}, \;\; \normop{\Gammahat^{-1}} <M.\footnote{Recall \Cref{explicitsol} for the definition of $\Gammahat$.}
    \end{equation}
\end{assumption}

\begin{definition}[Self-consistent equation]\label{def_sce}
Denote the eigendecomposition of  $\Gammahat^{-1/2} \Sigmabf \Gammahat^{-1/2}$ by
\begin{equation}\label{eigprod}
    \Gammahat^{-1/2} \Sigmabf \Gammahat^{-1/2}=\sum_{i=1}^p \that_i \cdot \what_i \what_i^\top.
\end{equation}
Define $H:=\{i\in \{1,\ldots,p\}: \that_i \neq 0\}$,
where $\abs{H}=h=\mathrm{rank}(\Sigmabf)$. Let $b_0 \in \R_+$ be the unique non-negative solution of
\begin{equation}\label{fp}
    1-\frac{n}{h}= \frac{1}{h} \sum_{i\in H} \frac{1}{1+\that_i b_0}.
\end{equation}
when $n<h$. 
We define the quantities
\begin{equation}\label{A1A2}
    \begin{aligned}
        & \mathcal{V}:=\frac{\sum_{i\in H} \frac{(\that_i b_0)^2}{(1+\that_i b_0)^2}}{\sum_{i\in H} \frac{\that_i b_0}{(1+\that_i b_0)^2}}, \quad \mathfrak{B} :=\sum_{i \in H} \frac{\that_i \left\langle\what_i, \Gammahat^{\frac{1}{2}} {\st} \right\rangle^2}{\left(1+\that_i b_0\right)^2}, \quad \Rcal :=\mathfrak{B}+\mathcal{V}\mathfrak{B}+ \sigma^2  \mathcal{V}, \quad \mathcal{U}:=\sigma^2 \frac{h}{n-h}.
    \end{aligned}
\end{equation}
\end{definition}
The following theorem states that in the sample-deficient regime ($n<h=\mathrm{rank}(\Sigmabf)$), the total risk $R$ and its components $B, V_{\Xbf}, V_{\Xbf,\epbf}$ converge to $\Rcal$ and $\mathfrak{B}, \mathcal{V}\mathfrak{B}$ and $\mathcal{V}$ respectively whereas in the sample-rich regime ($n>h=\mathrm{rank}(\Sigmabf)$), the total risk has only one non-zero component $V_{\Xbf,\epbf}$ and it converges to $\mathcal{U}$. We defer the proof to Appendix \ref{app:riskchar}.
\begin{theorem}\label{thm}
Let Assumption \ref{Assum} hold. 
\begin{itemize}
    \item [(i)] \textbf{Sample-deficient regime ($n<h\equiv\mathrm{rank}(\Sigmabf)$).} If in addition $1+M^{-1}<h/n<M$, we have that for any constant $D>0$, there exists $C=C(D,M)$ such that $$\abs{R -\Rcal }\leq Cn^{-1/7}\norm{{\st}}_2^2$$ with probability at least $1-Cn^{-D}$. Moreover, for some $C=C(M)$, 
    \begin{equation*}
        \begin{aligned}
            % & \abs{\BSC - (\mathcal{V}+1)\cdot \mathfrak{B} }\leq Cn^{-1/7}\norm{{\st} }_2^2\\
            % & \abs{\VSC - \sigma^2\cdot \mathcal{V} } \leq Cn^{-1/7}.
            &\abs{B -\mathfrak{B} }, \;\; \abs{V_\Xbf -\mathcal{V} \mathfrak{B} } \leq Cn^{-1/7}\norm{{\st}}_2^2, \qquad \abs{V_{\Xbf,\epbf} -\sigma^2 \cdot \mathcal{V}}  \leq Cn^{-1/7}.
        \end{aligned}
    \end{equation*}
     %Meanwhile, we have that $V_{\epbf} =0$ and
    % \begin{equation*}
    %     \begin{aligned}
    %         &\abs{B -\mathfrak{B} }, \abs{V_\Xbf -\mathcal{V} A_2 } \leq Cn^{-1/7}\norm{{\st} }_2^2 \\
    %         & \abs{V_{\Xbf,\epbf} -\sigma^2 \cdot \mathcal{V}}  \leq Cn^{-1/7}.
    %     \end{aligned}
    % \end{equation*}
    \vspace{-5mm}
    \item [(ii)] \textbf{Sample-rich regime ($n>h\equiv\mathrm{rank}(\Sigmabf)$).} If in addition $M^{-1}<h/n<1-M^{-1}$, we have $B =V_{\Xbf} =V_{\epbf} =0$ and for any constant $D>0$, there exists $C=C(D,M)$ such that $$\abs{R -\mathcal{U}}\le Cn^{-1/7}$$ with probability at least $1-Cn^{-D}$. Moreover, for some $C=C(M)$, $\abs{V_{\Xbf, \epbf} -\mathcal{U}}\le Cn^{-1/7}.$
    % that $\BSC =0$ and for any constant $D>0$, there exists $C=C(D,M)$ such that, with probability at least $1-Cn^{-D}$, 
    % \begin{equation}\label{underV}
    %     \abs{R -\mathcal{U}}, \abs{\VSC -\mathcal{U}}\le Cn^{-1/7}.
    % \end{equation}
    % Meanwhile, we have that $B =V_{\Xbf} =V_{\epbf} =0$ and $\abs{V_{\Xbf, \epbf} -\mathcal{U}}\le Cn^{-1/7}$ for some constant $C=C(M)$.
\end{itemize}
\end{theorem}

% \begin{remark}[Technical Novelty]
%     Note that by \eqref{Lb}, our setting reduces to the standard setting of plain ridgeless regression with feature covariance $\Gammahat^{-1/2} \Sigmabf \Gammahat^{-1/2}$. The latter setting has been explored in numerous prior works \cite{dobriban2018high,hastie2022surprises,cheng2022dimension} using random matrix theory. Our main technical contribution is to extend the resolvent technique to the fine-grained bias-variance characterization via a symmetrization argument (see \Cref{CharFBV}). Moreover, we emphasize that unlike \cite{wu2020optimal,sun2021towards,jin2024meta}, our result is non-asymptotic and does not require convergence of empirical spectral distributions. 
% \end{remark}

%\subsection{Double Divergence under Ground-truth Featurization}
In \Cref{twoboundaries}, we plot sample and asymptotic risks and bias-variance decomposition from \Cref{thm} for a predictor with ground-truth featurization ($\Bhat= \Bstar$) and a standard ridgeless predictor ($\Bhat=\Ibf$). We observe good agreement between sample and asymptotic quantities for both predictors. 

Interestingly, we see that the variance of the predictor with ground-truth featurization diverges for a second time as $n$ approaches $q$, resulting in a ``double-divergence'' in the risk curve whereas the standard ridgeless predictor only diverges once at the interpolation boundary $n=p$. The first divergence at $n=p$ is well-known \cite{hastie2022surprises} and can be attributed to the fact that the model only has marginally enough parameters to interpolate all the data at the boundary. On the other hand, the second divergence of the risk of the featurized predictor can be attributed the fact that the featurization introduced an approximate intrinsic-dimension into the transformed data as the ground truth featurization $\Bstar$ is not full rank, i.e. $\rank(\Bstar)=q<p$. 

As a result, the featurized predictor is not consistently better than the ridgeless predictor, suggesting an opportunity to improve on both and potentially remove the divergences in variance. Notably, it can be seen from \Cref{thm} that the first divergence cannot be mitigated in our ridgeless setting (i.e. $\lambda_0\to 0$) by adjusting featurization $\Bhat$ and requires tuning $\lambda_0$ as suggested in \cite{nakkiran2020optimal,wu2020optimal}; on the other hand, as shown in \Cref{section4}, the second divergence can indeed be mitigated by choosing $\Bhat$ optimally.

% {\color{red} We demonstrate \Cref{thm} in \Cref{twoboundaries} for the standard ridgeless predictor ($\Bhat = \Ibf_p$) and a predictor with ground-truth featurization ($\Bhat= \Bstar$ with regularization  $\lambdaa=\lambdaB=1,\lambdab=0$). The ground-truth predictor exhibits a second interpolation boundary as $n$ approaches $q=\mathrm{rank}(\Bstar)$. This is because $q$ becomes the intrinsic dimension of the transformed data according to \eqref{Lb}. Thus, the variance diverges as $n$ approaches this data dimension (right panel).}

% Below is comparison with existing results (add later)
% \begin{remark}
%     \Cref{thm} may be understood as extension of the results from \cite{hastie2022surprises,cheng2022dimension} for arbitrary generalized penalty matrix $\Gammahat$. Our results also extends \cite{hastie2022surprises,cheng2022dimension} in two other aspects: (i) the data covariance $\Sigmabf$ is permitted to be singular; (ii) \Cref{thm} characterizes FG bias and variance. Our results may also be compared to those in \cite{wu2020optimal} which provides risk characterization for general generalized penalty matrix. However, our analysis provides non-asymptotic guarantees and does not require assumptions on joint asymptotics of spectrum of matrix quantities such as those in \cite{wu2020optimal}, Assumption 1; \cite{wu2020optimal} also does not consider fine-grained bias and variance decomposition or singular data covariance $\Sigmabf$. 
% \end{remark}

\section{Downstream-Optimal Feature Transfer}\label{section4}

\begin{figure*}[t] % (*) for spanning both columns
  \centering

  % First row
  \includegraphics[width=0.95\columnwidth]{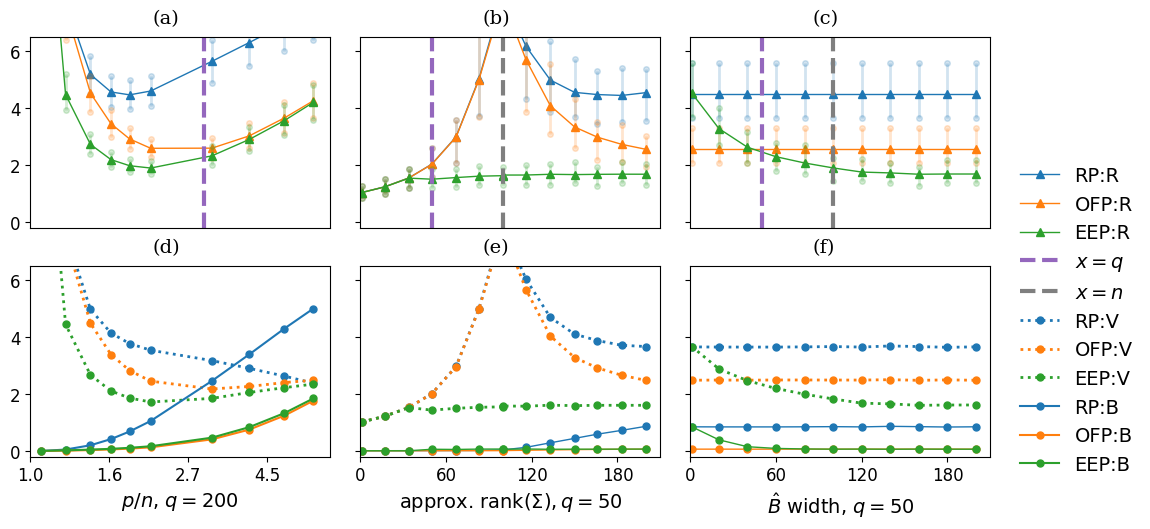} 
  % % Second row
  % \includegraphics[width=2\columnwidth]{Image/biasvar.png} 
  \caption{\textbf{(a)-(c)}: Empirical mean of the asymptotic risk $\Rcal $ (denoted $R$) over $3000$ draws of $\alphastar \sim N(\bm{0},\mathfrak{c}\cdot \Ibf)$, for $\Bhat \in \R^{p\times k}$ as RP, OFP, and EEP. Error bars depict empirical mean and standard deviation of the actual risk $R $, evaluated from simulated data $(\ybf , \Xbf )$ across different $\alphastar$ draws. \textbf{(d)-(f)}: Empirical mean of bias $B $ (denoted $B$) and variance $V $ (denoted $V$) over $\alphastar$ draws. All plots are generated with columns of $\Bstar \in \R^{p \times q}$ drawn independently from $N(\bm{0}, \Sigmabf^{\Bstar}), \Sigmabf^{\Bstar}_{ij}=0.5^{|i-j|}$ and $\Sigmabf$ from $\frac{1}{m}\mathbf{W} \mathbf{W}^\top+0.005\cdot \Ibf_p, \mathbf{W}\sim N(\bm{0}, \Ibf_p \otimes \Ibf_m)$, with $m$ being the approximate rank of $\Sigmabf$. \textbf{(a)} and \textbf{(d)} fix $p=600, q=300$ and vary $n$ from $560$ to $100$. \textbf{(b)} and \textbf{(e)} vary $m$ for $q=50$. \textbf{(c)} and \textbf{(f)} varies $k$, the width of $\Bhat$. \textbf{(d)} and \textbf{(h)} varies $q$.  We maintain $\sigma^2=1$ and set $\mathfrak{c}$ such that $\mathsf{SNR}:=\norm{\st}_2/\sigma=10$. We set $n=100, p=m=k=200$ unless specified otherwise. }
  \label{ablation}
\end{figure*}

In this section, we optimize the average risk across an ensemble of downstream tasks with respect to the pretrained feature matrix \(\Bhat\). In the sample-rich regime, the risk is determined by the limiting expression \(\sigma^2 \frac{h}{n-h}\), where \(h = \operatorname{rank}(\Sigmabf)\), independent of \(\Bhat\). Therefore, our focus here is on the sample-deficient regime where \(n < h\).

Setting the feature matrix to the ground truth, i.e., \(\Bhat = \Bstar\), is not necessarily optimal. Intuitively, the optimal \(\Bhat\) should account for both the distribution of covariates and the characteristics of the downstream task. A principled approach to selecting \(\Bhat\) would be to directly minimize the downstream risk. However, this is often impractical, either because downstream tasks are unknown at the time of pretraining or because the upstream-downstream process is not end-to-end differentiable. However, if we further assume that the task-specific parameter $\alphastar$ are distributed i.i.d. from a prior $P_{\alphastar}$ with zero mean and covariance $q^{-1}\Sigmabf_{\alphastar} \in \R^{q\times q}$, our theoretical results from \Cref{sec_risk_results} give us an analytic expression for the asymptotic behavior of downstream risk averaged across potential downstream tasks $R^{\mathsf{avg}} := \E_{\alphastar} R$:
\begin{equation}\label{avgg}
\begin{aligned}
    &\Rcal^{\mathsf{avg}}:=\E_{\alphastar} \Rcal = \mathfrak{B}^{\mathsf{avg}} +\mathfrak{B}^{\mathsf{avg}}\mathcal{V} + \sigma^2\mathcal{V} ,\\
\end{aligned}
\end{equation}
where
\begin{equation*}
    \begin{aligned}
        \mathfrak{B}^{\mathsf{avg}}:= \E_{\alphastar} \mathfrak{B} \qty(\Bhat, {\alphastar},\Bstar)=\frac{1}{\qtrue} \sum_{i \in H} \frac{\that_i \cdot\what_i^\top \Gammahat^{\frac{1}{2}} \Bstar \Sigmabf_{\alphastar} {\Bstar}^\top \Gammahat^{\frac{1}{2}} \what_i }{\left(1+\that_i b_0\right)^2}.
    \end{aligned}
\end{equation*}
Note that \eqref{avgg} captures typical downstream risk averaged across the tasks while requiring no downstream datasets. We also observe the expectation is free of unknown quantities at the pretraining stage (if we assume $\Sigmabf_{\alphastar}$ is also known\footnote{If not, a practical approach is to use a non-informative prior, $\Sigmabf_{\alphastar}=\mathfrak{c}\cdot\Ibf_q$ for some $\mathfrak{c}>0$.}) and is end-to-end differentiable. Motivated by these, we define the following predictor. 

\begin{definition}[End-to-end predictor (EEP)]\label{EEPdef}
   We find $\Bhat$ and regularization parameters $\bm{\lambda} \in \R_+^3$  by minimizing $\Rcal^{\mathsf{avg}}(\Bhat, \bm{\lambda}, \Sigmabf, \Bstar)$\footnote{$\bm{\lambda}$ are optimized to minimize $\Rcal^{\mathsf{avg}}$ in the OFP also.}
   \begin{equation*}
       \Bhat^{\mathrm{opt}}, \bm{\lambda}^{\mathrm{opt}} = \argmin_{\Bhat, \bm{\lambda}} \Rcal^{\mathsf{avg}}(\Bhat, \bm{\lambda}, \Sigmabf, \Bstar)
   \end{equation*}
   Then the end-to-end predictor is the predictor $\hat{y}$ defined in \Cref{Downstream} with $\Bhat=\Bhat^{\mathrm{opt}}, \bm{\lambda}=\bm{\lambda}^{\mathrm{opt}}$. 
\end{definition}

% \begin{remark}[A Backpropagation Routine]
%     To optimize $\Rcal^{\mathsf{avg}}$, we develop a custom backpropagation routine in \Cref{app:obtainfully} where we also discuss the non-standard step of backpropagating through the fixed point $b_0$. 
% \end{remark}

\begin{remark}[Minimax Optimality]\label{minimaxx}
    A minimax procedure may be developed such that one seeks to  control the risk for the downstream risk for the worst task $R^\mathrm{worst}:=\max_{\norm{\alphastar}_2^2\le \cfrak}R ,\cfrak>0$. See \Cref{appminimax}.
    
    % Our procedure remains the same except that $\Rcal^{\mathrm{avg}}$ in \Cref{EEPdef} is replaced with asymptotic limit of  $\Rcal^{\mathsf{worst}}$ of $R^\mathrm{worst}$
    % \begin{equation*}
    %     \Rcal^{\mathsf{worst}}:=\max_{{\alphastar} \in \mathbb{B}^q(\sqrt{\cfrak})} \Rcal  =  \sigma^2 \mathcal{V}+\cfrak \cdot\left(\mathcal{V}+1\right) \cdot  \eigtop \qty (\sum_{i \in H} \frac{\that_i \cdot  {\Bstar}^\top \Gammahat^{\frac{1}{2}} \what_i \what_i^\top \Gammahat^{\frac{1}{2}} \Bstar}{\left(1+\that_i b_0\right)^2}),  
    % \end{equation*}
    % Note that $\Rcal^{\mathsf{worst}}$ remains end-to-end differentiable. See \Cref{appminimax} and \ref{app:obtainfully} for derivation of the above formula, theoretical guarantee and optimization routines. 
\end{remark}

% Alternatively, when some data for the downstream tasks are accessible during the pretraining phase, $P_{\alphastar}(\cdot)$ can be determined by updating a prior based on this available data, effectively using a posterior distribution for $P_{\alphastar}(\cdot)$.

% \begin{remark}[Methodological Contributions and Novelty]
%     Prior and concurrent works \cite{wu2020optimal,sun2021towards,jin2024meta} have explored similar ideas of optimizing the generalized ridge penalty against the downstream risk. In particular, \cite{wu2020optimal,sun2021towards} require the strong assumption that the penalty matrix $\Gammahat$ commutes with data covariance $\Sigmabf$. The concurrent work \cite{jin2024meta} dispenses the commutativity assumption but their optimization result only applies to a particular choice of ridge penalty strength (see \cite{jin2024meta}, Theorem 2.5).
% \end{remark}
  
\Cref{concenprop} below supports the use of the objective $\Rcal^{\mathsf{avg}}$, as it implies that with a slightly stronger assumption on the prior distribution of $\alphastar$, the risk for specific downstream tasks concentrates to $\Rcal^{\mathsf{avg}}$. See \Cref{app:downstreamrisk} for a proof. 
\begin{proposition}\label{concenprop}
    Suppose that \Cref{Assum} holds and that 
    \begin{equation}\label{concen}
        {\alphastar}=\qtrue^{-1/2}{\Sigmabf^{1/2}_{\alphastar}} \xibf, \qquad \norm{\Sigmabf_{\alphastar}}_{\mathrm{op}}\le M
    \end{equation}
    for the constant $M$ in \Cref{Assum}, where $\xi$ is a random vector with independent, sub-Gaussian, zero-mean, unit-variance entries and the sub-Gaussian norm bounded by $M$. Then, for any $D>0$, there exists $C=C(M,D)$ such that with probability at least $1-C(n^{-D}+\qtrue^{-D})$, 
    \begin{equation*}
        \abs{R -\Rcal^{\mathsf{avg}}}\le C\qty(n^{-1/7}+\sqrt{\frac{\log \qtrue}{\qtrue}} \cdot \norm{\Bstar {\Bstar}^\top }_{\mathrm{op}}).\footnote{Typically, we may take $\|\Bstar {\Bstar}^\top \|_{\mathrm{op}} \sim O(1)$ because then, by Eq.~\eqref{concen}, the signal-to-noise ratio $\|{\st}\|_2^2/\sigma^2$ would also be of constant order.}
    \end{equation*}
\end{proposition}

Below and in \Cref{ablation}, we compare the end-to-end predictor (EEP) above with two other approaches with a series of ablation studies: (i) Ridgeless Predictor (RP): ignoring the pretrained features altogether and simply using the standard predictor, $\hat{y} (\xnewbf) = {\Xbf }^{\top}({\Xbf } {\Xbf }^\top)^{+}\ybf$; (ii) Oracle-Featurization Predictor (OFP): setting $\Bhat=\Bstar$ and optimizing $(\lambda, \lambdaa, \lambdab)$.

% The first approach we consider does not use pretrained featurization. We refer to this approach as \textit{no-featurization}. We include no-featurization approach as a baseline, and it is expected to underperform when the available downstream data $n$ are limited. It is only optimal when $n\gg p$, where no dimension reduction is necessary.
% \begin{definition}[No-featurization predictor]
%     The method has no hyperparameters and sets $$\hat{y} (\xnewbf) = {\Xbf }^{\top}({\Xbf } {\Xbf }^\top)^{+}\ybf .$$
% \end{definition}

% The second approach When $\Bstar$ is low-rank and the downstream data are limited (i.e. $\qtrue\ll n<p$), the ground-truth featurization can markedly increase performance by leveraging the shared representation. 
% \begin{definition}[Oracle-featurization predictor]
%     The oracle-featurization predictor is the predictor $\hat{y}$ defined in \Cref{Downstream} with $\Bhat=\Bstar$, i.e. the ground-truth featurization, and with $(\lambda, \lambdaa, \lambdab)$ set optimally. 
% \end{definition}

We find that the EEP easily finds $\Bhat$ with lower risk than the RP and OFP. From \Cref{ablation}, we see that the EEP does this by making a better tradeoff of bias and variance. This is despite either the bias or variance of the other predictors being lower in some settings. In (d), the EEP is unbiased when $q<n$ by exploiting $\Bstar$'s low-rank structure, just like the OFP. Adjusting $\bm{\lambda}$ in both the OFP and the EEP removes the divergence in variance (compare (a) and (d) to \Cref{twoboundaries}). However, \Cref{thm} predicts that variance diverges when $h=n$ regardless of $\Bhat$. Even when $\Sigmabf$ includes a jitter term so that it is actually full-rank and only approximately low-rank (see (b) and (e)), the OFP still sees an explosion in variance as $n$ approaches the approximate rank. In contrast, the EEP is free to align its eigenvalues and eigenvectors to those of $\Sigmabf$ and effectively modulate the variance at the expense of a slightly larger bias. 

Finally, we see in (c) and (f), how as the capacity of $\Bhat$ increases, its performance starts by matching that of the RP and eventually matches then exceeds that of the OFP. Notably, when $n>q$ (see (e) and (f)), OFP and EEP can completely remove the bias, which is not true when $n<q$ (see (d)); this corresponds to a phase transition between the hard- and soft-selection regime discussed in the next section. \Cref{SIablation} further demonstrates this by reproducing \Cref{ablation} with different choices of $q$. See \Cref{SIablationA} for detailed comparison of fine-grained risk components (i.e. $B , V _{\Xbf}, V _{\Xbf,\epbf}, V_{\epbf}$) of these predictors, and \Cref{SIablationC} for a series of ablation studies on other model parameters ($q$, SNR, and common structure in $\Bstar$).

\section{Structures of the Optimal Representation}
\label{sec:structure}
In the previous section, we introduced the end-to-end predictor (EEP). This predictor uses a feature matrix \(\Bhat = \Bhat^{\mathrm{opt}}\) and regularization parameters \(\bm{\lambda} = \bm{\lambda}^{\mathrm{opt}}\) chosen to minimize the asymptotic downstream risk averaged over the downstream ensemble \(\Rcal^{\mathsf{avg}}(\Bhat, \bm{\lambda}, \Sigmabf, \Bstar)\). In this section, we explore the structure of the optimal feature matrix and its connection to the bias-variance decomposition..

\subsection{Aligned, Spectrum-Only Case and Connection to PCR}
\label{sectionSO}

Understanding the solution to Eq.~\eqref{avgg}  is challenging in general. To build intuition, we first restrict to optimizing only the eigenvalues of $\Bhat$ and fixing its eigenvectors to align with those of $\Sigmabf$. Recall Eq.~\eqref{BhatSVD} and denote the eigen-decompositions by
\begin{equation*}
    \Sigmabf=\sum_{i=1}^p \eta_i \cdot \ubf_i \ubf_i^\top\qquad\text{and}\qquad \Bhat {\Bhat}^\top = \sum_{i=1}^p \dhat_i^2 \cdot \qbf_i \qbf_i^\top
\end{equation*}
for orthonormal eigenbases $\{\ubf_i\}_{i=1}^p$ and $\{\qbf_i\}_{i=1}^p$ (rows of $\Qhat$). For the discussion below, we restrict to the \textit{spectrum-only} case where we minimize $\Rcal^{\mathsf{avg}}$ over the eigenvalues $\{\dhat_i^2\}_{i=1}^p$ and regularization parameters $\bm{\lambda}$ while fixing the eigenvectors
\begin{equation}\label{matchbase}
    \qbf_i=\ubf_i,\qquad \forall i\in \{1,\ldots,p\}.
\end{equation}
The vectors $\ubf_i$ are then used to specify $\Bhat$ in this restricted setting, where $\Bhat$ is only able to reweight these fixed directions. 

Note that the predictor $\hat{y}$ in \Cref{Downstream} may be written as
$
    \yhat  (\xnewbf) = \tilde{\xbf}_{\mathrm{new}}^{\top} \qty({{\tilde{\Xbf}} }  {{\tilde{\Xbf}} })^{+}  {{\tilde{\Xbf}}^{\top}} \ybf , 
$
where $\tilde{\xbf}_{\mathrm{new}}:={\xnewbf^{\top}}\Gammahat^{-1/2}$ and $\tilde{\Xbf} :=\Xbf \Gammahat^{-1/2}$. The matrix quantity $\Gammahat^{-1/2}=\sum_{i=1}^{p} \rhat^{-1/2}_i \ubf_i \ubf_i^\top$ can then be interpreted as a linear featurization on the data matrix $\Xbf $ before regressing on $\ybf $. This has a natural connection to classical principal component regression (PCR), which projects to the top-$k$ eigenvectors of $\Sigmabf$ ($k<n$), commonly referred to as principal components (PCs), via the transform  $\Gammahat^{-1/2}=\sum_{i=1}^{k} \ubf_i \ubf_i^\top$. Therefore, the optimization of $\rhat_i$ may be understood as a generalization of PCR with soft feature selection: larger values of $\hat{r}_i$ deemphasize the corresponding PC $\mathbf{u}_i$, with $\hat{r}_i = \infty$ signifying that $\mathbf{u}_i$ is not selected. Indeed, choosing $\rhat_i=1,\forall i\le k$ and $\rhat_i=\infty,\forall i> k$ recovers classical PCR.

Observe from Eq.~\eqref{avgg} that $\Rcal^{\mathsf{avg}}$ is a simple function of the variance-component $\mathcal{V}$ and the bias-component $\mathfrak{B}^{\mathsf{avg}}$. We thus first study the properties of $\mathcal{V}$ and $\mathfrak{B}^{\mathsf{avg}}$ with respect to $\{\dhat_i^2\}_{i=1}^p$ and $\bm{\lambda}$. The solutions of the separated optimization problems
%\begin{equation*}
    $\min_{\{\rhat_i\}_{i=1}^p} \mathcal{V} \quad\text{and}\quad \min_{\{\rhat_i\}_{i=1}^p} \mathfrak{B}^{\mathsf{avg}}$
%\end{equation*}
can be characterized explicitly, using the results Propositions \ref{Equivalence}-\ref{Convexity} derived in the appendix, and a careful analysis of the Karush–Kuhn–Tucker (KKT) conditions.   

In the aligned spectrum case, we can write the numerator of $\mathfrak{B}^{\mathsf{avg}}$ as $\eta_i \theta_i$, where
\begin{equation}\label{vsdef}
    \theta_i:=\ubf^\top_i \Bstar \Sigmabf_{\alphastar} {\Bstar}^\top \ubf_i, \quad i=1,\ldots,p.
\end{equation}
Roughly, the coefficients $\theta_i$ track alignment of $\Bstar$ to $\ubf_i$ when $\Sigmabf_{\alphastar}$ is non-singular, whereas $\eta_i$ are the eigenvalues of $\Sigmabf$ and track alignment of $\Sigmabf$ to $\ubf_i$. Without loss of generality, we assume that $\eta_i\theta_i$ is nonincreasing. We denote
\begin{equation}\label{pos}
    h_1:=\abs{\{i: \eta_i\theta_i \neq 0\}},
\end{equation}
so that $\eta_i\theta_i=0$ for all $i> h_1$. Note that since $\eta_i=0$ implies $\eta_i\theta_i=0$, we have $h_1\leq h=\rank(\Sigmabf)\leq p$. When $h_1 < n$, there is low-dimensional structure in the problem despite the ambient dimension $p$. The analysis below reveals an interesting \textit{phase transition} phenomenon that to the best of our knowledge has not been known in prior work: classical PCR is optimal when $h_1 < n$, whereas for $h_1 > n$, it is preferable to employ soft selection on the PCs. 

In order to characterize the optimum for $\mathfrak{B}^{\mathsf{avg}} $ in the case where $n<h_1$, we introduce $h_0\in\{n,\ldots,h_1\}$, which is the unique integer such that $\tilde{h} \leq h_0$ if and only if
\begin{equation}\label{ineqho}
    \frac{1}{\hind-n} \sum_{i=1}^{\hind} \frac{\eta_{\hind}\theta_{\hind}}{\eta_i\theta_i}\ge 1.
\end{equation}
We prove the existence and uniqueness of $h_0$ and the following theorem in \Cref{app:characterizesol}.

\begin{theorem}\label{thm2}
    Suppose $h>n$ and Eq.~\eqref{matchbase} holds. Then, $\mathcal{V}$ is minimized by $\rhat_i = c\eta_i$ for $i\leq h$ and any finite $c>0$, to the optimal value $(h/n-1)^{-1}$. Meanwhile, the minimum of $\mathfrak{B}^{\mathsf{avg}}$ undergoes a phase transition w.r.t. $h_1$: 
    \begin{itemize}
        \item [(i)] \textbf{Soft-selection regime.} When $h_1>n$, $\mathfrak{B}^{\mathsf{avg}}$ is minimized by
    \begin{equation}\label{rewrite}
        \rhat_i = \begin{cases}
             c \eta_i \qty(\frac{1}{\ho-n} \sum_{j=1}^{\ho} \frac{\eta_i\theta_i}{\eta_j\theta_j}-1)^{-1} & \text{for } i\leq h_0 \\
            \infty & \text{for } h_0<i\leq h
        \end{cases},
    \end{equation}
    for any finite $c>0$, to the optimal value 
    $q^{-1} \sum_{\ho<i\le h_1} \eta_i\theta_i +\qty(\frac{q}{(\ho-n)^2} \sum_{i\le \ho} \frac{1}{\eta_i\theta_i} )^{-1}.$ 
    \item [(ii)] \textbf{Hard-selection regime.} When $h_1\leq n$, $\mathfrak{B}^{\mathsf{avg}}$ is minimized by
    \begin{equation}\label{PCRest}
        \rhat_i = \begin{cases}
             c_i  & \text{for } i\leq h_1 \\
            \infty & \text{for } h_1<i\leq h
        \end{cases},
    \end{equation}
    for any finite $c_i\ge 0$, to the optimal value zero. 
    \vspace{-3mm}
    \end{itemize}
    The values of $\rhat_i$ for $i>h$ can be set arbitrarily when minimizing both $\mathcal{V}$ and $\mathfrak{B}^{\mathsf{avg}}$.
\end{theorem}

\begin{remark}\label{generalA1}
    Choosing $\qbf_i \gets \ubf_i,\forall i$ as in Eq.~\eqref{matchbase} and $\rhat_i \gets c \eta_i,\forall i\le h$ as in \Cref{thm2} is in fact optimal for minimizing $\mathcal{V}$, without a priori restrictions $\qbf_i \gets \ubf_i,\forall i$. See \Cref{propgenA} for details.
\end{remark}

\begin{remark}[Convexity and Optimization]
    We defer the algorithmic question of minimizing $\Rcal^{\mathsf{avg}}$ w.r.t. $\qty{\rhat_i}_{i=1}^p$ to \Cref{sectionOptConv}, where we make further observations on the convexity of $\mathcal{V}$ and $\mathfrak{B}^\mathsf{avg}$ and provide novel convex programs for optimizing relaxations of $\Rcal^\mathsf{avg}$ and $\Rcal^\mathsf{worst}$. 
\end{remark}

% \begin{remark}[Technical Novelty]
% Our proof strategy follows the following steps: (i) observing the convexity in the bias-component $\mathfrak{B}^\mathsf{avg}$ and the variance-component $\mathcal{V}$, (ii) casting the fixed point equation \Cref{fp} as an equality constraint, (iii) deriving the optimal solution from the KKT conditions. This approach may be of independent interest for studying optimality question in other statistics or machine learning models where the risk asymptotics involve fixed point equations \cite{sur2019modern,bayati2011lasso,donoho2016high,xu2019number}. 
% \end{remark}

We first interpret the minimizer of $\mathfrak{B}^{\mathsf{avg}}$. Firstly, we note that the solution $\{\rhat_i\}$ is in fact sparse in both the soft-selection regime and strong-selection regime even in absence of sparsity-inducing priors or penalties. In the soft-selection regime, note that for $i\leq h_0$ we may write Eq.~\eqref{rewrite} as
\begin{equation}\label{rialternative}
        \rhat_i \gets c \qty(\frac{\theta_i}{\ho-n} \sum_{j\leq h_0,j\neq i} \frac{1}{\eta_j} \frac{1}{\theta_j}-\left( 1 - \frac{1}{h_0-n}\right)\frac{1}{\eta_i})^{-1}.
\end{equation}
Observe that $\rhat_i$ decreases as either $\eta_i$ or $\theta_i$ increase. This means that $\rhat_i$ will emphasize directions that align strongly with either the data $\Sigmabf$ or the ground-truth features $\Bstar$. Our result implies that bias is reduced the most by selecting PCs that align to both $\Sigmabf$ and $\Bstar$. This is in contrast to classical PCR, which simply selects the $k$-top PCs with equal weight. Interestingly, $\ubf_i$'s with comparatively small (but still possibly nonzero) $\eta_i\theta_i$ are completely discarded, i.e. $\hat{r}_i=\infty$ for $h_0<i\le h$, suggesting that the rank of the optimized featurization $\Bhat$ may be less than the intrinsic problem dimension $h_1$. 

In the hard-selection regime, our result suggests discarding ($\hat{r}_i=\infty$) all PCs where $\eta_i$ or $\theta_i$ is zero and selecting the rest with arbitrary (non-zero) weights. This includes the classical PCR, with $h_1$ being the optimal number of PCs to retain. Unlike in the soft-selection regime, the exact values of $\eta_i\theta_i$ and selection weights are inconsequential as long as they are non-zero, hence ``hard selection''. In contrast to the soft-selection regime, the rank of optimized featurization $\Bhat$ is precisely the intrinsic problem dimension $h_1$. Note also that the bias can be completely removed in this regime.

Recall that the total variance is proportional to $\mathcal{V}$, so \Cref{thm2} shows that to control the variance, we should weight the PCs proportional to $\eta_i$ regardless of $\theta_i$. Larger $\eta_i$ therefore have large $\hat{r}_i$ and are deemphasized, which is in contrast to classical PCR.

Optimizing $\Bhat$ against $\Rcal^{\mathsf{avg}}$ is then a trade-off of between optimizing $\mathcal{V}$ and $\mathfrak{B}^{\mathsf{avg}}$, or equivalently finding a way to balance bias and variance. Thus, minimizers of  $\Rcal^{\mathsf{avg}}$ typically combine the structures discussed above. 

\subsection{Fully-Optimized Representation}

\begin{figure*}[t]
  \centering
  \includegraphics[width=0.47\columnwidth] {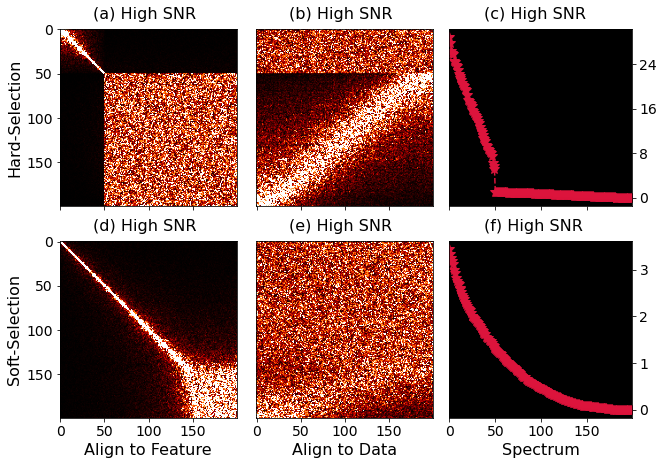} 
  \includegraphics[width=0.46\columnwidth] {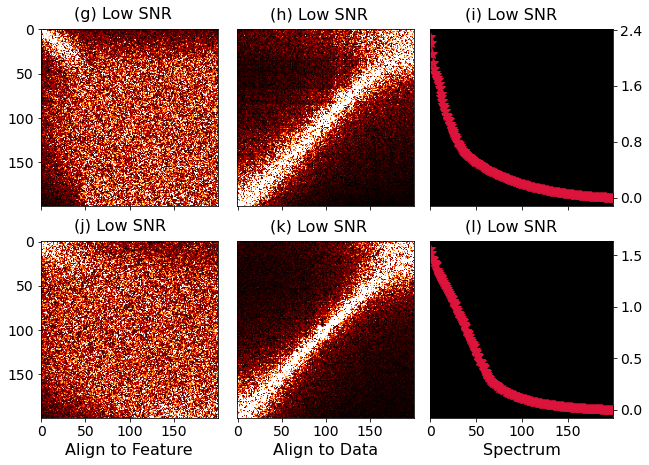} 
  \hspace{0mm}

  \caption{\textbf{(a), (d), (g), (j)}: Heat map of the matrix $\mathbf{M}\in \R^{p\times p}, \mathbf{M}_{ij}=\qbf_i^\top \mathbf{q} ^\star_j
 $ depicting the alignment between eigenvectors $\qty{\qbf_i}_{i=1}^p$ of $\Bhat {\Bhat}^\top$ and eigenvectors $\qty{\mathbf{q}_i ^\star}_{i=1}^p$ of ground-truth feature $\Bstar {\Bstar}^\top$. \textbf{(b), (e), (h), (k)}: Heat map of the matrix $\mathbf{N}\in \R^{p\times p}, \mathbf{N}_{ij}=\qbf_i^\top \ubf_j
 $ depicting the alignment between eigenvectors $\qty{\qbf_i}_{i=1}^p$ of $\Bhat {\Bhat}^\top$ and eigenvectors $\qty{\ubf_i}_{i=1}^p$ and eigenvectors of data covariance $\Sigmabf$. \textbf{(c), (f), (i), (l)}: eigenvalues of $\Bhat \Bhat^\top$. Top row depicts the regime $q=50<n=100$ and bottom row depicts the regime $q=150>n=100$.  Left panel is for $\mathsf{SNR}=\norm{\st}_2/\sigma=25$ and right panel is for $\mathsf{SNR}=0.5$. Throughout, we set $n=100, p=200, \sigma^2=1, \Sigmabf_{ij}=0.5^{|i-j|}$ and draw $\Bstar \sim N(\bm{0}, \Ibf_p \times \Ibf_q), \alphastar\sim N(\bm{0}, \mathfrak{c}\cdot\Ibf_q)$ where $\mathfrak{c}$ is chosen to adjust $\mathsf{SNR}$ to the specified levels.}
  \label{heat}
\end{figure*}

Now we would like to lift the assumption \eqref{matchbase} of aligned eigenvectors. Although intractable analytically in this case, $\Rcal^{\mathsf{avg}}$ can be optimized using backpropagation (see \Cref{app:obtainfully}). Many of the observations from \Cref{sectionSO} carry over to this case in \Cref{heat}, through a combination of eigenvector alignment and spectrum adjustment. By changing $q$, we can transition between the hard- and soft-selection regimes. Similarly, by adjusting the SNR, we can change the relative contributions of $\mathcal{V}$ and $\mathfrak{B}^\mathsf{avg}$ to the risk (see Eq.~\eqref{avgg}). When the SNR is high, $\mathfrak{B}^\mathsf{avg}$ dominates $\mathcal{V}$, and vice versa.

In the high-SNR regime, the bias component \(\mathfrak{B}^\mathsf{avg}\) dominates. Here, \(\Bhat\)'s top-\(q\) eigenvectors align with the leading eigenvectors of \(\BBstar\) to minimize bias, corresponding to the feature selection phenomenon predicted in \Cref{thm2}. At the same time, there is a pronounced alignment with the bottom eigenvectors of \(\Sigmabf\) beyond the \(q\)th eigenvector, reflecting the impact of the \(\mathcal{V}\) component. In the hard selection regime (with \(q=50\)), a clear spectral gap is present at \(q=50\), along with a sharp transition in the alignment pattern before and after the \(q\)th eigenvector. In contrast, the soft selection regime (with \(q=150\)) does not have a spectral gap and transition in eigenvector alignment pattern is less sharp. This behavior aligns with the hard and soft-selection phase transition predicted in \Cref{thm2}. In the low-SNR regime, the \(\mathcal{V}\) component becomes dominant. As a result, we observe a reverse alignment pattern with eigenvectors of \(\Sigmabf\), as expected from \Cref{thm2}.

The above results interpolate between optimizing $\mathcal{V}$ and $\mathfrak{B}^{\mathsf{avg}}$. We now investigate $\Bhat$'s behavior when optimizing each individually. We reproduce \Cref{heat} under exact settings except that we optimize for $\mathfrak{B}^{\mathsf{avg}}$  only. This gives us  \Cref{heatBiasVar}, (a)-(f). We do the same but optimize for $\mathcal{V}$ only. This gives us \Cref{heatBiasVar}, (g)-(l). The spectrum-gap at $q=50$ in \Cref{heat}, (c) may be attributed to ${\mathfrak{B}}^{\mathsf{avg}}$, which, when optimized alone, yields exactly $q=50$ non-zero eigenvalues (\Cref{heatBiasVar}, (c)), consistent with hard-selection behavior in the spectrum-only case. In contrast, minimizing $\mathfrak{B}^{\mathsf{avg}}$ in the soft-selection regime yields no spectrum-gap at $q$; instead, the number of non-zero eigenvalues ($\sim$120) is smaller than $q=150$ (\Cref{heatBiasVar}, (f)), consistent with theoretical prediction in the spectrum-only case (i.e. $\mathrm{rank}(\Bhat)\equiv \ho\le h_1$). Observe in (h), (k), the eigenvectors of $\Bhat \Bhat^\top$ align with eigenvectors of $\Sigmabf$ in exactly the reverse order. 
\begin{figure}
  \centering
  \includegraphics[width=0.501\columnwidth] {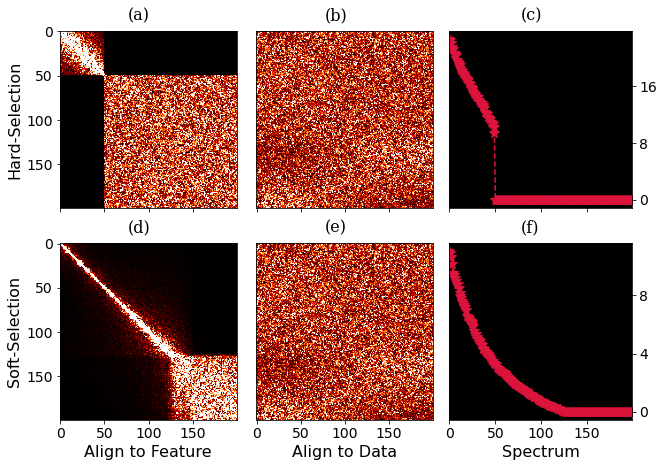} 
  \includegraphics[width=0.49\columnwidth] {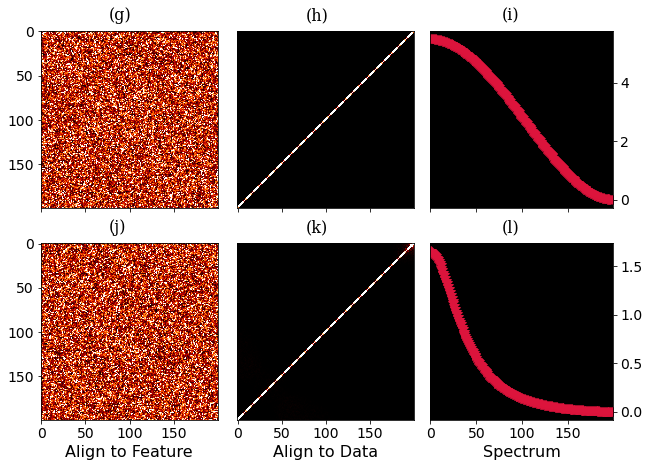} 
  \vspace{-5mm}
  \caption{\textbf{(a)-(f)}: Optimized for $\mathfrak{B}^\mathsf{avg}$ only. \textbf{(g)-(l)}: Optimized for $\mathcal{V}$ only. Otherwise, same settings as \Cref{heat}}
  \label{heatBiasVar}
\end{figure}

A final remark is on the role of $\Sigmabf$ when $\mathfrak{B}^\mathsf{avg}$ is minimized alone. \Cref{thm2} predicts that when minimizing $\mathfrak{B}^\mathsf{avg}$ in the soft-selection regime, $\Bhat$ should also have a tendency to align to top eigenvectors of $\Sigmabf$. This effect is overshadowed by the alignment to $\Bstar$ in \Cref{heatBiasVar}. However, once we set $\Bstar$ to be non-informative, i.e. $\Bstar {\Bstar}^\top =\Ibf_p$, we can indeed observe this effect emerging. 
\begin{figure}
  \centering
  \includegraphics[width=0.6\columnwidth] {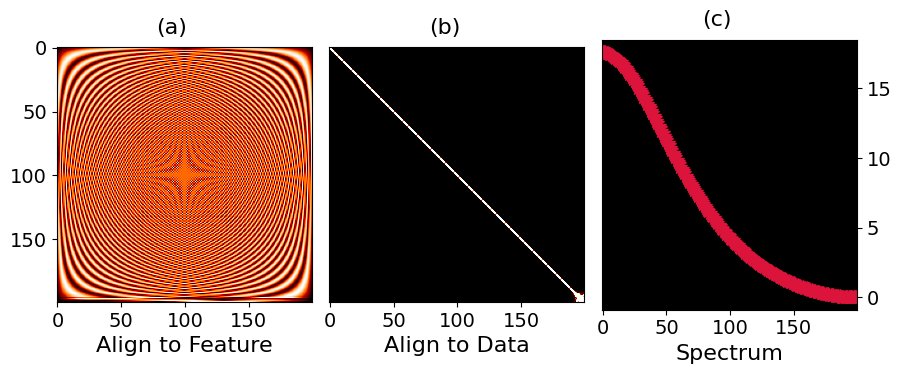} 
  \caption{Optimize $\mathfrak{B}^\mathsf{avg}$ only. Set $\Bstar {\Bstar}^\top =\Ibf_p$ and otherwise same as \Cref{heat}}
  \label{heatSigma}
\end{figure}
Observe that in \Cref{heatSigma} (b), the top eigenvectors of $\Bhat$ aligns precisely with top eigenvectors of $\Sigmabf$ as predicted by \eqref{rialternative} in the spectrum-only case.

% \paragraph{Notation.} Expectation over a random variable $X$ is denoted by $\E_{X}$. We use $(\cdot)^+$ to denote the Moore-Penrose pseudo-inverse, $\norm{\cdot}_2$ to denote the $\ell_2$ vector norm, $\norm{\cdot}_\mathrm{op}$ to denote the matrix operator norm, $\bm{1}_{p\times 1}$ to denote a length-$p$ column vector of $1$s, $\eigtop(\mathbf{A}), \eigbot(\mathbf{A})$ to denote the maximum and minimum eigenvalues of symmetric real-valued matrix $\mathbf{A}$,  and $\stackrel{L}{=}$ to denote equality in distribution. 

\section{Proofs of Main Results} 
Before presenting the proofs, we briefly outline the key technical contributions of this paper. First, we use random matrix theory to derive asymptotic characterization of fine-grained bias and variance. We first note that that the fine-grained bias-variance components relate to the classical decomposition as follows:  
\begin{equation}
    \E_{\Xbf } \BSC = B + V_{\Xbf}, \qquad \E_{\Xbf } \VSC = V_{\Xbf,\epbf} + V_{\epbf}. \label{DCb}
\end{equation}
The asymptotic expressions for \( \E_{\Xbf } \BSC \) and \( \E_{\Xbf } \VSC \) follow from the analysis in \cite{hastie2022surprises}, and the term \( V_{\epbf} \) vanishes due to symmetry. Thus, it suffices to derive the limiting expression for fine-grained bias \( B \). We observed that \( B \) involves two independent copies of a random matrix quantity rather than a single instance and applied resolvent methods separately to each copy (see \eqref{twocopies}).  

The second major contribution lies in the optimality analysis of asymptotic risk in the aligned case. A key insight is that the limiting expressions for fine-grained bias and variance can be reformulated as a convex optimization problem over the spectrum of the feature matrix, where the fixed point equation \eqref{fp} emerging from RMT is cast as a linear constraint. This reduces the analysis to solving the KKT conditions of the convex problem. The primary challenge is in conjecturing a solution that satisfies these conditions—particularly when optimizing \( \Bhat \), where the solution exhibits sparsity and undergoes a phase transition. Our approach may be of independent interest for studying optimality question in other statistics or machine learning models where the risk asymptotics involve fixed point equations \cite{sur2019modern,bayati2011lasso,donoho2016high,xu2019number}. 

% {\color{red} Isolate main results and for each say something about what is challeging for each. Technical contribution...}

\subsection{Proof of \Cref{thm}}\label{app:riskchar}
We prove \Cref{thm} in this section. As by-products, we provide bounds on asymptotic quantities such as $b_0, \mathcal{V}, \mathfrak{B} $ that will also be useful in later sections; we also provide characterization of fine-grained bias-variance decomposition.

\subsubsection{Preliminary Bounds on Asymptotic Quantities}

Recall that we used $\eta_{\min}^+$ to denote the smallest non-zero eigenvalue of $\Sigmabf$. Below we also denote the largest eigenvalue of $\Sigmabf$ as $\eta_{\max}\equiv \normop{\Sigmabf}$. Similarly, we use notation $${\that}^+_{\min}:=\min\{\that_i: \that_i>0, i=1,\ldots,p\}, \qquad \that_{\max}:=\max \{\that_i: i=1,\ldots,p\}.$$

Before we prove \Cref{thm}, we first prove a useful lemma that bounds ${\that}^+_{\min}$ and $\that_{\max}$. 
\begin{lemma}\label{bdthat}
    When $\Gammahat$ is non-singular, we have that
    \begin{equation*}
        {\that}^+_{\min}\ge \eta^+_{\min} \cdot \normop{\Gammahat}^{-1}, \qquad \that_{\max}\le \eta_{\max} \cdot \normop{\Gammahat^{-1}}.
    \end{equation*}
\end{lemma}
\begin{proof}[Proof of \Cref{bdthat}]
    The second relation follows from sub-multiplicativity of operator norm
    \begin{equation*}
        \that_{\max}=\normop{\Gammahat^{-1/2} \Sigmabf \Gammahat^{-1/2}
        } \le \normop{\Gammahat^{-1/2} 
        }^2 \normop{ \Sigmabf 
        }=\eta_{\max} \cdot \normop{\Gammahat^{-1}}
    \end{equation*}
    where the last equality follows from
    \begin{equation*}
        \eta_{\max}=\normop{\Sigmabf},\qquad \normop{\Gammahat^{-1/2}}=\eigtop \qty (\Gammahat^{-1/2})=\eigtop \qty (\Gammahat^{-1})^{1/2}.
    \end{equation*}
We now show the first claim is true if $\Sigmabf$ is non-singular. Using variational representation of eigenvalues, we have
\begin{equation*}
\begin{gathered}
\that_{\min }^{+}=\min _{\mathbf{v} \in \mathbb{R}^p:\|\mathbf{v}\|_2^2 =1} \mathbf{v}^{\top} \Gammahat^{-\frac{1}{2}} \Sigmabf \Gammahat^{-\frac{1}{2}} \mathbf{v}=\left\|\mathbf{v}^{\top} \Gammahat^{-\frac{1}{2}}\right\|_2^2 \cdot \min _{\mathbf{w} \in \mathbb{R}^p:\|\mathbf{w}\|_2^2=1} \mathbf{w}^{\top} \Sigmabf \mathbf{w}=\mathbf{v}^{\top} \Gammahat^{-1} \mathbf{v} \cdot \eta_{\min }^{+} \\
\geq \eigbot \left(\Gammahat^{-1}\right) \cdot \eta_{\min }^{+}=\eta_{\min }^{+} \cdot\|\Gammahat\|_{\mathrm{op}}^{-1}
\end{gathered}
\end{equation*}
where in the first equality we used that $\hat{t}_i>0, \forall i \in\{1, \ldots, p\}$ when $\Sigmabf$ is non-singular and thus $\that_{\min }^{+}$is simply the smallest eigenvalue.

Now we assume that $\Sigmabf$ is singular with $\operatorname{rank}(\Sigmabf)=h<p$. Let columns of $\mathbf{V} \in \mathbb{R}^{p \times(p-h)}$ consist of eigenvectors of $\Sigmabf$ that have zero eigenvalues and columns of $\mathbf{U} \in \mathbb{R}^{p \times h}$ consists of the rest of the eigenvectors. The compact form of eigen-decomposition of $\Sigmabf$ is then
\begin{equation}\label{compasigmeg}
\Sigmabf=\mathbf{U} \bm{\Lambda}_{+} \mathbf{U}^{\top}
\end{equation}
where $\bm{\Lambda}_{+}$ is diagonal matrix consisting of non-zero eigenvalues of $\Sigmabf$. In particular, the smallest entry of $\bm{\Lambda}_{+}$ is $\eta^+_{\min}$. 

Now we claim that eigenvectors of $\Gammahat^{-\frac{1}{2}} \Sigmabf \Gammahat^{-\frac{1}{2}}$ corresponding to the zero eigenvalue contained in the column space of $\Gammahat^{\frac{1}{2}} \mathbf{V}$. This follows from the fact that $\Gammahat^{-\frac{1}{2}} \Sigmabf \Gammahat^{-\frac{1}{2}}\left(\Gammahat^{\frac{1}{2}} \mathbf{V}\right)=0$ and that the columns of $\Gammahat^{\frac{1}{2}} \mathbf{V}$ are linearly independent. Since eigenvectors are orthonormal, this claim implies that any eigenvector $\hat{\mathbf{w}}_i$ of $\Gammahat^{-\frac{1}{2}} \Sigmabf \Gammahat^{-\frac{1}{2}}$ corresponding to a non-zero eigenvalue $\hat{t}_i>0$ must satisfy
\begin{equation*}
\hat{\mathbf{w}}_i^{\top} \Gammahat^{\frac{1}{2}} \mathbf{V}=0 .
\end{equation*}

This implies that $\hat{\mathbf{w}}_i$ must be contained in column space of $\Gammahat^{-\frac{1}{2}} \mathbf{U}$ since any vector in the column space of $\Gammahat^{-\frac{1}{2}} \mathbf{U}$ must be orthogonal to any vector in the column space of $\Gammahat^{\frac{1}{2}} \mathbf{V}$. Using this fact and the variational representation of eigenvalues, we obtain that
\begin{equation*}
\hat{t}_{\min }^{+}=\min _{\mathbf{a} \in \mathbb{R}^h: \mathbf{a} \neq 0}\left(\frac{\Gammahat^{-\frac{1}{2}} \mathbf{U} \mathbf{a}}{\left\|\Gammahat^{-\frac{1}{2}} \mathbf{U} \mathbf{a}\right\|_2}\right)^{\top} \Gammahat^{-\frac{1}{2}} \Sigmabf \Gammahat^{-\frac{1}{2}}\left(\frac{\Gammahat^{-\frac{1}{2}} \mathbf{U} \mathbf{a}}{\left\|\Gammahat^{-\frac{1}{2}} \mathbf{U} \mathbf{a}\right\|_2}\right).
\end{equation*}
This and \eqref{compasigmeg} implies that
\begin{equation*}
\begin{aligned}
    t_{\min }^{+}&=\min _{\mathbf{a} \in \mathbb{R}^h: \mathbf{a}\neq 0}\left(\frac{\Gammahat^{-\frac{1}{2}} \mathbf{U} \mathbf{a}}{\left\|\Gammahat^{-\frac{1}{2}} \mathbf{U} \mathbf{a}\right\|_2}\right)^{\top} \Gammahat^{-\frac{1}{2}} \mathbf{U} \mathbf{\Lambda}_{+} \mathbf{U}^{\top} \Gammahat^{-\frac{1}{2}}\left(\frac{\Gammahat^{-\frac{1}{2}} \mathbf{U} \mathbf{a}}{\left\|\Gammahat^{-\frac{1}{2}} \mathbf{U} \mathbf{a}\right\|_2}\right) \\
&=\min _{\mathbf{a} \in \mathbb{R}^h: \mathbf{a}\neq 0}\left(\frac{\mathbf{U}^{\top} \Gammahat^{-1} \mathbf{U} \mathbf{a}}{\left\|\Gammahat^{-\frac{1}{2}} \mathbf{U} \mathbf{a}\right\|_2}\right)^{\top} \Lambda_{+}\left(\frac{\mathbf{U}^{\top} \Gammahat^{-1} \mathbf{U} \mathbf{a}}{\left\|\Gammahat^{-\frac{1}{2}} \mathbf{U} \mathbf{a}\right\|_2}\right) \geq \min _{\mathbf{a} \in \mathbb{R}^h: \mathbf{a}\neq 0} \frac{\eta_{\min }^{+}\left\|\mathbf{U}^{\top} \Gammahat^{-1} \mathbf{U} \mathbf{a}\right\|_2^2}{\left\|\Gammahat^{-\frac{1}{2}} \mathbf{U} \mathbf{a}\right\|_2^2} \\
&=\eta_{\min }^{+} \cdot \min _{\mathbf{a} \in \mathbb{R}^h: \mathbf{a}\neq 0} \frac{\mathbf{a}^{\top}\left(\mathbf{U}^{\top} \Gammahat^{-1} \mathbf{U}\right)^2 \mathbf{a}}{\mathbf{a}^{\top} \mathbf{U}^{\top} \Gammahat^{-1} \mathbf{U} \mathbf{a}} \geq \eta_{\min }^{+} \cdot\|\Gammahat\|_{\mathrm{op}}^{-1}
\end{aligned}
\end{equation*}

We now explain why the last inequality holds. We write eigen-decomposition of $\mathbf{U}^{\top} \Gammahat^{-1} \mathbf{U}$ as $\mathbf{U}^{\top} \Gammahat^{-1} \mathbf{U}=\sum_{i \in[h]} \lambda_i \mathbf{v}_i \mathbf{v}_i^{\top}$. The smallest eigenvalue $\lambda_{\min }$ satisfies
\begin{equation*}
\lambda_{\min }=\min _{\mathbf{b} \in \mathbb{R}^n:\|\mathbf{b}\|_2^2=1} \mathbf{b}^{\top} \mathbf{U}^{\top} \Gammahat^{-1} \mathbf{U} \mathbf{b} \geq \min _{\mathbf{b} \in \mathbb{R}^n: \|\mathbf{b}\|_2^2=1} \mathbf{b}^{\top} \Gammahat^{-1} \mathbf{b} \geq \eigbot \left(\Gammahat^{-1}\right)
\end{equation*}
where we used in first inequality that $\|\mathbf{U} \mathbf{b}\|_2^2=\|\mathbf{b}\|_2^2=1$. Now for any $\mathbf{a} \in \mathbb{R}^h, \mathbf{a} \neq 0$, let $\omega_i:=\frac{\lambda_i\left(\mathbf{a}^{\top}\mathbf{v}_i\right)^2}{\sum_i \lambda_i\left(\mathbf{a}^{\top}\mathbf{v}_i\right)^2}$ be non-negative weights summing up to 1 . Then,
\begin{equation*}
\frac{\mathbf{a}^{\top}\left(\mathbf{U}^{\top} \Gammahat^{-1} \mathbf{U}\right)^2 \mathbf{a}}{\mathbf{a}^{\top} \mathbf{U}^{\top} \Gammahat^{-1} \mathbf{U} \mathbf{a}}=\frac{\sum_i \lambda_i^2\left(\mathbf{a}^{\top}\mathbf{v}_i\right)^2}{\sum_i \lambda_i\left(\mathbf{a}^{\top} \mathbf{v}_i\right)^2}=\sum_i \lambda_i \omega_i \geq \lambda_{\min } \geq \sigma_{\min }\left(\Gammahat^{-1}\right)=\|\Gammahat\|_{\mathrm{op}}^{-1}
\end{equation*}
as required. This concludes the proof. 
\end{proof}

Using \Cref{bdthat}, we prove another lemma that bounds the size of $b_0$ defined in \eqref{fp} and $\mathcal{V}, \mathfrak{B} $ defined in \eqref{A1A2}. 

\begin{lemma}\label{bA1A2size}
    For We have that 
    \begin{equation*}
        b_0 \ge \frac{1}{\frac{h}{n}-1}\cdot \frac{1}{\that_{\max}}, \quad \mathcal{V} \le \qty(\qty(\frac{\that_{\max}}{\that_{\min}^+}-1)\qty(\frac{h}{n}-1)^{-1}+1), \quad \mathfrak{B} \le \norm{\Sigmabf^{1/2} \st}_2^2.
    \end{equation*}
\end{lemma}
\begin{proof}[Proof of \Cref{bA1A2size}]
We first prove the lower bound of $b_0$. We will use notation 
\begin{equation}
    x_i:=\frac{1}{1+\that_i b_0}.
\end{equation}
We may rewrite \eqref{fp} as
\begin{equation}\label{fixx}
    h-n=\sum_{i\in H} x_i.
\end{equation}
We have from \eqref{fixx} that
\begin{equation*}
    \frac{h-n}{h}=\frac{1}{h} \sum_{i \in H} x_i \geq \frac{1}{1+\hat{t}_{\max } b_0} \Rightarrow b_0 \geq \frac{n}{h-n} \frac{1}{\hat{t}_{\max }}
\end{equation*}
as required. 

We now prove the upper bound on $\mathcal{V}$. We have from definition of $H$ that for any $i\in H$, 
\begin{equation*}
    x_i^2 \le \frac{1}{1+\that_{\min}^+b_0} x_i
\end{equation*}
which, along with the lower bound on $b_0$ proved above, implies that 
\begin{equation}\label{ubx2}
\frac{1}{h}\sum_{i \in H} x_i^2 \leq \frac{1}{1+\hat{t}_{\min }^{+} b_0} \sum_{i \in H} x_i \leq \frac{1}{1+\frac{n}{h-n} \frac{\hat{t}_{\min }^{+}}{\hat{t}_{\max }}}\cdot \frac{1}{h} \sum_{i \in H} x_i=\frac{1}{1+\frac{n}{h-n} \frac{\hat{t}_{\min }^{+}}{\hat{t}_{\max }}} \cdot \qty(1-\frac{n}{h}).
\end{equation}
Note the following identities
\begin{equation}\label{idA1}
    \begin{aligned}
        &\bigsum_{i \in H} \frac{\left(\that_i b_0\right)^2}{\left(1+\that_i b_0\right)^2} = \sum_{i\in H} (1-x_i)^2 = h+\sum_{i\in H} x_i^2 -2\sum_{i\in H} x_i=2n-h+ \sum_{i\in H} x_i^2\\
        & \bigsum_{i \in H} \frac{\that_i b_0}{\left(1+\that_i b_0\right)^2}=\sum_{i\in H} x_i -x_i^2 = h-n-\sum_{i\in H} x_i^2.
    \end{aligned}
\end{equation}
Using these, we may obtain that
\begin{equation*}
    \mathcal{V}=\frac{\frac{2 n}{h}-1+\frac{1}{h} \sum_i x_i^2}{1-\frac{n}{h}-\frac{1}{h} \sum_i x_i^2}.
\end{equation*}
The upper bound of $\mathcal{V}$ follows from \eqref{ubx2} and the observation that the RHS of the above is increasing in $\frac{1}{h}  \sum_i x_i^2$. 

Finally, we prove the upper bound on $\mathfrak{B} $. Using \eqref{eigprod} at the last equality below, we have
\begin{equation*}
    \mathfrak{B} ={\st}^\top \Gammahat^{1/2} \sum_{i\in H} \frac{\that_i\cdot \what_i \what_i^\top }{1+\that_i b_0} \Gammahat^{1/2} \st\le {\st}^\top \Gammahat^{1/2} \sum_{i\in H} {\that_i\cdot \what_i \what_i^\top } \Gammahat^{1/2} \st={\st}^\top \Sigmabf \st.
\end{equation*}
This concludes the proof. 
\end{proof}

\subsubsection{Characterization of Fine-Grained Bias-Variance Decomposition}
Before we prove \Cref{thm}, we first characterize the total risk and its fine-grained bias-variance decomposition. The following proposition relies on the main results from \cite{hastie2022surprises}.
\begin{proposition}\label{mMprop}
    Let \Cref{Assum} hold. 
\begin{itemize}
    \item [(i)] \textbf{Sample-deficient regime ($n<h$).} If in addition $1+M^{-1}<h/n<M$, we have that for any constant $D>0$, there exists $C=C(D,M)$ such that $$\abs{R -\Rcal }\leq Cn^{-1/7}\norm{{\st}}_2^2,\quad \abs{\BSC -\qty(1+\mathcal{V})\mathfrak{B} } \leq Cn^{-1/7}\norm{{\st}}_2^2,\quad \abs{\VSC -\sigma^2 \cdot \mathcal{V}}  \leq Cn^{-1/7}$$ with probability at least $1-Cn^{-D}$. 
    
    \item [(ii)] \textbf{Sample-rich regime ($h<n$).} If in addition $M^{-1}<h/n<1-M^{-1}$, we have $\BSC =0$ and for any constant $D>0$, there exists $C=C(D,M)$ such that 
    \begin{equation}\label{underV}
        \abs{R -\mathcal{U}}\le Cn^{-1/7}, \qquad \abs{\VSC -\mathcal{U}}\le Cn^{-1/7}
    \end{equation}
    with probability at least $1-Cn^{-D}$. 
\end{itemize}
\end{proposition}

\begin{proof}[Proof of \Cref{mMprop}]

Note that have that
\begin{equation*}
\begin{aligned}
\betahatbf  &=\Gammahat^{-\frac{1}{2}}\left(\Gammahat^{-\frac{1}{2}} {\Xbf}  {\Xbf}  \Gammahat^{-\frac{1}{2}}\right)^{+} \Gammahat^{-\frac{1}{2}} {\Xbf}  \mathbf{y}  \\
& =\Gammahat^{-\frac{1}{2}} \Pbf \Gammahat^{\frac{1}{2}} \st+\underbrace{\Gammahat^{-\frac{1}{2}}\left(\Gammahat^{-\frac{1}{2}} {\Xbf}  {\Xbf}  \Gammahat^{-\frac{1}{2}}\right)^{+} \Gammahat^{-\frac{1}{2}} {\Xbf}  {\epbf} }_{=: \hat{\boldsymbol{e}} }.
\end{aligned}
\end{equation*}
It follows that
\begin{equation}\label{VE}
\E_{\boldsymbol{\epbf} } \betahatbf =\Gammahat^{-\frac{1}{2}} {\Pbf}  \Gammahat^{\frac{1}{2}} \st, \quad \operatorname{Cov}_{\epbf }\left(\widehat{\boldsymbol{e}} \right)=\Gammahat^{-\frac{1}{2}}\left(\Gammahat^{-\frac{1}{2}} {\Xbf}  {\Xbf}  \Gammahat^{-\frac{1}{2}}\right)^{+} \Gammahat^{-\frac{1}{2}}
\end{equation}
Define
\begin{equation}\label{Pj}
{\Pbf} :=\left(\Gammahat^{-\frac{1}{2}} {\Xbf}  {\Xbf}  \Gammahat^{-\frac{1}{2}}\right)^{+} \Gammahat^{-\frac{1}{2}} {\Xbf}  {\Xbf}  \Gammahat^{-\frac{1}{2}}.
\end{equation}
Using \eqref{VE} and \eqref{Pj}, we obtain
\begin{equation*}
\begin{aligned}
B_{\mathsf{SC}}  &=\E_{{\epbf} , \xnewbf }\left(y_{\text {new }} -\E_{{\epbf} } {\yhat} \right)^2\\
&=\E_{{\epbf} , \xnewbf }\left({\xnewbf }^{\top}\left(\st-\E_{{\epbf} } \betahatbf \right)\right)^2 \\
& =\left(\Gammahat^{\frac{1}{2}} \st\right)^{\top}\left(\Ibf-{\Pbf} \right) \Gammahat^{-\frac{1}{2}} \Sigmabf \Gammahat^{-\frac{1}{2}}\left(\Ibf-{\Pbf} \right)\left(\Gammahat^{\frac{1}{2}} \st\right) \\
V_{\mathsf{SC}} &=\E_{\xnewbf } \mathbb{V}_{{\epbf} }\left({\yhat} \right)=\E_{\epbf , \xnewbf }\left({\xnewbf }^{\top}\left(\st-\E_{\epbf } \betahatbf \right)\right)^2 \\
& =\E_{{\epbf} }\left(\left(\st-\E_{{\epbf} } \betahatbf \right)^{\top} \Sigmabf\left(\st-\E_{{\epbf} } \betahatbf \right)\right)=\operatorname{Tr}\left(\operatorname{Cov}_{{\epbf} }\left(\hat{\boldsymbol{e}} \right) \Sigmabf\right) \\
& =\sigma^2 \operatorname{Tr}\left(\left(\Gammahat^{-\frac{1}{2}} {\Xbf}  {\Xbf}  \Gammahat^{-\frac{1}{2}}\right)^{+} \Gammahat^{-\frac{1}{2}} \Sigmabf \Gammahat^{-\frac{1}{2}}\right).
\end{aligned}
\end{equation*}

We now recall that $H$ denote the index subset of $\{1, \ldots, p\}$ for which $\hat{t}_i \neq 0$ and $h:=$ $|H|$. Below, we use $H(i)$ to denote the $i$-th element of $H$. We also let $\Zj:=\left[\mathbf{z}_1 , \ldots, \mathbf{z}_n \right] \in$ $\mathbb{R}^{n \times p}$ for $\left\{\mathbf{z}_i \right\}_{i=1}^n$ defined in \Cref{Assum} and note that $\mathbf{X} =\Zj \mathbf{\Sigma}^{1 / 2}$. Let us introduce some notations
\begin{equation*}
\begin{aligned}
& \What_h^{\top}:=\left[\what_{H(1)}, \ldots, \what_{H(h)}\right] \in \mathbb{R}^{p \times h}, \quad \Th:=\operatorname{diag}\left(\left[\hat{t}_{H(1)}, \ldots, \hat{t}_{H(h)}\right]\right) \in \mathbb{R}^{h \times h}, \\
& \Zhj:=\Zj \What_h^{\top} \in \mathbb{R}^{n \times h}, \quad \Ph :=\left(\Th^{1 / 2} {\Zhj}^{{\top}} \Zhj \Th^{1 / 2}\right)^{+} \Th^{1 / 2} \Zhj \Zhj \Th^{1 / 2} \in \R^{h\times h} \\
& \widehat{\boldsymbol{\Lambda}}_h:=\operatorname{diag}\left(\left[r\left(\dhat_{H(1)}^2\right), \ldots, r\left(\dhat_{H(h)}^2\right)\right]\right) \in \mathbb{R}^{h \times h}
\end{aligned}
\end{equation*}
where $\mathrm{diag}([x_1,\ldots,x_n])$ denotes diagonal matrix with diagonal entries $x_1,\ldots,x_n$ and $[\mathbf{x}_1,\ldots,\mathbf{x}_n]$ denotes matrix with columns $\mathbf{x}_1,\ldots,\mathbf{x}_n$. We also recall that $\left(\what_i\right)_{i=1}^p,\left(\hat{t}_i\right)_{i=1}^p$ are defined in \eqref{eigprod}, $\left(\dhat_i\right)_{i=1}^p$ is defined in \Cref{BhatSVD}, and $r(\cdot)$ is defined in \eqref{ri}.

Using the notation above, we have that
\begin{equation*}
\Gammahat^{-\frac{1}{2}} \Sigmabf \Gammahat^{-\frac{1}{2}}=\sum_{i \in H} \hat{t}_i \cdot \what_i \what_i^{\top}=\What_h^{\top} \Th \What_h
\end{equation*}
which implies that
\begin{equation*}
\mathbf{z}_i  \Sigmabf^{\frac{1}{2}} \Gammahat^{-\frac{1}{2}} \sim N\left(\bm{0}, \What_h^{\top} \Th \What_h\right) .
\end{equation*}

We thus have
\begin{equation*}
\Xbf  \Gammahat^{-\frac{1}{2}}=\Zj \Sigmabf^{\frac{1}{2}} \Gammahat^{-\frac{1}{2}} \stackrel{L}{=} \Zj \What_h^{\top} \Th^{\frac{1}{2}} \What_h \stackrel{L}{=} \Zhj \Th^{\frac{1}{2}} \What_h
\end{equation*}
where we note that $\Zhj$ has iid $N(0,1)$ entries since $\mathbf{z}_i  \What_h^{\top} \sim N\left(\bm{0}, \Ibf_h\right)$. It follows that
\begin{equation}\label{usefmapt}
\begin{aligned}
    \widehat{\mathbf{P}} &=\left(\Gammahat^{-\frac{1}{2}} {\Xbf }^\top \Xbf  \Gammahat^{-\frac{1}{2}}\right)^{+} \Gammahat^{-\frac{1}{2}} {\Xbf }^\top \Xbf  \Gammahat^{-\frac{1}{2}}\\ &\stackrel{L}{=} \What_h^{\top}\left(\Th^{\frac{1}{2}} \mathbf{Z}_h  \Zhj \Th^{\frac{1}{2}}\right)^{+} \Th^{\frac{1}{2}} \mathbf{Z}_h  \Zhj \Th^{\frac{1}{2}} \What_h \\
&=\What_h^{\top} \widehat{\mathbf{P}}_h  \What_h .
\end{aligned}
\end{equation}

Using \eqref{usefmapt} and $\What_h \What_h^{\top}=\Ibf_h$, we have
\begin{equation*}
\begin{gathered}
(\Ibf_p-\Pbf ) \Gammahat^{-\frac{1}{2}} \Sigmabf \Gammahat^{-\frac{1}{2}}(\Ibf_p-\Pbf ) \stackrel{L}{=}\left(\Ibf_h-\What_h^{\top} \Ph \What_h\right) \What_h^{\top} \Th \What_h\left(\Ibf_h-\What_h^{\top} \Ph \What_h\right) \\
=\What_h^{\top}\left[\Ibf_h-\Ph\right] \Th\left[\Ibf_h-\Ph\right] \What_h .
\end{gathered}
\end{equation*}

Therefore, 
\begin{equation}\label{BVred}
\begin{aligned}
& B_{\mathsf{SC}}  \stackrel{L}{=}\left(\What_h \Gammahat^{\frac{1}{2}} \st\right)^{\top}\left[\Ibf_h-\Ph\right] \Th\left[\Ibf_h-\Ph\right]\left(\What_h \Gammahat^{\frac{1}{2}} \st\right), \\
& V_{\mathsf{SC}}  \stackrel{L}{=} \sigma^2 \operatorname{Tr}\left(\left(\What_h^{\top} \Th^{\frac{1}{2}} {\Zhj}^{\top} \Zhj \Th^{\frac{1}{2}} \What_h\right)^{+} \What_h^{\top} \Th \What_h\right)=\sigma^2 \operatorname{Tr}\left(\left(\Th^{\frac{1}{2}} {\Zhj}^{\top} \Zhj \Th^{\frac{1}{2}}\right)^{+} \Th\right)
\end{aligned}
\end{equation}
where we used the fact that $(\mathbf{A}_1 \mathbf{A}_2)^{+}=\mathbf{A}_2^{+} \mathbf{A}_1^{+}$ if $\mathbf{A}_1$ has orthonormal columns or $\mathbf{A}_2$ has orthonormal rows. Compare expressions in \eqref{BVred} with $B_X$ and $V_X$ in \cite{hastie2022surprises}, Lemma 1, we note that $B_{\mathsf{SC}} $ and $V_{\mathsf{SC}} $ are equal in law to the bias and variance of ridgeless regression in \cite{hastie2022surprises} with data covariance $\Sigmabf$ and signal $\beta ^\star$ in \cite{hastie2022surprises} replaced by $\Th$ and $\What_h \Gammahat^{1/2}\st$, respectively. Note that the new data covariance matrix $\Th$ is diagonal with the smallest and the largest diagonal entry,  $\that_{\min}^+$ and $\that_{\max}$, satisfy 
\begin{equation}
    \max(1/\that_{\min}^+, \that_{\max})\le C(M)
\end{equation}
where we have used \eqref{asseq} and \Cref{bdthat}. This and other assumptions in \Cref{Assum} ensures that \cite{hastie2022surprises}, Assumption 1 holds and we may directly apply \cite{hastie2022surprises}, Theorem 2 to characterize $B_{\mathsf{SC}} $ and $V_{\mathsf{SC}} $ above. We conclude the proof by noting that that 
\begin{equation}
    \norm{\What_h \Gammahat^{1/2}\st}_2\le \norm{\Gammahat^{1/2}\st}_2 \le \normop{\Gammahat^{1/2}} \norm{\st}_2\le C(M)\cdot \norm{\st}_2
\end{equation}
where we have used the fact $\What_h^\top \What_h$ is an orthogonal projection in the first inequality and \eqref{asseq} in the last inequality. 
\end{proof}

\subsubsection{Characterization of Fine-grained Bias-Variance Decomposition}\label{CharFBV}

Below, we characterize fine-grained bias $B $ and variance components $V_{\Xbf} , V_{\epbf} , V_{\Xbf,\epbf} $, leveraging \Cref{mMprop} and results from \cite{knowles2017anisotropic}, which also proves \Cref{thm} along with \Cref{mMprop}. 

\begin{proof}[Proof of \Cref{thm}]
We first show that $V _\epbf=0$. It follows from
\begin{equation*}
\begin{aligned}
 V_{\epbf}  &=\mathbb{E}_{\xnewbf } \mathbb{V}_{\Xbf , \epbf }\left({\xnewbf }^{\top} \mathbb{E}_{\Xbf } \betahatbf \right) \\
& =\sigma^2 \mathbb{E}_{\xnewbf }\left({\xnewbf }^{\top} \Gammahat^{-\frac{1}{2}} \mathbb{E}\left[\left(\Gammahat^{-\frac{1}{2}} {\Xbf }^\top \Xbf  \Gammahat^{-\frac{1}{2}}\right)^{+} \Gammahat^{-\frac{1}{2}} X^{\top}\right]\right)\left(\mathbb{E}\left[\Xbf  \Gammahat^{-\frac{1}{2}}\left(\Gammahat^{-\frac{1}{2}} {\Xbf }^\top \Xbf  \Gammahat^{-\frac{1}{2}}\right)^{+}\right] \Gammahat^{-\frac{1}{2}} \xnewbf  \right) \\
& =\sigma^2 \operatorname{Tr}\left(\mathbb{E}\left[\left(\Gammahat^{-\frac{1}{2}} {\Xbf }^\top \Xbf  \Gammahat^{-\frac{1}{2}}\right)^{+} \Gammahat^{-\frac{1}{2}} {\Xbf }^{\top}\right] \mathbb{E}\left[\Xbf  \Gammahat^{-\frac{1}{2}}\left(\Gammahat^{-\frac{1}{2}} {\Xbf }^\top \Xbf  \Gammahat^{-\frac{1}{2}}\right)^{+}\right] \Gammahat^{-\frac{1}{2}} \Sigmabf \Gammahat^{-\frac{1}{2}}\right) \\
& =0
\end{aligned}
\end{equation*}
where the last equality is by symmetry
\begin{equation*}
\mathbb{E}\left[\Xbf  \Gammahat^{-\frac{1}{2}}\left(\Gammahat^{-\frac{1}{2}} {\Xbf }^\top \Xbf  \Gammahat^{-\frac{1}{2}}\right)^{+}\right]=-\mathbb{E}\left[\Xbf  \Gammahat^{-\frac{1}{2}}\left(\Gammahat^{-\frac{1}{2}} {\Xbf }^\top \Xbf  \Gammahat^{-\frac{1}{2}}\right)^{+}\right] .
\end{equation*}
Note the relation
\begin{equation}
    \E_{\Xbf } \BSC =B +V_{\Xbf} , \qquad  \E_{\Xbf } \VSC =V_{\Xbf,\epbf} +V_{\epbf} 
\end{equation}
and the fact that both $B $ and $V_{\Xbf} $ must be non-negative. Recall that in the sample-rich regime,  $B_{\mathsf{SC}} =0$. The above implies that $B =0, V_{\Xbf} =0$ and $V_{\Xbf,\epbf} =\E_{\Xbf } \VSC $. Let $\mathcal{E}$ denote the event that \eqref{underV} holds for the choice $D=1/7$. We know that $\mathbb{P}(\mathcal{E}^c)\le C(M)\cdot n^{-1/7}$. It follows that
\begin{equation}\label{higheventrick}
    \abs{V_{\Xbf,\epbf} -\mathcal{U}}\le \E_{\Xbf } \abs{\VSC -\mathcal{U}}\cdot \mathbb{I}_\mathcal{E}+\qty(\E_{\Xbf }\VSC  + \mathcal{U})\cdot \mathbb{P}(\mathcal{E}^c)\le C(M)\cdot n^{-1/7}\cdot \qty(\E_{\Xbf }\VSC  + \mathcal{U}+1)
\end{equation}
Note that in the under parameterized regime, ${\Zhj}^{\top} \Zhj$ is non-singular almost surely and thus using \eqref{BVred}
\begin{equation*}
    \E_{\Xbfj} V_{\mathsf{SC}} =\sigma^2 \E \operatorname{Tr}\left(\left(\Th^{\frac{1}{2}} {\Zhj}^{\top} \Zhj \Th^{\frac{1}{2}}\right)^{+} \Th\right)=\sigma^2\cdot \E \operatorname{Tr}\left(\left({\Zhj}^{\top} \Zhj\right)^{+} \right)=\sigma^2\cdot \frac{h}{n-h-1}\le C(M).
\end{equation*}
using moment property of inverted Wishart distribution (see e.g. \cite{kollo2005advanced}, Theorem 2.4.14.) and the assumption that $M^{-1}<h/n<1-M^{-1}$. It is straightforward to verify that $\mathcal{U}\le C(M)$. We can then conclude the proof for the characterization of fine-grained bias and variances in the sample-rich regime. 

From now one, we assume that we are in the sample-deficient regime. First note that
\begin{equation}\label{finegrainedB}
    \begin{aligned}
        B &=\E_{\xnewbf } \qty({\xnewbf }^\top \E_{\Xbf , \epbf } \qty({\betahatbf }-\st))^2\\
        &={\st}^\top \Gammahat^{1/2} \E_{\Xbf }(\Ibf-\Pbf ) \Gammahat^{-1/2} \Sigmabf \Gammahat^{-1/2} \E_{\Xbf }(\Ibf-\Pbf ) \Gammahat^{1/2} {\st}\\
        &= {\st}^\top \Gammahat^{1/2} \What_h^\top \E(\Ibf-\Pbf _h) \Th \E(\Ibf-\Pbf _h) \What \Gammahat^{1/2} {\st}
    \end{aligned}
\end{equation}
where we have used \eqref{usefmapt} for the last equality. Note that the key difference between fine-grained bias \eqref{finegrainedB} and classical bias $B_{\mathsf{SC}}$ defined in \eqref{BVred} is that $B$ is the expectation over 
\begin{equation}\label{twocopies}
    {\st}^\top \Gammahat^{1/2} \What_h^\top (\Ibf-\Pbf _h^{(1)}) \Th (\Ibf-\Pbf _h^{(2)}) \What \Gammahat^{1/2} {\st}
\end{equation}
for two \emph{independent copies} $\Pbf_h^{(1)}, \Pbf_h^{(2)}$  of random matrix $\Pbf_h$. Our strategy is then to apply resolvent method twice for the two independent copies. 

Let us introduce some notations
\begin{equation}\label{defvjPj}
\begin{aligned}
& \upbf :=\Th \mathbb{E} \left[\Ibf-\Ph \right] \What_h \Gammahat^{\frac{1}{2}} \st, \quad \tilde{\upbf} :=\What_h \Gammahat^{\frac{1}{2}} \st \\
& \Phpj:=\Ibf-\Ph =\Ibf-\left(\Th^{\frac{1}{2}}  {\Zhj}^{\top} \Zhj \Th^{\frac{1}{2}}\right)^{+} \Th^{\frac{1}{2}} {\Zhj}^{\top} \Zhj \Th^{\frac{1}{2}}. 
\end{aligned}
\end{equation}
Note that using \eqref{asseq} and \Cref{bdthat}, we have that
\begin{equation}\label{upnorm}
    \norm{\upbfj}_2^2, \norm{\upbfjtilde}_2^2 \le C(M) \norm{\stj}_2^2.
\end{equation}
Using the above, we have that
\begin{equation}\label{Bjnormb}
    B =\E {\upbfjtildeT} \Phpj \upbfj \le  C(M)\cdot \norm{\stj}_2^2.
\end{equation}
Let us further define
\begin{equation*}
\Pw:=\varpi\cdot \left(\frac{1}{n} \XbfhjT \Xbfj_h+\varpi \cdot \Ibf_h\right)^{-1}
\end{equation*}
where $\Xbfhj:=\Zhj \Th^{1/2} \in \R^{n\times h}$. We also let 
\begin{equation*}
    \frac{1}{n} \XbfhjT \Xbfhj = \sum_{i=1}^{h_+} s_i^+ \cdot \mathbf{l}_i \mathbf{l}_i^\top
\end{equation*}
where $\left(s_i^{+}\right)_{i=1}^{h_{+}}$are non-zero eigenvalues of $\frac{1}{n} \XbfhjT \Xbfhj$ and $\left(\mathbf{l}_i\right)_{i=1}^{h_{+}}$ are corresponding eigenvectors. It follows that for any $\varpi>0$,
\begin{equation*}
    \Pw-\Phpj =\sum_{i=1}^{h_+} \frac{\varpi}{s_i^+ + \varpi} \cdot \mathbf{l}_i \mathbf{l}_i^\top \implies \normop{\Pw-\Phpj} \le \frac{\varpi}{s_{\min}^+}
\end{equation*}
where $s_{\min }^{+}:=\min \left\{s_i^{+}, i=1, \ldots, h_{+}\right\}$. Note that almost surely
\begin{equation*}
    s_{\min }^{+}=\sigma_n \qty(\Th^{1/2} {\Zhj}^\top  \Zhj \Th^{1/2})=\sigma_n \qty(\Zhj\Th {\Zhj}^\top  )\ge \that_{\min}^+ \cdot  \sigma_n \qty(\Zhj {\Zhj}^\top )=\that_{\min}^+ \cdot  \eigbot \qty({\Zhj}^\top \Zhj )
\end{equation*}
where $\sigma_n(\cdot)$ denotes the $n$-th largest eigenvalue. Given that $1+M^{-1}<h/n<M$, a well known result (see e.g. \cite{rudelson2009smallest} Theorem 1.1) bounds $\eigbot \qty({\Zhj}^\top \Zhj )$ away from $0$ with high probability. Using this result, we have that for any $D>0$, there exists some constant $C=C(M,D)$ such that with probability $1-Cn^{-D}$, 
\begin{equation}\label{star1}
    \abs{\upbfjtildeT \Pw \upbfj - \upbfjtildeT \Phpj \upbfj}\le C(M)\cdot \varpi\cdot \norm{\stj}_2^2. 
\end{equation}
Now consider $z=-\varpi+z_I\cdot i$ such that $\varpi>0$ and $|z_I|<M$. Define $b_z$ to be the unique solution with $\mathrm{Im}(b_z)>0$ (see \cite{knowles2017anisotropic}, Lemma 2.2) of the equation
\begin{equation*}
    \frac{1}{b_z}=-z+\frac{1}{n} \sum_{i\in H} \frac{\that_i}{1+\that_i b_z}.
\end{equation*}
Using \cite{knowles2017anisotropic}, Theorem 3.16 (1) as well as Remark 3.17, we obtain that for any $\epsilon, \epsilon_0, D>0$, there exists $c=c(\epsilon, \epsilon_0, D)>0$ such that with probability at least $1-cn^{-D}$, for all $z$ such that $z_I \in (0,M)$ and $\varpi \in (n^{-D},\infty)$,
\begin{equation}\label{thekeyres}
    \abs{\upbfjtildeT \Pz \upbfj - \upbfjtildeT \Pzinf \upbfj}\le \sqrt{\frac{\operatorname{Im}\left(b_z\right)}{z_I} \cdot n^{-1+\epsilon}} \cdot C(M) \cdot\left\|\stj\right\|_2^2
\end{equation}
where we have also used \eqref{upnorm} and the following definition
\begin{equation*}
    \Pz:=\qty(\Ibf+b_z \Th)^{-1}, \quad \Pzinf :=-z\qty(\frac{1}{n} \XbfhjT \Xbfhj -z\Ibf)^{-1}.
\end{equation*}
\cite{knowles2017anisotropic}, Lemma 2.2 then states that $b_z$ is in fact the Stieltjes transform of a probability measure $\rho$ with support in $[0,C(M)]$. Thus
\begin{equation*}
\operatorname{Im}\left(b_z\right)=\operatorname{Im} \int \frac{1}{x-z} d \rho(x)=\int \frac{z_I}{(x+\varpi)^2+z_I^2} d \rho(x) \implies \left|\operatorname{Im}\left(b_z\right)\right| \leq \frac{z_I}{\varpi^2}.
\end{equation*}
Therefore, taking the limit $z_I\to 0$ in \eqref{thekeyres} above, we obtain that with probability at least $1-cn^{-D}$,
\begin{equation}\label{star2}
\left| \upbfjtildeT \Pw \upbfj - \upbfjtildeT \Pwinf \upbfj \right| \leq \frac{1}{n^{(1-\epsilon) / 2} \varpi} \cdot C(M) \cdot\left\|\stj\right\|_2^2, \quad \forall \varpi \in\left(n^{-2/3+\epsilon_0}, \infty\right).
\end{equation}
Let us define
\begin{equation}\label{Phpinfdef}
    \Phpinf := \qty(\Ibf_h + b_0\cdot \Th)^{-1}
\end{equation}
for $b_0$ defined in \eqref{fp}. Let us define $b_\varpi$ to be $b_z$ with $z\gets -\varpi$. A similar argument to the one used in \cite{hastie2022surprises}, Theorem 2 to prove its Eq. (107) may be used to show that $$\abs{b_\varpi-b_0}\le C(M)\cdot \varpi$$ which then implies that $\normop{\Pwinf-\Phpinf}\le C(M)\cdot \varpi$. We thus obtain
\begin{equation}\label{star3}
    \abs{\upbfjtildeT \Pwinf \upbfj - \upbfjtildeT \Phpinf \upbfj}\le C(M)\cdot \varpi \cdot \norm{\stj}_2^2.
\end{equation}
Combining \eqref{star1}, \eqref{star2} and \eqref{star3} above, we obtain that for any $D>0$, there exists some $C=C(M,D)$ such that with probability at least $1-Cn^{-D}$,
\begin{equation}\label{star4}
\left|\upbfjtildeT \Phpj \upbfj - \upbfjtildeT \Phpinf \upbfj \right| \leq C\cdot\left(2 \varpi+\frac{1}{n^{(1-\varepsilon)/{2}} \varpi}\right)\left\|\stj\right\|_2^2 \leq C \cdot n^{-1 / 7} \cdot\left\|\stj\right\|_2^2
\end{equation}
where we have set $\varpi \gets n^{-1/n}$ and chosen $\epsilon_0, \epsilon$ to be sufficiently small. Let us denote this high probability event as $\mathcal{E}$. Recall that $B  \equiv \E_{\Xbfj} \upbfjtildeT \Phpj \upbfj$. Choosing $D=1/7$ for \eqref{star4}, we have that
\begin{equation}\label{sdf}
    \begin{aligned}
        &\abs{B -\upbfjtildeT \Phpinf \upbfj} \le \E \abs{\upbfjtildeT \Phpj \upbf-\upbfjtildeT \Phpinf \upbfj}\\
        &\qquad  \le \E \abs{\upbfjtildeT \Phpj \upbf-\upbfjtildeT \Phpinf \upbfj} \cdot \mathbb{I}_\mathcal{E} +\left( \E \abs{\upbfjtildeT \Phpj \upbf}+\abs{\upbfjtildeT \Phpinf \upbfj} \right) \mathbb{P}\left(\mathcal{E}^c\right)\\
        &\qquad \le C(M)\cdot n^{-1/7} \qty(1+ \norm{\stj}_2^2)
    \end{aligned}
\end{equation}
where in the last inequality we used \eqref{star4} with $D=1/7$ and the 
bound 
\begin{equation*}
     \E \abs{\upbfjtildeT \Phpj \upbf}+ \abs{\upbfjtildeT \Phpinf \upbfj}\le C(M)\cdot \norm{\stj}_2^2.
\end{equation*}
This bound follows from \eqref{upnorm}, \eqref{Phpinfdef} along with \Cref{bdthat} and \ref{bA1A2size}, and the fact that $\Phpj$ is an orthogonal projection as defined in \eqref{defvjPj}. 

Now define 
\begin{equation*}
    \upbfjbreve := \Th \Phpinf \What_h \Gammahat^{1/2} \stj.
\end{equation*}
Observe that
\begin{subequations}
\begin{align}
    & \upbfjtildeT \Phpinf \upbfj = {\stj}^\top \Gammahat^{1/2} \What_h^\top \Phpinf \Th \E_{\Xbfhj} \qty(\Ibf-\Ph) \What_h \Gammahat^{1/2} \stj= \E_{\Xbfhj} \upbfjbreveT \Phpj \upbfj \label{dffff1}\\
    & \upbfjbreveT \Phpinf \upbfj={\stj}^\top \Gammahat^{1/2} \What_h^\top \Phpinf \Th \Phpinf \What_h \Gammahat^{1/2} \stj=\mathfrak{B} .\label{dffff2}
\end{align}
\end{subequations}
An argument analogous to the one we used to derive \eqref{sdf} may be used to show that
\begin{equation}\label{star5}
    \begin{aligned}
        \abs{\upbfjtildeT \Phpinf \upbfj-\mathfrak{B} }&=\abs{\upbfjtildeT \Phpinf \upbfj- \upbfjbreveT \Phpinf \upbfj}\\
        &\le \E \abs{ \upbfjbreveT \Phpj \upbfj- \upbfjbreveT \Phpinf \upbfj}\\
        &\le C(M)\cdot \norm{\stj}_2^2
    \end{aligned}
\end{equation}
where we used \eqref{dffff1} and \eqref{dffff2} in the first and second line, respectively. Combing \eqref{sdf} and \eqref{star5} yields the characterization result for $B $. 

It remains to prove characterization result for $V _{\Xbf}$ and $V_{\Xbf, \epbf}$. It follows from \eqref{BVred}, \Cref{bdthat} and \Cref{bA1A2size} that
\begin{equation*}
    \E_{\Xbfhj} B _{\mathsf{SC}}\le C(M)\cdot \norm{\stj}_2^2,\qquad |\mathfrak{B} |\le C(M)\cdot \norm{\stj}_2^2,\qquad |\mathcal{V}|\le \sigma^2 \cdot C(M).
\end{equation*}
Using \eqref{BVred}, we may also obtain
\begin{equation*}
    \begin{aligned}
        \E_{\Xbfhj} V_{\mathsf{SC}}  &= \E \Tr(\qty(\Th^{1/2} {\Zhj}^\top \Zhj \Th^{1/2} )^+ \Th^{1/2}) \stackrel{(i)}{\le} \E \sum_{i=1}^n \frac{\that_{\max}}{\sigma_i\qty(\Th^{1/2} {\Zhj}^\top \Zhj \Th^{1/2} )}\\
        & = \E \sum_{i=1}^n \frac{\that_{\max}}{\sigma_i\qty( \Zhj \Th {\Zhj}^\top  )}\stackrel{(ii)}{\le} \frac{\that_{\max}}{\that_{\min}^+} \cdot \E \sum_{i=1}^n \frac{1}{\sigma_i \qty(\Zhj {\Zhj}^\top)} \stackrel{(iii)}{\le} \frac{\that_{\max}}{\that_{\min}^+} \cdot \frac{n}{h-n-1}\stackrel{(iv)}{\le}  C(M)
    \end{aligned}
\end{equation*}
where $\sigma_i(\cdot)$ denotes the $i$-th largest eigenvalue. Here, we used at $(i)$ Von Neumann's trace inequality, at (ii) the variational representation of eigenvalues and the fact that $\eigbot(\Th)=\that_{\min}^+$, at (iii) moment property of inverted Wishart distribution (see e.g. \cite{kollo2005advanced}, Theorem 2.4.14.) and at (iv) the assumption that $1+M^{-1}<h/n<M$ along with \Cref{bdthat}. Using these bounds, an approach analogous to \eqref{higheventrick} yields that
\begin{equation*}
    \abs{\E_{\Xbfhj} B _{\mathcal{SC}} -(\mathcal{V}+1)\cdot \mathfrak{B} }\le C(M)\cdot n^{-1/7}\cdot \norm{\stj}_2^2, \quad \abs{\E_{\Xbfhj} V_{\mathsf{SC}} }\le \sigma^2\cdot C(M)\cdot n^{-1/7}.
\end{equation*}
The characterization result for $V _{\Xbf}$ and $V_{\Xbf, \epbf}$ follows from the above, the following relations
\begin{equation*}
    \E_{\Xbf } \BSC =B +V_{\Xbf} , \qquad  \E_{\Xbf } \VSC =V_{\Xbf,\epbf} +V_{\epbf} 
\end{equation*}
and a straightforward application of triangle inequality. This concludes the proof of \Cref{thm}.
\end{proof}

\subsection{Proof of \Cref{thm2}}\label{app:characterizesol}
We prove \Cref{thm2} in this section, using \Cref{Convexity} and \Cref{Equivalence} proved in the previous section. The proof draws extensive use of the Karush-Kuhn-Tucker (KKT) optimality conditions; see \cite{boyd2004convex}, Section 5 for a review. 

Recall from \eqref{pos} and \eqref{vsdef} the following definition
\begin{equation}
    h_1:=\abs{\{i: \vs_i \equiv \eta_i \cdot \ubf^\top_i \Bstar \Sigmabf_{\alphastar} {\Bstar}^\top \ubf_i \neq 0\}}, \quad \theta_i:=\ubf^\top_i \Bstar \Sigmabf_{\alphastar} {\Bstar}^\top \ubf_i, \quad i=1,\ldots,p.
\end{equation}
For convenience, we will introduce the following notation
\begin{equation*}
    \vs_i=\eta_i\cdot \theta_i, i=1,...,p.
\end{equation*}
Recall that we assumed without loss of generality that $\vs_i,i=1,\ldots,p$ is in descending order. It follows that
$$H=\qty{1,...,h}.$$

\subsubsection{Existence and Uniqueness of $\ho$}
We prove in this section that $\ho$ defined in \eqref{ineqho} indeed exists and is unique. 

\begin{proposition}\label{hodef}
Suppose that $n<h_1$. There exists a unique integer $\ho \in\{n,\ldots,h_1\}$ such that $\tilde{h}\le h_0$ if and only if
\begin{equation}
    \frac{1}{\hind-n} \sum_{i=1}^{\hind} \frac{\eta_{\hind}\theta_{\hind}}{\eta_i\theta_i}\ge 1.
\end{equation}
\end{proposition}

\begin{proof}[Proof of \Cref{hodef}]
Recall from the statement of \Cref{hodef} that $h_1>n$. Now let us state the following claim.  

\textbf{Claim 1.} If for some $\hind \in \qty{n,...,h_1-1}$
\begin{equation*}
    \frac{1}{\hind} \sum_{i\le \hind} \frac{\vs_{\hind}}{\vs_{i}} + \frac{n}{\hind} -1\ge 0,
\end{equation*}
we would have that
\begin{equation*}
    \frac{1}{\hind+1} \sum_{i\le \hind+1} \frac{\vs_{\hind+1}}{\vs_{i}} + \frac{n}{\hind+1} -1 \le \frac{1}{\hind} \sum_{i\le \hind} \frac{\vs_{\hind}}{\vs_{i}} + \frac{n}{\hind} -1.
\end{equation*}
To see claim 1, note that
\begin{equation*}
    \qty(\frac{1}{\hind+1} \sum_{i\le \hind+1} \frac{\vs_{\hind+1}}{\vs_{i}} + \frac{n}{\hind+1} -1 ) - \qty(\frac{1}{\hind} \sum_{i\le \hind} \frac{\vs_{\hind}}{\vs_{i}} + \frac{n}{\hind} -1)=\frac{1}{\hind+1} \qty(1-\frac{n}{\hind} - \qty(\vs_{\hind} +\frac{\vs_{\hind}}{\hind}-\vs_{\hind+1}) \cdot \sum_{i\le \hind} \frac{1}{\vs_i} ) 
\end{equation*}
and that, by $\vs_{\hind} \ge \vs_{\hind+1}$, 
\begin{equation*}
    \qty(\vs_{\hind} +\frac{\vs_{\hind}}{\hind}-\vs_{\hind+1}) \cdot \sum_{i\le \hind} \frac{1}{\vs_i}\ge \frac{\vs_{\hind}}{\hind} \cdot \sum_{i\le \hind} \frac{1}{\vs_i} \ge 1-\frac{n}{\hind}.
\end{equation*}
Combining the above proves Claim 1.

Now we state another claim.

\textbf{Claim 2.} If for some $\hind \in \qty{n,...,h_1-1}$
\begin{equation}\label{claim2cond}
    \frac{1}{\hind} \sum_{i\le \hind} \frac{\vs_{\hind}}{\vs_{i}} + \frac{n}{\hind} -1<0,
\end{equation}
we would have that
\begin{equation*}
    \frac{1}{\hind+1} \sum_{i\le \hind+1} \frac{\vs_{\hind+1}}{\vs_{i}} + \frac{n}{\hind+1} -1<0.
\end{equation*}
To see claim 2, note that
    \begin{equation*}
    \begin{aligned}
        \frac{1}{\hind+1} \sum_{i\le \hind+1} \frac{\vs_{\hind+1}}{\vs_{i}} + \frac{n}{\hind+1} -1
        &= \frac{\hind}{\hind+1} \qty(\frac{\vs_{\hind+1}}{\vs_{\hind}}\cdot \qty( \frac{\vs_{\hind}}{\hind} \sum_{i\le \hind} \frac{1}{\vs_{i}} ) -\qty(1-\frac{n}{\hind}))\\
        &=\frac{\hind}{\hind+1} \cdot \qty(\frac{\vs_{\hind+1}}{\vs_{\hind}}\cdot \qty(\frac{1}{\hind} \sum_{i\le \hind} \frac{\vs_{\hind}}{\vs_{i}} + \frac{n}{\hind} -1) - \qty(1-\frac{\vs_{\hind+1}}{\vs_{\hind}}) \qty(1-\frac{n}{\hind}))\\
        & <0
\end{aligned}
\end{equation*}
where we used \eqref{claim2cond} and $\vs_{\hind} \ge \vs_{\hind+1}$ in the last line. This proves Claim 2.

Now let us denote
\begin{equation*}
    \Upsilon(\hind):=\frac{1}{\hind} \sum_{i\le \hind} \frac{\vs_{\hind}}{\vs_{i}} + \frac{n}{\hind} -1, \quad \hind=n,...,h.
\end{equation*}
We have the initial condition that $\Upsilon(n)>0$. If we have the terminal condition that $\Upsilon(h)\ge 0$, it follows from Claim 1 that $\Upsilon(\hind)\ge 0,\forall \hind \in \{n,...,h\}$. If we have the terminal condition that $\Upsilon(h_1)< 0$, combination of Claim 1 and Claim 2 then show that $\Upsilon(\hind)$ decreases monotonically as $\hind$ increases from $n$ to $h_1$. This implies that there exists a unique $\ho$ such that for all $\hind \le \ho$, $\Upsilon(\hind)\le 0$ and for all $\hind > \ho, \Upsilon(\hind)< 0$. We conclude the proof by noting the equivalence
\begin{equation*}
    \Upsilon(\hind)\ge 0 \iff \frac{1}{\hind-n} \sum_{i=1}^{\hind} \frac{\vs_{\hind}}{\vs_i}\ge 1, \quad \hind=n,...,h_1.
\end{equation*}

\end{proof}

\subsubsection{Analysis of KKT condition}
\begin{proof}[Proof of \Cref{thm2}]
We will first prove optimization result for $\mathcal{V}$ and then optimization result for $\mathfrak{B}$.

\textbf{Optimize $\mathcal{V}$.} We first consider the following optimization problem
\begin{equation}\label{optpf}
    \begin{aligned}
        &\min_{\xH \in [0,1]^h} \tilde{\mathcal{V}}=\frac{2n-h+\sum_{i\in H} x_i^2}{h-n-\sum_{i\in H} x_i^2} \\ 
        &\quad \mathrm{subject \; to\;} \frac{1}{h}\sum_{i\in H} x_i =1-\frac{n}{h}
    \end{aligned}
\end{equation}
for $\tilde{\mathcal{V}}$ defined in \eqref{deftildeA}. Recall from \Cref{Convexity} and its proof that $\tilde{\mathcal{V}}=g\qty(\sum_{i\in H} x_i)$ where $g$ is a strictly increasing convex function on the range of $\sum_{i\in H} x_i$. It is therefore sufficient to consider the convex optimization problem of minimizing $\frac{1}{2h} \sum_{i\in H} x_i^2$ under the linear constraint $ \frac{1}{h}\sum_{i\in H} x_i =1-\frac{n}{h}$. The corresponding Lagrangian is  
\begin{equation*}
    \frac{1}{2h} \sum_{i\in H} x_i^2 + \rho \cdot \qty(\frac{1}{h}\sum_{i\in H} x_i -1+\frac{n}{h}). 
\end{equation*}
where $\rho \in \R$ is the Lagrange multiplier. Minimizing the Lagrangian yields the optimal solution $x_i=1-n/h, \forall i \in H$. The result then follows from \Cref{Equivalence}. 

\textbf{Optimize $\mathfrak{B}^{\mathsf{avg}}$.}  We first consider the case where $h_1>n$. 

\textbf{Case 1: $h_1>n$.} 
Consider the following optimization problem
\begin{equation}
    \begin{aligned}
        &\min_{\xH \in [0,1]^h} \tilde{\mathfrak{B}}^{\mathsf{avg}}=\frac{1}{q}\sum_{i\in H} \vs_i  x_i^2 \\ 
        &\quad \mathrm{subject \; to\;} \frac{1}{h}\sum_{i\in H} x_i =1-\frac{n}{h}.
    \end{aligned}
\end{equation}
To find the optimizer, we may consider the Lagrangian
\begin{equation*}
    \frac{1}{2} \sum_{i\in H} \vs_i \cdot x_i^2 +\rho \qty(\frac{1}{h} \sum_{i\in H} x_i -1+\frac{n}{h}) + \sum_{i\in H} \rzero \cdot (-x_i) +\sum_{i\in H} \rone (x_i-1)
\end{equation*}
where $\rho, \qty{\rzero}_{i\in H}, \qty{\rone}_{i\in H}$ are KKT multipliers. The associated KKT conditions are then
\begin{itemize}
    \item \textbf{(Stationarity)} $\vs_i x_i +\rho/h+\rzero-\rone=0,\forall i \in H$
    \item \textbf{(Primal feasibility)} $h^{-1} \sum_{i\in H} x_i = 1-n/h, x_i \in [0,1],\forall i \in H$
    \item \textbf{(Dual feasibility)} $\rzero, \rone \ge 0, \forall i \in H$
    \item \textbf{(Complementary slackness)} $\rzero x_i=0, \rone (x_i-1)=0, \forall i \in H$
\end{itemize}

The stationarity condition holds if and only if
\begin{equation}\label{exx}
    x_i = \frac{\rzero-\rone}{\vs_i} - \frac{\rho}{h \vs_i}, \;\; \forall i \le h_1, \qquad \rzero =\frac{1}{h} \rho+\rone,\;\; \forall i > h_1.
\end{equation}
The primal feasibility condition holds if and only if for all $i\in H$,  $x_i \in [0,1]$ and
\begin{equation}\label{exrho}
    \frac{1}{h} \sum_{i\le h_1} \frac{\rzero-\rone}{\vs_i} - \frac{1}{h} \sum_{i\le h_1}  \frac{\rho}{h \vs_i}+\frac{1}{h}\sum_{h_1<i\le h} x_i=1-\frac{n}{h} \iff -\frac{\rho}{h}=\frac{1-\frac{n+\sum_{h_1<i\le h} x_i}{h} + \frac{1}{h} \sum_{i\le h_1} \frac{\rzero-\rone}{\vs_i}}{\frac{1}{h} \sum_{i\le h_1}  \frac{\rho}{ \vs_i}}.
\end{equation}
Note that \eqref{exx} and \eqref{exrho} above express $\qty{x_i}_{i=1}^{h_1},\rho$ in terms of $\qty{\rzero}_{i\in H},\qty{\rone}_{i\in H}$ and $\qty{x_i}_{i=h_1+1}^{h}$ such that for any choice of $\qty{x_i}_{i=h_1+1}^{h}$ and any choice of $\rzero,\rone$ that satisfies
\begin{equation}\label{stationaryEQ}
    \rzero =\frac{1}{h} \rho+\rone, \forall i > h_1,
\end{equation}
the stationarity condition and $h^{-1} \sum_{i\in H} x_i = 1-n/h,\forall i \in H$ in the primal feasibility condition will hold. Now we claim that if we choose $\qty{x_i}_{i=h_1+1}^{h}$ and KKT multipliers as
\begin{equation}\label{multchoice}
 x_i=1,\;\; \forall i\in \{h_1+1,...,h\}, \qquad \rzero=0,\;\; \forall i\in H, \qquad \rone=\left\{\begin{array}{c}
0, \quad \forall i \leq \ho, \\
-\vs_i+\frac{1-\frac{n}{\ho}}{\frac{1}{\ho} \sum_{j \le \ho} \frac{1}{\vs_j}}, \quad \forall i\in \{\ho+1,...,h\}
\end{array}\right.
\end{equation}
for $\ho \in \{n,...,h_1\}$ defined in \Cref{hodef}, then the KKT conditions hold. 

We first show that the complementary slackness and stationarity condition hold. Plugging the choice of $\qty{x_i}_{i=h_1+1}^{h}$ and the KKT multipliers from \eqref{multchoice} into \eqref{exrho}, we obtain that
\begin{equation}\label{exrhoplut}
    -\frac{\rho}{h} =\frac{1-\frac{n}{\ho}}{\frac{1}{\ho} \sum_{j\le \ho} \frac{1}{\vs_j}}.
\end{equation}
We see that plugging $\rho$ from \eqref{exrhoplut} and $\qty{\rzero}_{i\in H},\qty{\rone}_{i\in H}, \qty{x_i}_{i=h_1+1}^{h}$ from \eqref{multchoice} into the expression of $x_i$ in \eqref{exx} yields
\begin{equation}\label{xidef}
x_i=\left\{\begin{array}{c}
\frac{1-\frac{n}{\ho}}{\frac{1}{\ho} \sum_{j\le \ho} \frac{\vs_i}{\vs_j}}, \quad \forall i \in \qty{1,...,\ho}, \\
1, \quad \forall i \in \qty{\ho+1,...,h}
\end{array}\right..
\end{equation}
It is now clear that the complementary slackness condition holds. Meanwhile, it is easy to see that \eqref{stationaryEQ} holds, given that $\vs_i=0,\forall i \in \qty{\ho+1,...,h}$

Now we show that $x_i \in [0,1],\forall i \in H$ in the primal feasibility. This holds for all $i \in \{\ho+1,...,h\}$ since $x_i=1$ and for $i \in \{1,...,\ho\}$, we have
\begin{equation*}
    \frac{1}{\ho} \sum_{j\le \ho} \frac{\vs_i}{\vs_j} \ge \frac{1}{\ho} \sum_{j\le \ho} \frac{\vs_{\ho}}{\vs_j}\ge 1-\frac{n}{\ho}
\end{equation*}
where we used $\vs_i \ge \vs_{\ho}, \forall i \in \{1,...,\ho \}$  in the first inequality and \Cref{hodef} in the second inequality. This implies that $x_i \in [0,1],\forall i \in \{1,...,\ho\}$. We may then conclude that the primal feasibility condition holds. 

Finally, we prove that the dual feasibility condition holds. Note that we already have $\rzero=0,\forall i \in H$ and $\rone=0,\forall i \in \{1,...,\ho\}$. We only need to show that $\rone \ge 0.\forall i \in \{\ho+1,...,h\}$. To see this, note that for all $i \in \{\ho,...,h\}$, we have that 
\begin{equation*}
    \begin{aligned}
        \frac{\vs_{i}}{\ho} \sum_{j\le \ho} \frac{1}{\vs_{j}} &\le \frac{\ho+1}{\ho} \frac{\vs_{\ho+1}}{\ho+1} \sum_{j\le \ho} \frac{1}{\vs_{j}}\\ &= \frac{\ho+1}{\ho} \cdot \frac{\vs_{\ho+1}}{\ho+1} \cdot \qty( \sum_{j\le \ho} \frac{1}{\vs_j}+\frac{1}{\vs_{\ho+1}} ) -\frac{1}{\ho} \\&<\frac{\ho+1}{\ho}\cdot \qty(1-\frac{n}{\ho+1})-\frac{1}{\ho} \\&=1-\frac{n}{\ho}
    \end{aligned}
\end{equation*}
where we used $\vs_i \le \vs_{\ho}, \forall i \in \{\ho+1,...,h \}$ in the first inequality and \Cref{hodef} in the second inequality. The above then implies that 
$$\rone=-\vs_i+\frac{1-\frac{n}{\ho}}{\frac{1}{\ho} \sum_{j \le \ho} \frac{1}{\vs_j}}>0, \quad \forall i \in \{\ho+1,...,h \}$$
as required. This proves that the dual feasibility condition holds and therefore the claim. 

We have shown above that for $\qty{x_i}_{i\in H}$ defined in \eqref{xidef} and KKT multipliers defined in \eqref{multchoice}, the KKT conditions hold. Note that the objective of \eqref{optpf} is differentiable and recall from \Cref{Convexity} that it is also convex. Furthermore, we have proved $\qty{x_i}_{i\in H}$ defined in \eqref{xidef} satisfies the constraint of \eqref{optpf}; this along with the fact that all constraints of \eqref{optpf} are linear implies that the Slater's condition holds (see \cite{boyd2004convex} Section 5.2.3 for a review of the Slater's condition). Given the above, we know from \cite{boyd2004convex} Section 5.5.3 that KKT conditions are sufficient for optimality. It follows that $\qty{x_i}_{i\in H}$ defined in \eqref{xidef} is the optimal solution of \eqref{optpf}. The optimal choice of $\qty{\rhat_i}$ then follows \Cref{Equivalence}.

We now consider the case where $h_1<n$. We adopt a direct approach of optimizing $\mathfrak{B}^{\mathsf{avg}}$; a similar approach to Case 1 optimizing \eqref{optpf} using KKT condition yields the same result. 

\textbf{Case 2: $h_1\le n$.} We first discuss the case $h_1<n$. Let us consider the following choice of $\qty{\rhat_i}_{i\in H}$
    \begin{equation}\label{finite}
        \rhat_i = \begin{cases}
             c_i  & \text{for } i\leq h_1 \\
            \frac{\eta_i}{r} & \text{for } h_1<i\leq h \\
             \text{any value} & \text{for } i> h 
        \end{cases},
    \end{equation}
for $r>0$ and $c_i\in [0, +\infty)$. With this choice, it follows from \eqref{fp} that
\begin{equation}\label{fpdiff}
    h-n=\sum_{i\le h_1} \frac{c_i}{\eta_i b_0}+\frac{h-h_1}{1+r b_0}.
\end{equation}
From this we obtain a lower bound on $b_0$
\begin{equation*}
    h-n-\frac{h-h_1}{1+rb_0}=\sum_{i\le h_1} \frac{c_i}{ \eta_i b_0}\ge 0 \implies b_0 \ge \frac{1}{r} \frac{n-h_1}{h-n}.
\end{equation*}
Using \eqref{simplifedAs}, we then have
\begin{equation*}
    \mathfrak{B}^{\mathsf{avg}}=\frac{1}{q}\sum_{i\in H} \frac{\vs_i}{\qty(1+\frac{\eta_i}{\rhat_i} b_0)^2}=\frac{1}{q} \sum_{i\le h_1} \frac{c_i^2\vs_i}{\qty(c_i+\eta_i b_0)^2}\le \frac{1}{q} \sum_{i\le h_1} \frac{c_i^2\vs_i r^2}{\qty(c_ir+\eta_i \frac{n-h_1}{h-n})^2}.
\end{equation*}
where we used lower bound on $b_0$ at the last inequality. Note that $\mathfrak{B}^{\mathsf{avg}}\to 0$ as we take $r\to 0$. Since $\mathfrak{B}^{\mathsf{avg}}$ is a non-negative quantity, $\qty{\rhat_i}_{i\in H}$ in \eqref{PCRest}, which is the limit of $\qty{\rhat_i}_{i\in H}$ in \eqref{finite} as $r\to 0$, must be optimal. 

For the case where $h_1=n$, we may obtain from \eqref{fpdiff} that
\begin{equation*}
b_0=\frac{1}{2}\left(\sqrt{\left(\frac{1}{h-n} \sum_{i \leq n} \frac{c_i}{\eta_i}\right)} \sqrt{\left(\frac{1}{h-n} \sum_{i \leq n} \frac{c_i}{\eta_i}\right)+\frac{4}{r}}+\frac{1}{h-n} \sum_{i \leq n} \frac{c_i}{\eta_i}\right).
\end{equation*}
which diverges to $+\infty$ as we take $r\to 0$. We thus have
\begin{equation*}
    \mathfrak{B}^{\mathsf{avg}}=\frac{1}{q}\sum_{i\in H} \frac{\vs_i}{\qty(1+\frac{\eta_i}{\rhat_i} b_0)^2}=\frac{1}{q} \sum_{i\le h_1} \frac{c_i^2\vs_i}{\qty(c_i+\eta_i b_0)^2}
\end{equation*}
goes to $0$ as we take $r\to 0$ as $h_1<n$ case. This concludes the proof.

% Consider the choice of $\qty{x_i}_{i\in H}$ and $\rho, \qty{\rzero}_{i\in H},\qty{\rone}_{i\in H}$
% \begin{equation}\label{choicexrlow}
%     \rho=0=\rzero=\rone=0,\;\;\forall i\in H, \qquad x_i=\left\{\begin{array}{c}
% 0, \quad \forall i \in \qty{1,...,h_1}, \\
% \frac{h-n}{h-h_1}, \quad \forall i \in \qty{h_1+1,...,h}
% \end{array}\right..
% \end{equation}
% We now show that the KKT conditions hold for the above choices. Firstly note that dual feasibility and complementary slackness trivially holds. Given the choices of KKT multipliers, the stationarity condition can be rewritten as
% \begin{equation*}
%     \vs_i x_i =0,\quad \forall i \in H
% \end{equation*}
% which follows from that $x_i=0,\forall i \in \qty{1,...,h_1}$ and  that $\vs_i=0, \forall i \in \qty{h_1+1,...,h}$. Since $h_1<n$, we see that $x_i\in [0,1]$; we also have that
% \begin{equation*}
%     \frac{1}{h}\sum_{i\in H} x_i = \frac{1}{h}\sum_{h_1<i\le h} \frac{h-n}{h-h_1}=1-\frac{n}{h}.
% \end{equation*}
% This proves that the primal feasibility also holds. Since $\qty{x_i}_{i\in H}$ in \eqref{choicexrlow} satisfies the KKT condition. A same argument as in Case shows that $\qty{x_i}_{i\in H}$ is the optimizer of \eqref{optpf} and the optimized objective of \eqref{optpf} is precisely zero. 
\end{proof}

\Cref{propgenA} below justifies the statement of \Cref{generalA1}.
\begin{proposition}[Fully optimized $\Bhat$ w.r.t. $\mathcal{V}$]\label{propgenA}
     Suppose $h>n$. $\mathcal{V}$ is minimized by
    \begin{equation}\label{eq_inv_align-1}
        \qbf_i=\ubf_i,\forall i \in \{1,...,p\}, \qquad \rhat_i = \begin{cases}
             c\eta_i & \text{for } i\leq h \\
            \text{any value} & \text{for } i> h 
        \end{cases},
    \end{equation}
    for any $c>0$, to the optimal value $(h/n-1)^{-1}$.
\end{proposition}

\begin{proof}[Proof of \Cref{propgenA}]
    Recall from the proof of \Cref{bA1A2size} in \Cref{app:riskchar} that given the change of variable
    $$x_i \gets \frac{1}{1+\that_i b_0},\forall i\in H$$
    we have
    \begin{equation}\label{optA1gen}
        \mathcal{V}:=\frac{\sum_{i\in H} \frac{(\that_i b_0)^2}{(1+\that_i b_0)^2}}{\sum_{i\in H} \frac{\that_i b_0}{(1+\that_i b_0)^2}}=\frac{\frac{2 n}{h}-1+\frac{1}{h} \sum_i x_i^2}{1-\frac{n}{h}-\frac{1}{h} \sum_i x_i^2}
    \end{equation}
    where $\qty{x_i}_{i\in H}$ is subjected to the constraints
    \begin{equation*}
        \frac{1}{h}\sum_{i\in H} x_i = 1-\frac{n}{h}, \qquad x_i \in [0,1],\forall i\in H.
    \end{equation*}
    This is precisely the optimization problem \eqref{optpf} we studied in the proof of \Cref{thm2} where we obtained the the optimal solution is $x_i=1-n/h, \forall i \in H$. However, since we no longer assumes \eqref{matchbase}, we cannot apply \Cref{Equivalence}. Instead, we observe the candidate solution
    $$\Gammahat\gets \sum_{i\in H} c\cdot \eta_i \cdot \ubf_i \ubf_i^\top + \sum_{i\in H^c} \rhat_i \cdot \ubf_i \ubf_i^\top$$
    for any $c>0,\rhat_i\neq 0$ satisfies that
    \begin{equation*}
    \Gammahat^{-1/2} \Sigmabf \Gammahat^{-1/2}=\sum_{i\in H}  c^{-1}\what_i \what_i^\top.
    \end{equation*}
    This implies that $\that_i = c^{-1}, \forall i \in H$. We may plug them into the fixed point equation \eqref{fp} to solve for $b_0$. We obtain
    \begin{equation*}
        b_0 =\frac{n}{h-n}\cdot c
    \end{equation*}
    which implies that $x_i=\frac{1}{1+\that_i b_0},\forall i\in H$. In short, choosing $\qbf_i \gets \ubf_i,\forall i$ as in \eqref{matchbase} and $\rhat_i\gets c\eta_i,\forall i$ as in \Cref{thm2} gives the optimal solution for \eqref{optA1gen}. This concludes the proof.

\end{proof}

\section{Discussion and Future Directions}
We find that despite its simplicity, our simple, exactly solvable model yields new insights regarding optimal pretraining in transfer learning. We discover the relative importance of learning shared structure in the features and covariates for controlling the downstream performance of any estimator, with their behavior even showing a phase transition.

In future work, we plan to extend our model to consider distribution shift in the covariates for different downstream tasks. This requires introducing a hierarchical model for the distribution shift, which was outside the scope of this work---although we expect our method of analysis to carry over. Another direction is to extend the results from linear models to random feature regression to make a tighter connection to neural networks.

\bibliographystyle{plain}
\bibliography{ref}

\appendix

\section{Proof of \Cref{explicitsol}}\label{app:explicit}
In this section, we prove \Cref{explicitsol}, which provides explicit form ${\alphahatbf} $ and ${\betahatbf} $.

\begin{proof}[Proof of \Cref{explicitsol}]
We first find $\alphabf\gets {\alphahatbf} (\betabf)$ that minimizes $\Loss (\betabf, \alphabf)$ in \eqref{loss} for a fixed $\betabf$. Dropping terms in $\Loss (\betabf, \alphabf)$  that does not involve $\alphabf$, we obtain that
\begin{equation}\label{alphapf}
    {\alphahatbf} (\betabf)=\argmin_{\alphabf} \Loss (\betabf, \alphabf)=\argmin_{\alphabf} \lambdaB\norm{\betabf-\Bhat\alphabf}_2^2+\lambdaa \norm{\alphabf}_2^2=\qty(\Bhat^\top \Bhat + \frac{2 \lambdaa}{\lambdaB} \Ibf)^{-1}\Bhat^\top \betabf
\end{equation}
where the last equality follows from standard ordinary least squares (OLS) formulas. We may then plug $\alphabf\gets {\alphahatbf} (\betabf)$ into $\Loss (\betabf, \alphabf)$ and obtain
\begin{equation*}
    \begin{aligned}
        \Loss (\betabf, {\alphahatbf} (\betabf)) &=\norm{\ybf -\Xbf  \betabf}_2^2+\lambdazero\cdot\bigg(\lambdaB \norm{\betabf-\Bhat {\alphahatbf}  (\betabf)}_2^2+\lambdaa \norm{{\alphahatbf}  (\betabf)}_2^2+\lambdab\norm{\betabf}_2^2 \bigg)\\
        &=\norm{\ybf -\Xbf  \betabf}_2^2+\lambdazero\cdot\bigg(\lambdaB \norm{\qty(\Ibf-\Bhat \qty(\Bhat^\top \Bhat + \frac{2 \lambdaa}{\lambdaB})^{-1}\Bhat^\top \betabf)\betabf}_2^2\\&\qquad \qquad \qquad \qquad \qquad \qquad  +\lambdaa \norm{\qty(\Bhat^\top \Bhat + \frac{2 \lambdaa}{\lambdaB})^{-1}\Bhat^\top \betabf}_2^2+\lambdab\norm{\betabf}_2^2 \bigg)\\
        &=\norm{\ybf -\Xbf  \betabf}_2^2+\lambdazero\cdot \betabf^\top \Gammahat \betabf
    \end{aligned}
\end{equation*}
where
\begin{equation}\label{GammaBhatform}
    \Gammahat=\lambdaB\cdot \qty(\Ibf-\Bhat \qty(\Bhat^\top \Bhat + \frac{2 \lambdaa}{\lambdaB})^{-1}\Bhat^\top)^2+\lambdaa \cdot \Bhat \qty (\Bhat^\top \Bhat +\frac{2\lambdaa}{\lambdaB}\cdot \Ibf)^{-2}\Bhat^\top + \lambdab\cdot \Ibf.
\end{equation}
It then follows from standard OLS formulas that
\begin{equation*}
    \betahatbf =\qty({\Xbf }^\top\Xbf +\lambdazero\cdot \Gammahat)^{-1} {\Xbf }^\top \ybf=\Gammahat^{1/2} \qty(\Gammahat^{-1/2}{\Xbf }^\top\Xbf  \Gammahat^{-1/2}+\lambdazero\cdot \Ibf)^{-1} \Gammahat^{-1/2} {\Xbf }^\top \ybf.
\end{equation*}
Using a well-known result of pseudo-inverse (see e.g. \cite{ward1977limit}, Section 1), we obtain that
\begin{equation}\label{betapf}
    \lim_{\lambda_0 \to 0} \betahatbf =\Gammahat^{1/2} \qty(\Gammahat^{-1/2}{\Xbf }^\top\Xbf  \Gammahat^{-1/2})^+ \Gammahat^{-1/2} {\Xbf }^\top \ybf=\Gammahat^{-1}{\Xbf }^{\top}\qty({\Xbf } \Gammahat^{-1} {\Xbf }^\top)^{+}\ybf  
\end{equation}
where the second equality follows from well-known identity of pseudo-inverse \cite{golub2013matrix}, Section 5.5.2. The claim \eqref{sol} then follows from \eqref{alphapf} and \eqref{betapf}. 

It remains to show $\Gammahat=\Qhat^\top \Lambdahat \Qhat$ for $\Lambdahat$ defined in \eqref{ri}. Let us first plug in SVD representation $\Bhat=\Qhat^\top \Dhat \Ohat$ into \eqref{GammaBhatform} where we recall that $\Qhat \in \R^{p\times p}, \Ohat \in \R^{k \times k}$ are orthogonal matrices and $\Dhat \in \R^{p\times k}$ is diagonal matrix. We obtain that
\begin{equation*}
    \Gammahat=\Qhat^\top \underbrace{\qty(\lambdaB\cdot \qty(\Ibf-\Dhat \qty(\Dhat^\top \Dhat+\frac{2\lambdaa}{\lambdaB}\cdot \Ibf)^{-1}\Dhat^\top)^2+\lambdaa\cdot \Dhat \qty(\Dhat^\top \Dhat+\frac{2\lambdaa}{\lambdaB}\cdot \Ibf)^{-2} \Dhat^\top+\lambdab\cdot \Ibf)}_{=:\Lambdahat} \Qhat.
\end{equation*}
Note that $\Lambdahat$ contains matrix quantities
\begin{equation*}
    \Dhat \qty(\Dhat^\top \Dhat+\frac{2\lambdaa}{\lambdaB}\cdot \Ibf)^{-1}\Dhat^\top,\qquad \Dhat \qty(\Dhat^\top \Dhat+\frac{2\lambdaa}{\lambdaB}\cdot \Ibf)^{-2}\Dhat^\top
\end{equation*}
Straightforward algebra manipulation shows that both are diagonal with diagonal entries as follows: for $r\in \{1,2\}$,
\begin{equation*}
    \qty(\Dhat \qty(\Dhat^\top \Dhat+\frac{2\lambdaa}{\lambdaB}\cdot \Ibf)^{-r}\Dhat^\top)_{ii}=\frac{\dhat_i^2}{\qty(\dhat_i^2+\frac{2\lambdaa}{\lambdaB})^r}, \quad i=1,\ldots,p.
\end{equation*}
where we recall $\dhat_i:=\qty(\Dhat \bm{1}_{k\times 1})_{i}, i=1,\ldots,p.$ It follows that $\Lambdahat$ is also diagonal with diagonal entries given by \eqref{ri}. This concludes the proof. 
\end{proof}

\section{Downstream Risk Concentration}\label{app:downstreamrisk}
We prove \Cref{concenprop} in this section. Our result is based on the Hanson-Wright inequality (see e.g. \cite{vershyninhanson}) which we recall for reader's convenience.
\begin{lemma}[Hanson-Wright Inequality]\label{Hanson}
    Let $\xibf$ be some random vector satisfying conditions in \Cref{conclemma} and $\mathbf{A}$ be any $n\times n$ real-valued matrix. There exists an absolute constant $c>0$ such that for any $x>0$, with probability larger than $1-\exp(-x)$, we have
    \begin{equation*}
        \xibf^\top \mathbf{A} \xibf-\mathbb{E}\qty(\xibf^\top \mathbf{A} \xibf) \leq c M^2\|\mathbf{A}\|_\mathrm{op} x+c M^2\|\mathbf{A}\|_{F} \sqrt{x}.
    \end{equation*}
    where $\norm{\cdot}_F$ denotes Frobenius norm. 
\end{lemma}

We first prove the following lemma using Hanson-Wright inequality.

\begin{lemma}\label{conclemma}
    Given \Cref{Assum} and \eqref{concen}, there exists a constant $C(M,D)>0$ such that for any $D>0$, with probability at least $1-2 q^{-D}$, 
    \begin{equation*}
        \abs{\mathfrak{B}  -\mathfrak{B}^{\mathsf{avg}}}\le C\cdot M\cdot \normop{\BBstar } \cdot \sqrt{\frac{D\log q}{q} }.
    \end{equation*}
\end{lemma}

\begin{proof}[Proof of \Cref{conclemma}]
    Let us introduce the notation
    \begin{equation*}
    \mathbf{A} := \frac{1}{q} \Sigmabf_{\alphastar}^{1/2} {\Bstar}^{\top} \Gammahat^{1/2} \qty(\sum_{i \in H} \frac{\that_i \what_i \what_i^\top}{(1+\that_i b_0)^2}) \Gammahat^{1/2} \Bstar \Sigmabf_{\alphastar}^{1/2}. 
    \end{equation*}
    Then, we can write
    \begin{equation*}
        \abs{\mathfrak{B} -\mathfrak{B}^{\mathsf{avg}}}=\abs{\xibf^\top \mathbf{A} \xibf - \E \xibf^\top \mathbf{A} \xibf}.
    \end{equation*}
    Now, using sub-multiplicativity of matrix norm and \eqref{concen}, we obtain that
    \begin{equation*}
        \normop{\mathbf{A}}\le \frac{M}{q} \cdot \normop{\BBstar} \cdot \normop{\Gammahat}\cdot \max_{i \in H} \qty{\frac{\that_i}{1+\that_i\cdot b_0}}. 
    \end{equation*}
    Using the elementary inequality $x/(1+x b_0)\le (4b_0)^{-1}$, \Cref{bdthat}, \Cref{bA1A2size} and \Cref{Assum}, we obtain that
    \begin{equation*}
        \max_{i \in H} \qty{\frac{\that_i}{1+\that_i\cdot b_0}}\le \frac{1}{4b_0}\le M. 
    \end{equation*}
    It then follows from \Cref{Assum} and the above that
    \begin{equation*}
        \normop{\mathbf{A}}\le \frac{M}{q} \cdot \normop{\BBstar}.
    \end{equation*}
    Using the elementary matrix inequality $\norm{\mathbf{A}}_F\le \sqrt{\mathrm{rank}(\mathbf{A})} \cdot \normop{\mathbf{A}}$ and the fact that
    \begin{equation*}
        \mathrm{rank}(\mathbf{A})\le \min \qty(\mathrm{rank}\qty(\Sigmabf_{\alphastar}), \mathrm{rank} \qty(\Bstar)))\le q,
    \end{equation*}
    we have that
    \begin{equation*}
        \norm{\mathbf{A}}_F \le \sqrt{\min \qty(\mathrm{rank}\qty(\Sigmabf_{\alphastar}), \mathrm{rank} \qty(\Bstar)))} \cdot \frac{M}{q} \cdot \normop{\BBstar}\le \frac{M}{\sqrt{q}}\normop{\BBstar}.
    \end{equation*}
    Applying Hanson-Wright inequality, we obtain that, with probability greater than $1-2\exp(-x)$,
    \begin{equation}
        \abs{\xibf^\top \mathbf{A} \xibf - \E \xibf^\top \mathbf{A} \xibf} \leq C(M)\cdot \qty(\frac{x}{q}  +\sqrt{\frac{x}{q}}) \cdot \normop{\BBstar}. 
    \end{equation}
    The required statement follows if we set $x\gets D\cdot \log q$.
\end{proof}

We are now ready to prove \Cref{concenprop}.

\begin{proof}[Proof of \Cref{concenprop}]
Using triangle inequality, we have that
\begin{equation*}
    \abs{R -\Rcal^{\mathsf{avg}}}\le \abs{R -\Rcal }+\abs{\Rcal -\Rcal^{\mathsf{avg}}}=\abs{R -\Rcal }+(1+\mathcal{V})\cdot \abs{\mathfrak{B}  -\mathfrak{B}^{\mathsf{avg}}}
\end{equation*}
Using the above, \Cref{bdthat} and \Cref{bA1A2size}, we obtain
\begin{equation*}
    \abs{R -\Rcal^{\mathsf{avg}}}\le \abs{R -\Rcal }+C(M)\cdot \abs{\mathfrak{B}  -\mathfrak{B}^{\mathsf{avg}}}.
\end{equation*}
The result then follows from \Cref{thm}, \Cref{conclemma} and an application of the union bound. 
    
\end{proof}

\section{Minimax Optimality}\label{appminimax}
Prior-averaged optimality is proposed to minimize risk averaged across potential downstream tasks. It however does not control the worst-case downstream risk. In this section, we assume that 
\begin{equation}\label{assnorm}
    \norm{{\alphastar}}^2 \le \cfrak,\forall j\in \{1,\ldots,J\},
\end{equation} 
and seek to choose $\Bhat$ to control the worst possible downstream risk  
$\max_{{\alphastar} \in \mathbb{B}^q(\sqrt{\cfrak})} R .$

Note that the wosrt-case downstream asymptotic risk is $\Rcal $ maximized over ${\alphastar} \in \mathbb{B}^q(\sqrt{\cfrak})$, which admits a closed form as follows,
\begin{equation}\label{Rworstdef}
    \begin{aligned}
        \Rcal^{\mathsf{worst}}(\Bhat, \bm{\lambda}, \Sigmabf, \Bstar) &:= \max_{{\alphastar} \in \mathbb{B}^q(\sqrt{\cfrak})} \Rcal  \\& =  \sigma^2 \mathcal{V}+\left(\mathcal{V}+1\right) \mathfrak{B}^{\mathsf{worst}}
    \end{aligned}
\end{equation}
where, with $\eigtop(\cdot)$ denoting the top eigenvalue,
\begin{equation}\label{minii}
    \begin{aligned}
        \mathfrak{B}^{\mathsf{worst}} &:= \max_{{\alphastar} \in \mathbb{B}^q(\sqrt{\cfrak})} \mathfrak{B} \qty(\Bhat, \bm{\lambda}, {\alphastar},\Sigmabf,\Bstar) \\ &= \cfrak \cdot \eigtop \qty (\sum_{i \in H} \frac{\that_i \cdot  {\Bstar}^\top \Gammahat^{\frac{1}{2}} \what_i \what_i^\top \Gammahat^{\frac{1}{2}} \Bstar}{\left(1+\that_i b_0\right)^2}). 
    \end{aligned}
\end{equation}
The second equality above follows from variational characterization of eigenvalues (see e.g. \cite{tao2023topics}, Theorem 1.3.2). This representation of $\mathfrak{B}^{\mathsf{worst}}$ is quite convenient from an optimization point of view since it allows us to avoid solving a bi-level optimization problem. That is, our optimization problem may be written as
\begin{equation*}
    \min_{\Bhat, \bm{\lambda}} \max_{{\alphastar} \in \mathbb{B}^q(\sqrt{\cfrak})} \Rcal =\min_{\Bhat, \bm{\lambda}} \sigma^2 \mathcal{V}+\cfrak\cdot (\mathfrak{B}+1)\cdot \eigtop \qty (\sum_{i \in H} \frac{\that_i \cdot  {\Bstar}^\top \Gammahat^{\frac{1}{2}} \what_i \what_i^\top \Gammahat^{\frac{1}{2}} \Bstar}{\left(1+\that_i b_0\right)^2}).
\end{equation*}
Instead of tackling the bi-level optimization on the LHS, we may minimize RHS using backpropagation routine developed in \Cref{app:obtainfully}. For the latter, note that it is standard to differentiate through eigendecomposition and thus $\eigtop(\cdot)$. See \Cref{app:obtainfully} for more detailed discussion and pointers to relevant reference. Furthermore, as we will show in \Cref{sectionOptConv}, $\mathfrak{B}^{\mathsf{worst}}$ may be written into a convex objective under linear constraint under the ``spectrum-only'' case discussed in \Cref{sectionSO}; this allows us to design efficient convex programs for convex relaxation of $\Rcal^\mathsf{worst}$. 

We may then define minimax-optimal $\Bhat$ as the one that minimizes the worst case risk.
\begin{definition}[Minimax-optimal pretraining]\label{minimax}
    Minimax optimal pretraining consists of two stages: (i) learn $\Bstar$ and $\Sigmabf$ from pretraining data; (ii) choose optimal feature $\Bhat \in \R^{p\times k}$ and regularization parameters $\bm{\lambda}\in \R^3_+$ by minimizing $\Rcal^{\mathsf{worst}}$. 
\end{definition}
The ensuing result provides an justification for \Cref{minimax}, showing that the objective $\Rcal^{\mathsf{worst}}$, minimized to determine $\Bhat$, tends to approximate the actual worst-case risk.
\begin{proposition}\label{miniprop}
    Suppose that \Cref{Assum} and \eqref{assnorm} hold. Then, for any $D>0$, there exists $C=C(M,D)$ such that, with probability at least $1-Cn^{-D}$, 
    \begin{equation*}
        \abs{\max_{{\alphastar} \in \mathbb{B}^q(\sqrt{\cfrak})} R -\Rcal^{\mathsf{worst}}} \le C\cdot \cfrak n^{-1/7}\normop{\Bstar {\Bstar}^\top}.
    \end{equation*}
\end{proposition}

\begin{proof}[Proof of \Cref{miniprop}]
    Note that 
    $$
    \abs{R -\Rcal }\leq Cn^{-1/7}\norm{{\st}}_2^2=Cn^{-1/7} {\alphastar}^\top \BBstar \alphastar
    $$
    from \Cref{thm} implies the following
    \begin{equation*}
        \begin{aligned}
             \max_{{\alphastar} \in \mathbb{B}^q(\sqrt{\cfrak})} R  & \ge \max_{{\alphastar} \in \mathbb{B}^q(\sqrt{\cfrak})}\Rcal  -Cn^{-1/7} \max_{{\alphastar} \in \mathbb{B}^q(\sqrt{\cfrak})} {\alphastar}^\top \BBstar \alphastar\\
             & =\Rcal^{\mathsf{worst}}+Cn^{-1/7} {\alphastar}^\top \BBstar \alphastar \\
            \max_{{\alphastar} \in \mathbb{B}^q(\sqrt{\cfrak})} R  & \le \max_{{\alphastar} \in \mathbb{B}^q(\sqrt{\cfrak})} \Rcal  +Cn^{-1/7} \max_{{\alphastar} \in \mathbb{B}^q(\sqrt{\cfrak})} {\alphastar}^\top \BBstar \alphastar\\ &=\Rcal^{\mathsf{worst}}-Cn^{-1/7} {\alphastar}^\top \BBstar \alphastar
        \end{aligned}
    \end{equation*}
    where we have used the variational representation of top eigenvalue to obtain the equalities above. This concludes the proof. 
\end{proof}

\section{Convexity of Spectrum-Only Optimization and Convex Programs}\label{sectionOptConv}
We restrict ourselves to the spectrum-only optimization in this section. As discussed in \Cref{sectionSO}, instead of optimize the entire $\Bhat$, we fix eigenvectors as
\begin{equation}
    \qbf_i=\ubf_i,\forall i\in \{1,\ldots,p\}.
\end{equation}
and optimize only the eigenvalues $\{\dhat_i^2\}_{i=1}^p$ and regularization parameters $\bm{\lambda}$. 

Recall that our goal is to optimize the prior-averaged-optimal and minimax-optimal objectives
\begin{equation*}
    \Rcal^{\mathsf{avg}}=\mathfrak{B}^{\mathsf{avg}} +\left(\mathfrak{B}^{\mathsf{avg}} + \sigma^2\right) \mathcal{V}, \qquad \Rcal^{\mathsf{worst}}=\mathfrak{B}^{\mathsf{worst}} +\left(\mathfrak{B}^{\mathsf{worst}} + \sigma^2\right) \mathcal{V}
\end{equation*}
which are defined in \Cref{section4} and \Cref{appminimax} respectively. In this section, we cast the above optimization problems in a more general form
\begin{equation}
    \min_{\{\rhat_i\}_{i=1}^p} f\qty(\mathcal{V}, \mathfrak{B}^{\mathsf{avg}}, \mathfrak{B}^{\mathsf{worst}})
\end{equation}
where $f$ is arbitrary function. We show in \Cref{sectionEquiv} below that to solve the above, we may solve a linearly constraint optimization problem
\begin{equation*}
    \min_{\xH \in [0,1]^h} f\qty(\tilde{\mathcal{V}}, \tilde{\mathfrak{B}}^{\mathsf{avg}}, \tilde{\mathfrak{B}}^{\mathsf{worst}})  \quad \mathrm{s.t.}\quad \sum_{i\in H} \frac{x_i}{h} =1-\frac{n}{h}
\end{equation*}
where $\tilde{\mathcal{V}}, \tilde{\mathfrak{B}}^{\mathsf{avg}}$ and $\tilde{\mathfrak{B}}^{\mathsf{worst}}$ are objectives after applying certain change of variables to $\mathcal{V}, \mathfrak{B}^{\mathsf{avg}}$ and $\mathfrak{B}^{\mathsf{avg}}$. We give their definition in \Cref{sectionEquiv}. A convenient fact that we prove in \Cref{sectionConv} is that $\tilde{\mathcal{V}}, \tilde{\mathfrak{B}}^{\mathsf{avg}}$ and $\tilde{\mathfrak{B}}^{\mathsf{worst}}$ are in in fact convex functions of $\xH$. The above results are important for proving \Cref{thm2}. As a byproduct, we also provide convex relaxation of the objectives $\Rcal^{\mathsf{avg}}$ and

\subsection{Equivalence}\label{sectionEquiv}
Given $x_i\in[0,1]$ for $i\in \{1,\ldots,p\}$, we define a new objective
\begin{equation}\label{deftildeA}
    \tilde{\mathcal{V}}(\xH):=\frac{2n-h+\sum_{i\in H} x_i^2}{h-n-\sum_{i\in H} x_i^2},\quad \tilde{\mathfrak{B}}^{\mathsf{avg}}(\xH):=\frac{1}{q}\sum_{i\in H} \vs_i  x_i^2, \quad \tilde{\mathfrak{B}}^{\mathsf{worst}}:= \mathfrak{c}\cdot \eigtop \qty({\Bstar}^\top \sum_{i\in H} \eta_i x_i^2 \cdot  \ubf_i \ubf_i^\top \Bstar)
\end{equation}
where $\xH \in [0,1]^h$ with entries $x_i,i\in H$ and 
\begin{equation}
    \vs_i:=\eta_i \cdot \ubf^\top_i \Bstar \Sigmabf_{\alphastar} {\Bstar}^\top \ubf_i, \quad i=1,\ldots,p.
\end{equation}

The following proposition shows that the original optimization problem of optimizing \eqref{or} with respect to $\{r_i\}_{i=1}^p$ may be cast into the new optimization problem \eqref{eqopt} in terms of $\xH=\qty{x_i}_{i\in H}$. 

\begin{proposition}[Equivalence]\label{Equivalence}
Given \eqref{matchbase}, the solution $\{r_i\}_{i=1}^p$ of the optimization problem 
\begin{equation}\label{or}
    \min_{\{\rhat_i\}_{i=1}^p} f\qty(\mathcal{V}, \mathfrak{B}^{\mathsf{avg}}, \mathfrak{B}^{\mathsf{worst}})
\end{equation}
for arbitrary function $f(\cdot)$ may be obtained by (i) solving the linearly constrained optimization problem
\begin{equation}\label{eqopt}
    \min_{\xH \in [0,1]^h} f\qty(\tilde{\mathcal{V}}, \tilde{\mathfrak{B}}^{\mathsf{avg}})  \quad \mathrm{s.t.}\quad \sum_{i\in H} \frac{x_i}{h} =1-\frac{n}{h}
\end{equation}
and (ii) assigning $\rhat_i \gets c\eta_i\qty(\frac{x_i}{1-x_i}), i\in H$ for any $c>0$ and any values to $\rhat_i, i\notin H$. 
\end{proposition}

\begin{proof}[Proof of \Cref{Equivalence}]
Given \eqref{matchbase}, we have that
\begin{equation*}
    \Gammahat^{-1/2} \Sigmabf \Gammahat^{-1/2} = \sum_{i\in H} \frac{\eta_i}{\rhat_i } \ubf_i \ubf_i^\top.
\end{equation*}
We may then take $\that_i \gets \eta_i/\rhat_i, \what_i \gets \ubf_i$ and simplify the expressions of $\mathcal{V}, \mathfrak{B}^{\mathsf{avg}}$ and $\mathfrak{B}^{\mathsf{worst}}$ to the following
\begin{equation}\label{simplifedAs}
\mathcal{V}= \frac{\sum_{i \in H} \frac{\left(\frac{\eta_i}{\rhat_i} b_0\right)^2}{\left(1+\frac{\eta_i}{\rhat_i} b_0\right)^2}}{\sum_{i \in H} \frac{\frac{\eta_i}{\rhat_i} b_0}{\left(1+\frac{\eta_i}{\rhat_i} b_0\right)^2}}, \quad \mathfrak{B}^{\mathsf{avg}}=\frac{1}{q}\sum_{i\in H} \frac{\eta_i \cdot \ubf_i^\top \Bstar \Sigmabf_{\alphastar} {\Bstar}^\top \ubf_i}{\qty(1+\frac{\eta_i}{\rhat_i} b_0)^2}, \quad \mathfrak{B}^{\mathsf{worst}}=\cfrak\cdot \eigtop \qty({\Bstar}^\top \sum_{i \in H} \frac{\eta_i \cdot \ubf_i \ubf_i^\top  }{\qty(1+\frac{\eta_i}{\rhat_i} b_0)^2}\Bstar).
\end{equation}
where $b_0$ is defined as the unique solution of the following equation
\begin{equation}\label{fpapp}
    1-\frac{n}{h}=\frac{1}{h} \sum_{i\in H} \frac{1}{1+\frac{\eta_i}{\rhat_i} b_0}. 
\end{equation}
Now consider the change of variable $x_i \gets \frac{1}{1+\frac{\eta_i}{\rhat_i} b_0}$, upon which $\mathcal{V}, \mathfrak{B}^{\mathsf{avg}}$ and $\mathfrak{B}^{\mathsf{worst}}$ becomes $\tilde{\mathcal{V}}, \tilde{\mathfrak{B}}^{\mathsf{avg}}$ and $\tilde{\mathfrak{B}}^{\mathsf{worst}}$ defined in \eqref{deftildeA}, respectively, and \eqref{fpapp} becomes the linear constraint $\frac{1}{h}\sum_{i\in H} x_i =1-\frac{n}{h}$. For $\mathcal{V}$ in particular, we used the identities in \eqref{idA1}. Let us defined optimized objectives for \eqref{or} and \eqref{eqopt} to be $\mathcal{F}$ and $\tilde{\mathcal{F}}$ respectively. The above discussion implies that 
\begin{equation}\label{optfF}
    \tilde{\mathcal{F}}\le \mathcal{F}.
\end{equation}
The result then follows from the observation that if we set  $\rhat_i \gets c\eta_i\qty(\frac{x_i}{1-x_i}), i\in H$ for any $c>0$, we would have \eqref{fpapp} holds for $b_0=c$ and the objective of \eqref{or} would evaluate to $\tilde{\mathcal{F}}$, which is the optimal value by \eqref{optfF}. This concludes the proof. 

\end{proof}

\subsection{Convexity}\label{sectionConv}
We now show that the objectives $\tilde{\mathcal{V}}, \tilde{\mathfrak{B}}^{\mathsf{avg}}$ and $\tilde{\mathfrak{B}}^{\mathsf{worst}}$ are convex.

\begin{proposition}[Convexity]\label{Convexity}
The objectives $\tilde{\mathcal{V}}$ and $\tilde{\mathfrak{B}}^{\mathsf{avg}}$ are convex functions of $\xH$ on the convex set $\mathcal{D}:=\qty{\xH \in [0,1]^h: h^{-1}\sum_{i\in H} x_i=1-n/h}$.
\end{proposition}

\begin{proof}[Proof of \Cref{Convexity}]
We first note that 
$$\mathcal{D}:=\qty{\xH \in [0,1]^h: \frac{1}{h}\sum_{i\in H} x_i=1-\frac{n}{h}}$$
is indeed a convex set as it is intersection between convex set $[0,1]^h$ and linear subspace $$\qty{\xH\in \R^h: h^{-1}\sum_{i\in H} x_i=1-n/h}.$$ We prove below that $\tilde{\mathcal{V}}, \tilde{\mathfrak{B}}^{\mathsf{avg}}, \tilde{\mathfrak{B}}^{\mathsf{worst}}$ are convex functions of $\xH$ on $\mathcal{D}$. 

\textbf{Convexity of $\tilde{\mathcal{V}}$.} Let us define the function $g: [0,h-n] \mapsto \R$ as $g(S)=\frac{2n-h+S}{h-n-S}$. It has the first and second derivatives
\begin{equation*}
    g^{\prime}(S)=\frac{n}{(h-n-S)^2}, \qquad g^{\prime \prime} (S)=\frac{2n}{(h-n-S)^3}
\end{equation*}
which implies that $g$ is an strictly increasing, convex function on its domain. Now note that we may write $\mathcal{V}=f \qty(\sum_{i\in H} x_i^2)$ where the inner function $\xH \mapsto \sum_{i=1}^H x_i^2$ is convex on $\mathcal{D}$. Moreover, $\xH \in \mathcal{D}$ implies that $\sum_{i\in H} x_i^2\le \sum_{i\in H} x_i= h-n$. The result follows from the fact that composition of an increasing, convex function with a convex function is convex. 

\textbf{Convexity of $\tilde{\mathfrak{B}}^{\mathsf{avg}}$.} The result follows from the fact that the Hessian of $\xH \mapsto \tilde{\mathfrak{B}}^{\mathsf{avg}}$ is positive semidefinite. To see this, note that
\begin{equation*}
    \nabla^2_{\xH} \tilde{\mathfrak{B}}^{\mathsf{avg}} = \frac{2}{q} \mathrm{diag}([\vs_{H(1)},...,\vs_{H(h)}])
\end{equation*}
and that $\vs_i=\eta_i \cdot \ubf^\top_i \Bstar \Sigmabf_{\alphastar} {\Bstar}^\top \ubf_i\ge 0, \forall i\in H$. 

\textbf{Convexity of $\tilde{\mathfrak{B}}^{\mathsf{worst}}$.} Using variational representation of eigenvalues, we have
\begin{equation*}
    \mathfrak{B}^{\mathsf{worst}}=\cfrak \cdot \eigtop \qty (\sum_{i \in H} \frac{\that_i \cdot  {\Bstar}^\top \Gammahat^{\frac{1}{2}} \what_i \what_i^\top \Gammahat^{\frac{1}{2}} \Bstar}{\left(1+\that_i b_0\right)^2})=\cfrak \cdot \max_{\ubf \in \mathbb{B}^q(1)} \ubf^\top \qty (\sum_{i \in H} \frac{\that_i \cdot  {\Bstar}^\top \Gammahat^{\frac{1}{2}} \what_i \what_i^\top \Gammahat^{\frac{1}{2}} \Bstar}{\left(1+\that_i b_0\right)^2}) \ubf.
\end{equation*}
Let us define the function $g:\mathcal{D}\times \mathbb{B}^q(1) \mapsto \R_+$ as
$$g(\xH, \ubf) = \ubf^\top \qty (\sum_{i \in H} \frac{\that_i \cdot  {\Bstar}^\top \Gammahat^{\frac{1}{2}} \what_i \what_i^\top \Gammahat^{\frac{1}{2}} \Bstar}{\left(1+\that_i b_0\right)^2}) \ubf.$$
Notice that $\xH\mapsto g(\xH, \ubf)$ is convex for any fixed $\ubf \in \mathbb{B}^q(1)$. To see this, note that its Hessian is diagonal with non-negative entries. The result then follows from Danskin's theorem, which states that $\xH \mapsto \max_{\ubf \in \mathbb{B}^q(1)}g(\xH, \ubf)$ is convex if $g(\xH, \ubf)$ is convex.     
\end{proof}

\subsection{Algorithmic Implications}\label{convexprogram}
The results from \Cref{sectionEquiv} and \Cref{sectionConv} may be leveraged to construct efficient convex algorithms. Using \Cref{sectionEquiv} results, we see that to optimize 
\begin{equation*}
    \Rcal^{\mathsf{avg}}=\mathfrak{B}^{\mathsf{avg}} +\left(\mathfrak{B}^{\mathsf{avg}} + \sigma^2\right) \mathcal{V}, \qquad \Rcal^{\mathsf{worst}}=\mathfrak{B}^{\mathsf{worst}} +\left(\mathfrak{B}^{\mathsf{worst}} + \sigma^2\right) \mathcal{V}
\end{equation*}
we only need to solve the following optimization problems
\begin{equation}\label{non-relax}
\begin{array}{ll}
\min_{\xH \in [0,1]^h} \tilde{\mathfrak{B}}^{\mathsf{avg}} +\left(\tilde{\mathfrak{B}}^{\mathsf{avg}} + \sigma^2\right) \tilde{\mathcal{V}}, & \qquad \min_{\xH \in [0,1]^h} \tilde{\mathfrak{B}}^{\mathsf{worst}} +\left(\tilde{\mathfrak{B}}^{\mathsf{worst}} + \sigma^2\right) \tilde{\mathcal{V}} \\
 \mathrm{s.t.}\quad \sum_{i\in H} \frac{x_i}{h} =1-\frac{n}{h} & \qquad \mathrm{s.t.}\quad \sum_{i\in H} \frac{x_i}{h} =1-\frac{n}{h}.
\end{array}
\end{equation}
We showed in \Cref{sectionConv} that $\tilde{\mathcal{V}}, \tilde{\mathfrak{B}}^{\mathsf{avg}}$ and $\tilde{\mathfrak{B}}^{\mathsf{worst}}$ are convex by themselves. However, due the presence of the interaction terms $\tilde{\mathfrak{B}}^{\mathsf{avg}}\cdot \tilde{\mathcal{V}}$ and   $\tilde{\mathfrak{B}}^{\mathsf{worst}}\cdot \tilde{\mathcal{V}}$, it is uncertain that the objective will remain convex. That said, the objective is in a simple enough form where most constrained local opitmization solvers may be applied. For example, we found that it is quite efficient to solve \eqref{non-relax} with the sequential quadratic programming (SQP) routine implemented in Python SciPy package \cite{2020SciPy-NMeth}. 

To arrive at a convex program, we consider the idea of convex relaxation (see \cite{li2015selected}, Chapter 6 for a review). Using the simple inequality $xy\le (\frac{x+y}{2})^2$, we have that 
\begin{equation*}
    \Rcal^{\mathsf{avg}}\le \mathfrak{B}^{\mathsf{avg}}+\qty(\frac{\mathcal{V}+\mathfrak{B}^{\mathsf{avg}}}{2})^2+\sigma^2\mathcal{V}=:\Rcal^{\mathsf{avg}}_{\mathsf{rel}},\quad \Rcal^{\mathsf{avg}}\le \mathfrak{B}^{\mathsf{worst}}+\qty(\frac{\mathcal{V}+\mathfrak{B}^{\mathsf{worst}}}{2})^2+\sigma^2\mathcal{V}=:\Rcal^{\mathsf{worst}}_{\mathsf{rel}}. 
\end{equation*}
Therefore, we may consider the following optimization problems
\begin{equation}\label{relax}
\begin{array}{ll}
\min_{\xH \in [0,1]^h} \tilde{\mathfrak{B}}^{\mathsf{avg}} +\left(\tilde{\mathfrak{B}}^{\mathsf{avg}} + \tilde{\mathcal{V}}\right)^2+\sigma^2 \tilde{\mathcal{V}}, & \qquad \min_{\xH \in [0,1]^h} \tilde{\mathfrak{B}}^{\mathsf{worst}} +\left(\tilde{\mathfrak{B}}^{\mathsf{worst}} + \tilde{\mathcal{V}}\right)^2+\sigma^2 \tilde{\mathcal{V}} \\
 \mathrm{s.t.}\quad \sum_{i\in H} \frac{x_i}{h} =1-\frac{n}{h} & \qquad \mathrm{s.t.}\quad \sum_{i\in H} \frac{x_i}{h} =1-\frac{n}{h}.
\end{array}
\end{equation}
It is easy to verify using \Cref{Convexity} that the optimization problems \eqref{relax} are convex program, which can be solved efficiently (see \cite{boyd2004convex} for a review). The solution, obtained from  $\rhat_i \gets c\eta_i\qty(\frac{x_i}{1-x_i}), i\in H$  will minimize the relaxed upper bounds, $\Rcal^{\mathsf{avg}}_{\mathsf{rel}}$ and $\Rcal^{\mathsf{worst}}_{\mathsf{rel}}$, of the original objectives, $\Rcal^{\mathsf{avg}}$ and $\Rcal^{\mathsf{worst}}$. 

\section{Obtain Fully-Optimized $\Bhat$ via Backpropogation}\label{app:obtainfully}
\begin{algorithm}[H]
\caption{Minimize $\Rcal^{\mathsf{avg}}$ or $\Rcal^{\mathsf{worst}}$ with Backpropogation}
\begin{algorithmic}[1]\label{algoBP}
\REQUIRE $\Bstar, \Sigmabf, \Sigmabf_{\alphastar}$
\STATE Initialize feature weights $\Bhat$ and regularization $\bm{\lambda}$
\REPEAT
\vspace{2mm}
\STATE {\textit{\#\# FORWARD PASS \#\#}}
\vspace{1mm}
\STATE Compute eigendecomposition
$\Bhat \Bhat^\top = \sum_{i=1}^p \dhat_i^2\cdot \qbf_i \qbf_i^\top$
\STATE Compute $\Gammahat\gets \sum_{i \in H} \rhat_i \cdot \qbf_i \qbf_i^\top$ for $\rhat_i=\rhat_i\qty(\Bhat, \bm{\lambda})$ from \eqref{ri}
\STATE Compute eigendecomposition $\Gammahat^{-1/2} \Sigmabf \Gammahat^{-1/2}=\sum_{i=1}^p \that_i \cdot \what_i \what_i^\top$
\STATE Apply Newton's method to find root $b_0$ of \eqref{fpapp}
\STATE Compute from \eqref{avgg} or \eqref{Rworstdef} the objective function 
$$L\gets \Rcal^{\mathsf{avg}}\qty(b_0,\qty{\that_i}_{i=1}^p, \qty{\what_i}_{i=1}^p ) \quad \text{or} \quad L \gets \Rcal^{\mathsf{worst}} \qty(b_0, \qty{\that_i}_{i=1}^p, \qty{\what_i}_{i=1}^p)$$
\STATE 
\vspace{2mm}
\STATE \textit{\#\# BACKWARD PASS \#\#}
\vspace{1mm}
\STATE Compute partial derivatives of $L$ w.r.t. $b_0=b_0\qty(\qty{\that_i}_{i=1}^p),$ $\qty{\that_i}_{i=1}^p,$ and gradients w.r.t. $\qty{\what_i}_{i=1}^p$:
$
\frac{\partial}{\partial b_0} L, \frac{\partial}{\partial \that_i} L,\frac{d}{d \what_i} L, i=1,...,p
$
\STATE Compute total derivatives of $L$ w.r.t. $\qty{\that_i}_{i=1}^p$
\begin{equation}
        \frac{d L}{d \that_i}\gets \begin{cases}
             0 & \text{for } i\notin H \\
            \frac{\partial L}{\partial \that_i} +\frac{\partial L}{\partial b_0} \cdot \frac{db_0}{d\that_i} & \text{for } i\in H
        \end{cases}, \quad \text{where}\;\; \frac{db_0}{d\that_i}=-\frac{\frac{b_0}{\qty(1+\that_i b_0)^2}}{\sum_{j\in H} \frac{\that_j} {\qty(1+\that_j b_0)^2}}, \forall i\in H
    \end{equation}
\STATE Compute gradients of $L$ w.r.t. $\Gammahat$ and finally $\Bhat, \bm{\lambda}$ using backpropagation formulas for eigendecompositions. 
\STATE Update feature weights $\Bhat$ and regularization $\bm{\lambda}$
\UNTIL{convergence criterion is met}
\ENSURE{$\Bhat, \bm{\lambda}$}
\end{algorithmic}
\end{algorithm}

Recall that we are provided with ground-truth featurization $\Bstar$, data covariance $\Sigmabf$ and prior knowledge on $\alphastar$ via $\Sigmabf_{\alphastar}$. The goal of this section is to discuss how to use backpropagation \cite{rumelhart1986learning} to find $\Bhat$ and $\bm{\lambda}$ that minimizes the prior-averaged objective defined in \eqref{avgg} or minimax objective defined in \eqref{Rworstdef}
\begin{equation*}
    \Rcal^{\mathsf{avg}}(\Bhat, \bm{\lambda})=\mathfrak{B}^{\mathsf{avg}} +\left(\mathfrak{B}^{\mathsf{avg}} + \sigma^2\right) \mathcal{V}, \qquad \Rcal^{\mathsf{worst}}(\Bhat, \bm{\lambda})=\sigma^2 \mathcal{V}+\left(\mathcal{V}+1\right) \mathfrak{B}^{\mathsf{worst}}
\end{equation*}
where
\begin{equation*}
\begin{aligned}
         &\mathcal{V}=\frac{\sum_{i\in H} \frac{(\that_i b_0)^2}{(1+\that_i b_0)^2}}{\sum_{i\in H} \frac{\that_i b_0}{(1+\that_i b_0)^2}}, \; \mathfrak{B}^{\mathsf{avg}}=\frac{1}{\qtrue} \sum_{i \in H} \frac{\that_i \cdot\what_i^\top \Gammahat^{\frac{1}{2}} \Bstar \Sigmabf_{\alphastar} {\Bstar}^\top \Gammahat^{\frac{1}{2}} \what_i }{\left(1+\that_i b_0\right)^2}, \\ &\mathfrak{B}^{\mathsf{worst}}=\cfrak \cdot \eigtop \qty (\sum_{i \in H} \frac{\that_i \cdot  {\Bstar}^\top \Gammahat^{\frac{1}{2}} \what_i \what_i^\top \Gammahat^{\frac{1}{2}} \Bstar}{\left(1+\that_i b_0\right)^2}).
\end{aligned}
\end{equation*}
Backpropagation consists a forward pass where the objective is computed for the current choice of learnable parameters $\Bhat, \bm{\lambda}$, and a backward pass where gradients of the optimization objective $L$ ($\Rcal^{\mathsf{avg}}$ or $\Rcal^{\mathsf{worst}}$ in our context) w.r.t. $\Bhat, \bm{\lambda}$ are computed using the chain rule. An outline of a backpropagation routine for our problem is given in \Cref{algoBP}.

Our implementation utilizes PyTorch, a widely recognized Python library that facilitates backpropagation through a computational framework known as automatic differentiation \cite{paszke2019pytorch}. We make a few remarks on implementing backpropagation for the optimization problem above. 

\paragraph{Backpropagating through the Fixed Point.} The first issue is the involvement of $b_0$. Recall that $b_0$ depends on $\qty(\that_i)_{i=1}^p$ through the fixed point equation
\begin{equation}\label{fpap}
    1-\frac{n}{h}= \frac{1}{h} \sum_{i\in H} \frac{1}{1+\that_i b_0}.
\end{equation}
During a forward pass, $b_0$ needs to be computed from a root-finding algorithm such as Newton's method (see \cite{boyd2004convex} for a review). During a backward pass, its dependencies on $\qty(\that_i)_{i=1}^p$ needs to be accounted for as we compute gradients of the objective (denoted $L$) w.r.t. $\qty(\that_i)_{i=1}^p$ via the total derivative formula
$$\frac{dL}{d\that_i}=\frac{\partial L}{\partial \that_i} +\frac{\partial L}{\partial b_0} \cdot \frac{db_0}{d\that_i},\forall i \in H.$$
Implicit differentiation of \eqref{fpap} yields
$$\frac{db_0}{d\that_i}=-\frac{\frac{b_0}{\qty(1+\that_i b_0)^2}}{\sum_{j\in H} \frac{\that_j} {\qty(1+\that_j b_0)^2}}, \forall i\in H.$$ On the PyTorch platform, the above can be easily achieved by implementing a custom PyTorch autograd function for $b_0$ by implementing \texttt{torch.autograd.Function} class. 

\paragraph{Backpropagating through Eigendecomposition.} A second issue is the presence of eigendecomposition operations in the forward pass. We remark that this is actually standard in the area of image processing and well-known formulas are available for differentiating eigen-decomposition and singular value deocmposition operations (see \cite{ionescu2015matrix} for a review). The backpropagation for eigendecomposition and singular value decomposition is already implemented in PyTorch via \texttt{torch.linalg.svd} and \texttt{torch.linalg.eigh}. We remark that a 64-bit floating point precision is needed to ensure numerical stability of these routines for our purpose. 

\paragraph{Optimize $\bm{\lambda}$ for Oracle-Featurization Predictor (OFP)} Regularization parameters 
$\bm{\lambda}$ of the Oracle-featurization predictor is optimized with respect to $\Rcal^{\mathsf{avg}}$. The procedure is the same as in \Cref{algoBP} with the exception that the feature weights $\Bhat$ is frozen to $\Bhat \gets \Bstar$. 

\paragraph{Other Technical Specifications.} We initialize $\Bhat \in \R^{p\times k}, \bm{\lambda} \in \R^3$ using the default initializations of \texttt{torch.nn.Linear} (each parameter is drawn iid from $\mathsf{Uniform}[\ell^{-1/2}, \ell^{-1/2}]$ where $\ell=k$ for $\Bhat$ and $\ell=1/3$ for $\bm{\lambda}$). We use Adam optimizer for gradient descent \cite{kingma2014adam}. This routine is implemented in PyTorch as \texttt{torch.optim.Adam}. We use all default settings except the learning rate we use is \texttt{lr=0.0001}. The stopping criteria is as follows: the gradient descent stops if the optimization objective does not improve by more than $0.1\%$ for a consecutive of 7 episodes where each episode consists of gradient descent 50 steps. The computation is carried out on a NVIDIA V100 Tensor Core GPU, accessed through Google's Colab service.

\section{Upstream Sample Complexity.}\label{app:pretrainsamplecomp}
In this section, we relax the assumption of oracle knowledge of the ground-truth representation $\Bstar$. There are many possible models for how these unknowns could be estimated from pretraining data, and our previous results are not tied to any particular setup. We consider the simple example given in \Cref{exampleupstream} and track how these estimation errors affect the downstream risk. 

\begin{example}\label{exampleupstream}
     In the upstream, the model receives training data for $q$ distinct upstream tasks $\qty(\yprebf^{(i)}, \Xprebf^{(i)})_{i=1}^{q}$ where $$\yprebf^{(i)}=\Xprebf^{(i)} {\mathbf{b} ^\star}^{(i)} + \bm{\varepsilon}^{(i)}_{\mathsf{pre}}$$ with $\Xprebf^{(i)}\in \R^{n_{\mathsf{pre}}\times p}$ and  $\bm{\varepsilon}^{(i)}_{\mathsf{pre}} \sim N(\bm{0},\sigma^2_{\mathsf{pre}} \Ibf_{n_{\mathsf{pre}}})$. Under \eqref{formulation}, each downstream task is assumed to be a linear combination of upstream tasks $\Bstar = [{\mathbf{b} ^\star}^{(1)},\ldots,{\mathbf{b} ^\star}^{(q)}]$
$$\st = \Bstar \alphastar = \sum_{i=1}^q {\mathbf{b} ^\star}^{(i)} \alpha ^\star_i$$
The goal is then to learn $\Bstar$ from upstream data $\{(\yprebf^{(i)}, \Xprebf^{(i)})\}_{i=1}^{q}$ and leverage this knowledge to improve performance downstream.
\end{example}

We recall the setting here for reader's convenience. Recall that in the upstream, the model receives training data for $q$ distinct upstream tasks $\qty(\yprebf^{(i)}, \Xprebf^{(i)})_{i=1}^{q}$ where $$\yprebf^{(i)}=\Xprebf^{(i)} {\mathbf{b} ^\star}^{(i)} + \bm{\varepsilon}^{(i)}_{\mathsf{pre}}$$ with $\Xprebf^{(i)}\in \R^{n_{\mathsf{pre}}\times p}$ and  $\bm{\varepsilon}^{(i)}_{\mathsf{pre}} \sim N(\bm{0},\sigma^2_{\mathsf{pre}} \Ibf_{n_{\mathsf{pre}}})$. Under \eqref{formulation}, each downstream task is assumed to be a linear combination of upstream tasks $\Bstar = [{\mathbf{b} ^\star}^{(1)},\ldots,{\mathbf{b} ^\star}^{(q)}]$
$$\st = \Bstar \alphastar = \sum_{i=1}^q {\mathbf{b} ^\star}^{(i)} \alpha ^\star_i$$
The goal is then to learn $\Bstar$ from upstream data $\{(\yprebf^{(i)}, \Xprebf^{(i)})\}_{i=1}^{q}$ and leverage this knowledge to improve performance downstream.

Assume that the pretraining data are abundant $n_{\mathsf{pre}}>p$ and the OLS estimators $\Bstarmis=[\tilde{\mathbf{b}}^{(1)},...,\tilde{\mathbf{b}}^{(q)}], \tilde{\mathbf{b}}^{(i)}=\qty({\Xprebf^{(i)}}^\top \Xprebf^{(i)})^{-1} {\Xprebf^{(i)}}^\top \yprebf^{(i)}$ 
are used to estimate $\Bstar$. Let $\tilde{\Rcal}^{\mathsf{avg}}$ be the objective $\Rcal^{\mathsf{avg}}$ based on the inaccurate estimate $\Bstarmis$. The result below characterizes the error in approximating $R $ with $\tilde{\Rcal}^{\mathsf{avg}}$. See proof in  \Cref{app:pretrainsamplecomp}.
\begin{theorem}\label{precomplexity}
    Suppose the assumptions in \Cref{concenprop} hold. Assume in addition that 
    $\eigbot^{-1}\qty({n_{\mathsf{pre}}^{-1}}\Xprebf^\top \Xprebf)\le M$. Then, for any $D>0$, there exists a constant $C=C(D,M)$ such that with probability at least $1-C(p^{-D}+n^{-D}+q^{-D})$, we have $\abs{R -\tilde{\Rcal}^{\mathsf{avg}}}\le C\cdot \mathcal{E}$ where
    \vspace{-2mm}
    \begin{equation*}
        \begin{aligned}
        \mathcal{E}:=\frac{1}{n^{1/7}}+\qty(\sqrt{\frac{\log q}{q}}+\sqrt{\frac{p}{n_{\mathsf{pre}}}\cdot \sigma^2_{\mathsf{pre}}})\cdot \normop{\BBstar}.
        \end{aligned}
    \end{equation*}
\end{theorem}

We prove \Cref{precomplexity} below. Recall definition of $\Bstarmis$ from \Cref{SectionPrertain}. Let us define the following notation
\begin{equation*}
     \tilde{\mathfrak{B}}=\sum_{i \in H} \frac{\that_i \left\langle\what_i, \Gammahat^{\frac{1}{2}} {\stmis} \right\rangle^2}{\left(1+\that_i b_0\right)^2}, \qquad \mathrm{where} \quad \stmis=\Bstarmis \alphastar.
\end{equation*}

We first prove the following lemma. 
\begin{lemma}\label{lemfdfs}
    We have the following inequality
    \begin{equation*}
        \abs{\mathfrak{B} - \tilde{\mathfrak{B}}}\le \frac{1}{4b_0}\cdot \normop{\Gammahat} \cdot \qty(\norm{(\Bstarmis-\Bstar)\alphastar}_2^2+2\norm{\st}\cdot \norm{(\Bstarmis-\Bstar)\alphastar}_2)
    \end{equation*}
\end{lemma}
\begin{proof}[Proof of \Cref{lemfdfs}]
First note that
\begin{equation*}
    \mathfrak{B} - \tilde{\mathfrak{B}}=\sum_{i\in H} \frac{\that_i \cdot \what_i^\top \Lbf \what_i}{\qty(1+\that_i b_0)^2}
\end{equation*}
where
\begin{equation*}
    \Lbf:=\Gammahat^{1/2} (\Bstarmis-\Bstar) \alphastar {\alphastar}^\top \qty(\Bstarmis-\Bstar)^\top \Gammahat^{1/2} +2\cdot \qty(\Gammahat^{1/2} \Bstar \alphastar) \qty(\Gammahat^{1/2}\qty(\Bstarmis-\Bstar)\alphastar)^\top.
\end{equation*}
From von Neumann's trace inequality, we have that
\begin{equation}\label{dfke}
    \abs{\mathfrak{B} - \tilde{\mathfrak{B}}} \le \normop{\sum_{i\in H}\frac{\that_i \what_i \what_i^\top}{\qty(1+\that_i b_0)^2}} \cdot \Tr(\Lbf)
\end{equation}
where
\begin{equation*}
    \Tr(\Lbf)=\norm{\Gammahat^{1/2} \qty(\Bstarmis-\Bstar)\alphastar}_2^2+2{\alphastar}^\top \qty(\Bstarmis-\Bstar)^\top \Gammahat \Bstar \alphastar.
\end{equation*}
For the RHS, we have that
\begin{equation*}
    \begin{aligned}
        & \normop{\sum_{i\in H}\frac{\that_i \what_i \what_i^\top}{\qty(1+\that_i b_0)^2}}\le \frac{1}{4b_0},\\
        &\norm{\Gammahat^{1/2} \qty(\Bstarmis-\Bstar)\alphastar}_2^2 \le \norm{\qty(\Bstarmis-\Bstar) \alphastar}_2^2 \cdot \normop{\Gammahat},\\
        &2{\alphastar}^\top \qty(\Bstarmis-\Bstar)^\top \Gammahat \Bstar \alphastar \le 2\normop{\Gammahat} \cdot \norm{\st}_2 \cdot \norm{(\Bstarmis-\Bstar)\alphastar}_2
    \end{aligned}
\end{equation*}
where we used the elementary inequality $x/(1+xb_0)^2 \le 1/(4b_0)$ in the first line. The result then follows from plugging the above into \eqref{dfke}. 
\end{proof}

We see from \Cref{lemfdfs} that $ \mathfrak{B} - \tilde{\mathfrak{B}}$ depends on the quantity $\norm{(\Bstarmis-\Bstar)\alphastar}_2$. The following lemma bounds this quantity using Hanson-Wright inequality, which we state in \Cref{Hanson}. 

\begin{lemma}\label{lemsdsd}
    Assume that $\frac{1}{n_\mathsf{pre}} \Xprebf^\top \Xprebf$ is non-singular. We have the following statement conditioned on $\alphastar$. For any $D>0$, there exists a constant $C(D)$ such that with probability at least $1-p^{-D}$, 
    \begin{equation*}
        \norm{\qty(\Bstar-\Bstarmis)\alphastar}_2^2\le C\cdot \frac{p}{n_{\mathsf{pre}}}\cdot \sigma^2_{\mathsf{pre}}\cdot \norm{\alphastar}_2^2 \cdot \eigbot^{-1} \qty(\frac{1}{n_\mathsf{pre}} \Xprebf^\top \Xprebf) \cdot \qty(1+2\sqrt{\frac{\log p}{p}}+2\frac{\log p}{p}).
    \end{equation*}
\end{lemma}
\begin{proof}[Proof of \Cref{lemsdsd}]
    Let us adopt the notation
    \begin{equation*}
        \mathbf{E}=\qty[\bm{\varepsilon}^{(1)}_{\mathsf{pre}},...,\bm{\varepsilon}^{(q)}_{\mathsf{pre}}].
    \end{equation*}

    Standard OLS theory then implies that
    \begin{equation*}
        \norm{\qty(\Bstarmis-\Bstar)\alphastar}_2^2=\norm{\qty(\Xprebf^\top \Xprebf)^{-1} \Xprebf^\top \mathbf{E} \alphastar}_2^2.
    \end{equation*}
    From the assumption that $\bm{\varepsilon}^{(i)}_{\mathsf{pre}} \stackrel{iid}{\sim} N(\bm{0},\Ibf_{n_{\mathsf{pre}}}), i=1,...,q$, we obtain that
    \begin{equation*}
        \qty(\Xprebf^\top \Xprebf) \Xprebf^\top \mathbf{E} \alphastar \sim N \qty(0, \sigma^2_{\mathsf{pre}}\cdot \frac{\norm{\alphastar}_2^2}{n_{\mathsf{pre}}} \cdot \qty(\frac{1}{n_{\mathsf{pre}}} \Xprebf^\top \Xprebf)^{-1}). 
    \end{equation*}
   This allows us to apply Hanson-Wright inequality, which yields that for any $x>0$, with probability at least $1-\exp(-x)$
   \begin{equation*}
       \begin{aligned}
           \norm{\qty(\Bstarmis-\Bstar)\alphastar}_2^2 &\le  \frac{c\cdot \sigma^2_{\mathsf{pre}}\cdot \norm{\alphastar}_2^2}{n_\mathsf{pre}} \cdot \qty(\sqrt{x}\cdot \norm{\frac{1}{n_\mathsf{pre}} \Xprebf^\top \Xprebf}_F+x\cdot \normop{\frac{1}{n_\mathsf{pre}} \Xprebf^\top \Xprebf})  \\
           & \le  \frac{c\cdot p\cdot  \sigma^2_{\mathsf{pre}}\cdot \norm{\alphastar}_2^2}{n_\mathsf{pre}} \qty(\frac{x}{p}+\sqrt{\frac{x}{p}})\cdot \eigbot^{-1} \qty(\frac{1}{n_\mathsf{pre}} \Xprebf^\top \Xprebf)
       \end{aligned}
   \end{equation*}
   where $c$ is some absolute constant. The result follows from taking $x\gets D\log p$. 
\end{proof}

Now we are ready to prove the main result \Cref{precomplexity}. 
\begin{proof}[Proof of \Cref{precomplexity}]
Combining \Cref{lemfdfs} and \Cref{lemsdsd} yields the following statement: conditioned on $\alphastar$, for any $D>0$, there exists a constant $C(D)$ such that with probability at least $1-p^{-D}$, 
\begin{equation*} 
        \abs{\mathfrak{B} - \tilde{\mathfrak{B}}}\le \frac{C}{4b_0}\cdot \normop{\Gammahat} \cdot \norm{\st} \cdot \sqrt{\frac{p}{n_\mathsf{pre}}}\cdot \sigma_{\mathsf{pre}} \cdot \norm{\alphastar}_2\cdot \eigbot^{-1/2} \qty(\frac{1}{n_\mathsf{pre}} \Xprebf^\top \Xprebf).
\end{equation*}
Using the assumption that 
$$\eigbot^{-1}\qty({n_{\mathsf{pre}}^{-1}}\Xprebf^\top \Xprebf)\le M,$$
along with \Cref{bdthat} and \Cref{bA1A2size}, we have the further upper bound that
\begin{equation*}
    \abs{\mathfrak{B} - \tilde{\mathfrak{B}}}\le C(D,M) \cdot  \sqrt{\frac{p}{n_\mathsf{pre}}\cdot  \sigma_{\mathsf{pre}}^2}\cdot \normop{\BBstar} \cdot\norm{\alphastar}_2^2.
\end{equation*}
Taking expectation with respect to $\alphastar$ and uses \eqref{concen}, the above becomes 
\begin{equation*}
    \E_{\alphastar} \abs{\mathfrak{B} - \tilde{\mathfrak{B}}}\le C(D,M) \cdot  \sqrt{\frac{p}{n_\mathsf{pre}}\cdot  \sigma_{\mathsf{pre}}^2}\cdot \normop{\BBstar}.
\end{equation*}
Now using the above, along with \Cref{bdthat} and \Cref{bA1A2size}, we have the following statement: for any $D>0$, there exists a constant $C(D,M)$ such that with probability at least $1-p^{-D}$, 
\begin{equation*}
    \abs{\tilde{\Rcal}^{\mathsf{avg}} - \Rcal^{\mathsf{avg}}}=\qty(1+\mathcal{V}) \cdot \abs{\E_{\alphastar} \qty(\mathfrak{B}  -  \tilde{\mathfrak{B}}) }\le C \cdot  \sqrt{\frac{p}{n_\mathsf{pre}}\cdot  \sigma_{\mathsf{pre}}^2}\cdot \normop{\BBstar}.
\end{equation*}
The result follows by combining this statement with \Cref{concenprop} via a union bound argument.

\end{proof}

\section{Supplementary Figures}

\subsection{Compare Fine-Grained Risk Components in \Cref{ablation}}\label{SIablationA}
\Cref{FG} plot asymptotic characterization (i.e. $\mathfrak{B} , \mathfrak{B}\cdot \mathcal{V}, \sigma^2 \mathcal{V} $) for fine-grained risk components (i.e. $B , V_{\Xbf} , V_{\Xbf,\epbf} $) averaged across $N=3000$ dranws of $\alphastar$. First row shows that as we increase $p/n$, bias monotonically increases and $V_{\Xbf,\epbf} $ monotonically decreases. 

Observe how EEP trades a small increase in bias for a large decrease invariance. The second row shows that $B $ and $V_{\Xbf} $ remain close to zero for OFP and EEP; this is expected as we are in the hard-selection regime where $q<n$. Notably, EEP can avoids the divergence of $V_{\Xbf,\epbf} $ at $n=q$ by slightly increasing its bias. The third row shows an interesting phenomenon not captured in \Cref{ablation}: $V_{\Xbf,\epbf} $ for EEP first increases before the width of $\Bhat$ reaches the capacity necessary to cotain $\Bstar$ (i.e. $q=50$). This may be explained by the following: when $k<q$, the model prioritizes using its additional resources to minimize $B $ and neglect the impact on $V_{\Xbf,\epbf} $; as $k$ surpasses $q$, the model has enough capacity to completely remove the bias and starts to use any additional resources to regulating variance. 
\begin{figure}[H] % (*) for spanning both columns
  \centering
  % First row
  \includegraphics[width=0.9\columnwidth]{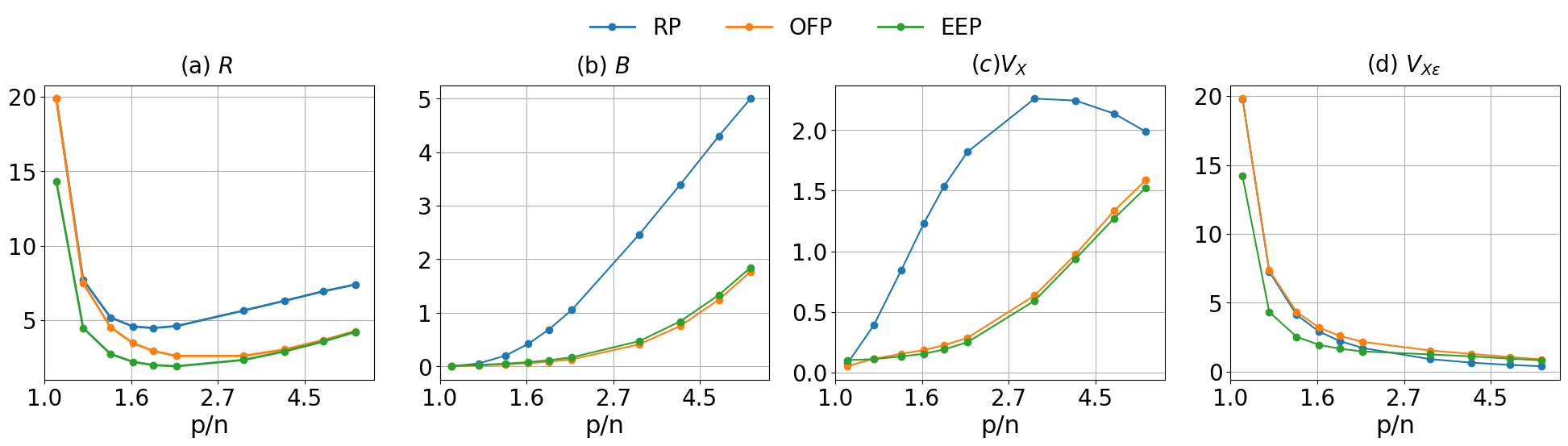} 

  % Second row
  \includegraphics[width=0.9\columnwidth]{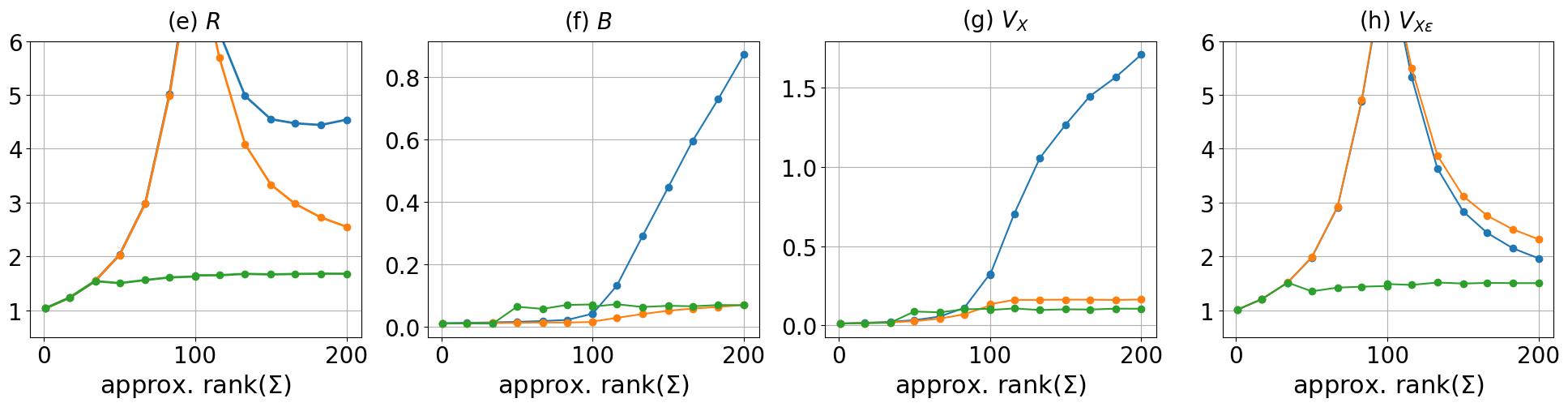} 

  % Third row
  \includegraphics[width=0.9\columnwidth]{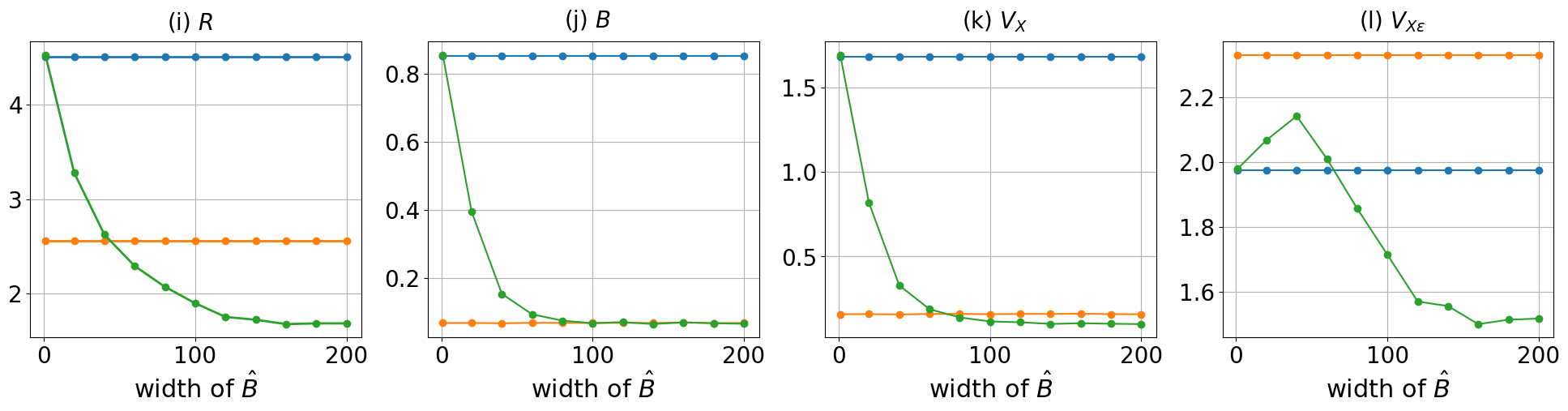}

  % Forth row
  % \includegraphics[width=0.9\columnwidth]{Image/FG4.png}
  \caption{Row (I) is in the same settings as Column (I) in \Cref{ablation}. Row (II) is in the same settings as Column (II) in \Cref{ablation}. Row (III) is in the same settings as Column (III) in \Cref{ablation}}
  \label{FG}
\end{figure}

\subsection{\Cref{ablation} with different choices of $q$}\label{SIablation}
\Cref{ablationB} repeats \Cref{ablation} with a different choice of $q$ in each column: if $q<n$ in \Cref{ablation}, \Cref{ablationB} plots $q>n$ in the corresponding column and vice versa. 

Most notable observation is that when $q<n$, bias can be completely removed by OFP and EEP whereas when $q>n$, the same does not happen. For this reason, as $p/n$ or $\mathsf{rank}(\Sigmabf)$ increases, the risk typically holds flat for EEP when $q<n$ (compare Column (I) and (II) between \Cref{ablation} and \Cref{ablationB}) but keeps increasing when $q>n$. 

\begin{figure}[H] % (*) for spanning both columns
  \centering
  % First row
  \includegraphics[width=0.75\columnwidth]{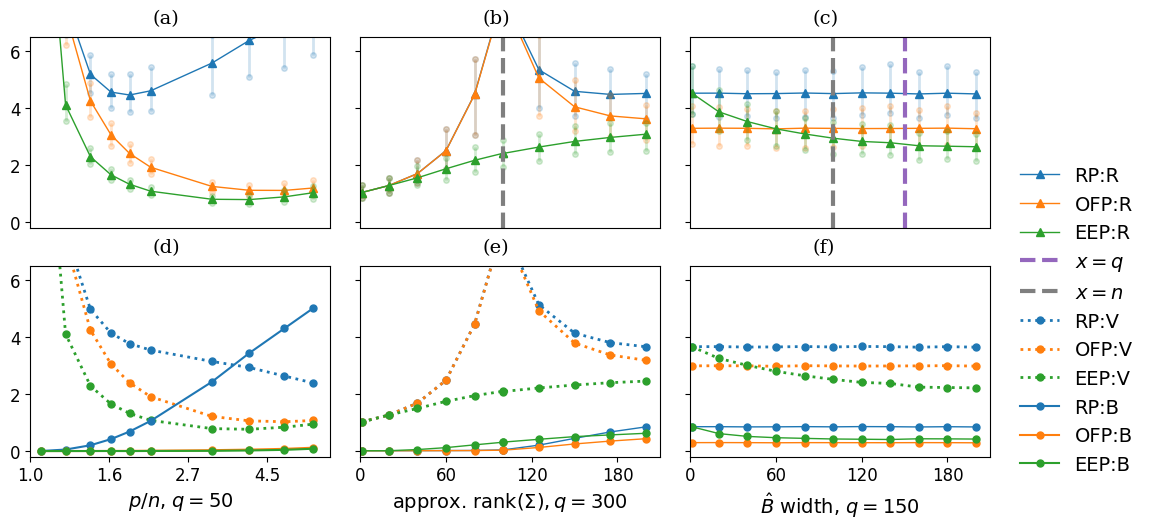} 
  \vspace{-0.5cm} 
  % % Second row
  % \includegraphics[width=2\columnwidth]{Image/biasvar.png} 
  \caption{Same settings as \Cref{ablation}. Column (I) fixes $p=600, q=50$ and vary $n$ from $560$ to $100$. Columns (II) varies $m$ for $q=300$. Column (III) varies $k$, the width of $\Bhat$ with $q=150$.}
  \label{ablationB}
\end{figure}

\subsection{Ablating $q$, SNR and Common Structure in $\Bstar$}\label{SIablationC}
Columns of \Cref{ablationC} varies different problem or model parameters: width $q$ of $\Bstar$, SNR, and AR coefficient $\varrho$ of columns of $\Bstar$ where columns of $\Bstar \in \R^{p \times q}$ drawn independently from $N(\bm{0}, \Sigmabf^{\Bstar}), \Sigmabf^{\Bstar}_{ij}=\varrho^{|i-j|}$. A larger $\varrho$ roughly corresponds to stronger common structure in the ground-truth featurization $\Bstar$. 

In (a), (d), we observe that risk of OFP and EEP increases as $q$ increases, and that when $q$ is small, EEP's bias is closer to that of OFP and when $q$ is large it approaches that of RP, suggesting a shift of EEP's emphasis from minimizing bias towards minimizing variance. In (b), (d), risk of all predictors increase as SNR increases where EEP and OFP's risks increase at a slower rate as they maintain bias to be relatively flat. In (c), (f), risk of EEP and OFP decrease as they are able to leverage stronger common structure in $\Bstar$ to reduce bias; we also see that EEP's bias moves from that of RP to that of OFP as $\varrho$ increases, suggesting its stronger ability to shift emphasis between controlling variance and controlling bias.
\begin{figure}[H] % (*) for spanning both columns
  \centering
  % First row
  \includegraphics[width=0.75\columnwidth]{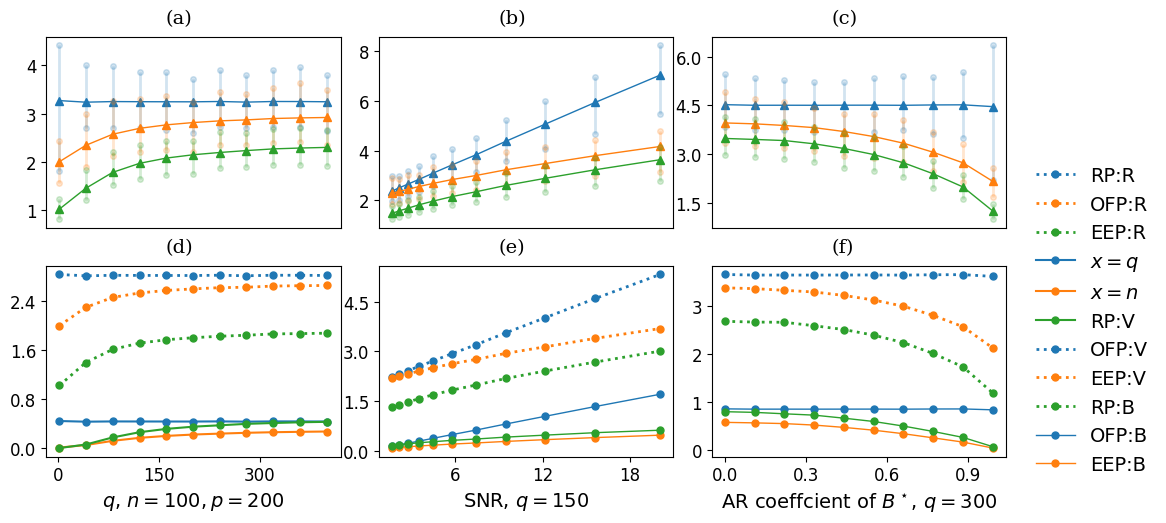} 
  \vspace{-0.5cm} 
  % % Second row
  % \includegraphics[width=2\columnwidth]{Image/biasvar.png} 
  \caption{Same settings as \Cref{ablation}. Column (I) varies $q$ from 0 to 400 (recall $n=100,p=200$). Column (II) varies SNR from 1.6 to 20 with $q=150$. Column (III) varies AR coefficient $\varrho$ of $\Bstar$'s columns' common covariance: columns of $\Bstar \in \R^{p \times q}$ drawn independently from $N(\bm{0}, \Sigmabf^{\Bstar}), \Sigmabf^{\Bstar}_{ij}=\varrho^{|i-j|}$. }
  \label{ablationC}
\end{figure}

% \subsection{Optimize $\mathcal{V}, \mathfrak{B}$ alone}
% \label{SIheatBA}

\end{document}